\documentclass[11pt]{article}
\usepackage{amsmath}
\usepackage{graphicx}
\usepackage{natbib}
\setcitestyle{numbers,square,comma}
\usepackage{url} 
\usepackage[table,x11names]{xcolor}
\usepackage{graphicx}
\usepackage{longtable}
\usepackage{float}
\usepackage{forest}
\usepackage{setspace}
\usepackage{wrapfig}
\usepackage{soul}
\usepackage{amsmath}
\usepackage{mathtools}
\usepackage{textcomp}
\usepackage{marvosym}
\usepackage{wasysym}
\usepackage{latexsym}
\usepackage{amssymb}
\usepackage{hyperref}
\usepackage{graphicx}
\usepackage{subcaption}
\usepackage{longtable}
\usepackage{float}
\usepackage{array}
\usepackage{algpseudocode}
\usepackage{algorithm}
\usepackage{algpseudocode}
\usepackage{comment}
\usepackage{dsfont}
\usepackage{float}
\usepackage{multirow}
\usepackage{comment}
\newcommand{\blind}{0}

\input{./Definitions}

\begin{document}


\if0\blind
{
  \title{\bf Sparse Separable Factor Analysis in the Complex Domain with an Application to Local Field Potential Data}
  \author{Ian Hultman \hspace{.2cm}\\
    Department of Statistics and Actuarial Science, University of Iowa\\
    and \\
    Kirtikanth Kalapatapu  \hspace{.2cm}\\
    Department of Statistics and Actuarial Science, University of Iowa\\
    and \\
    Yassine Filali  \hspace{.2cm}\\
    Department of Molecular Physiology and Biophysics, University of Iowa\\
    and \\
    Rainbo Hultman  \hspace{.2cm}\\
    Department of Molecular Physiology and Biophysics, University of Iowa\\
    and \\
    Sanvesh Srivastava\thanks{Email: sanvesh-srivastava@uiowa.edu}  \hspace{.2cm} \\
    Department of Statistics and Actuarial Science, University of Iowa
  }
  \maketitle
} \fi

\if1\blind
{
  \bigskip
  \bigskip
  \bigskip
  \begin{center}
    {\LARGE\bf Title}
\end{center}
  \medskip
} \fi

\bigskip
\begin{abstract}
  Complex-valued arrays arise in signal processing, where scientific interpretation depends on retaining amplitude and phase information. Existing covariance estimation methods either ignore the multiway organization of such data or rely on real-domain embeddings that do not directly exploit their complex structure. We develop sparse separable factor analysis (SSFA), a latent factor model for complex-valued arrays with a separable covariance structure across modes. Each mode-specific covariance matrix is modeled through a low-rank Hermitian factor structure and a diagonal residual covariance matrix. To obtain interpretable estimates, we impose elementwise lasso penalties on the complex loading matrices and estimate the SSFA parameters using a mode-wise parameter-expanded  expectation-maximization procedure. The resulting loading updates admit closed-form complex soft-thresholding solutions, which shrink the modulus of each loading while preserving its phase. A separate balancing step resolves the scale nonidentifiability of the separable covariance structure. Simulation studies show that SSFA improves covariance estimation relative to vectorization-based methods, including complex principal component analysis. We apply SSFA to local field potential recordings from mice, where we compare separability structures induced by different groupings of brain region, frequency, and time and perform model-based imputation of recordings missing because of electrode misplacement.
\end{abstract}

\noindent%
{\it Keywords:} complex-valued arrays; separable covariance; factor analysis; sparse loadings; parameter-expanded expectation-maximization; local field potentials.
\vfill

\newpage

\section{Introduction}
\label{sec:intro}

Complex-valued arrays arise in many scientific applications. In neuroscience, for example, Fourier transforms of multichannel electrical recordings produce complex-valued arrays indexed by channel, frequency, and time. A key challenge is to estimate their covariance while preserving the multiway organization and amplitude-phase dependence. Factor models provide a natural framework for low-rank covariance estimation, while sparse extensions improve interpretability by shrinking some loading entries exactly to zero \citep{carvalho2008high,fruhwirth2010parsimonious}. Separable factor analysis (SFA) imposes a Kronecker covariance structure across modes, but existing formulations are limited to real-valued data and dense loading matrices \citep{fosdick2014separable}. Existing sparse factor models are also restricted to real-valued vectors \citep{rovckova2016fast,Srietal17,fruhwirth2025sparse}. We address this gap by developing sparse separable factor analysis (SSFA) for complex-valued arrays, with sparse mode-specific loading matrices estimated using a lasso-regularized parameter-expanded expectation-maximization (PX-EM) procedure.

The complex representation is important in electrophysiology and spectroscopy. Local field potential (LFP) and electroencephalogram (EEG) data are commonly analyzed in the frequency domain, where each frequency component is characterized by its amplitude and phase \citep{brillinger2001time,urban2023oscillating}. Latent dependence may involve only subsets of brain regions and frequency bands, motivating sparse mode-specific loading matrices \citep{Galetal17,Taletal23,el2025methods}. These loadings identify the brain regions and frequency bands that contribute to dependence and yield parsimonious low-rank covariance estimates. Retaining the complex-valued representation preserves amplitude and phase information, whereas real-domain embeddings do not directly exploit this structure and can reduce interpretability.

Estimating the covariance structure of complex-valued arrays presents several challenges. First, applying principal component analysis (PCA) to vectorized arrays produces low-rank covariance estimates that do not exploit the multiway organization of the data and can be difficult to interpret \citep{horel1984complex,brillinger2001time}. Vectorization also increases the dimension of the loading vectors, which can lead to unstable estimation when the sample size is small relative to the array dimension. Second, imposing elementwise sparsity on complex-valued loading matrices requires care because each loading has a modulus and a phase. Penalizing its real and imaginary parts separately does not directly preserve this structure and can reduce interpretability. Finally, a separable covariance matrix is expressed as a Kronecker product of mode-specific covariance matrices, which are identifiable only up to multiplicative rescaling \citep{2011_Hoff}. Regularizing the mode-specific loading matrices therefore requires an identifiability constraint that fixes their relative scales.

Extensive literature exists on methods for modeling dependence among spectral features in electrophysiological and other multivariate time series. Classical approaches characterize linear spectral dependence under nonstationarity \citep{ombao2005slex,ombao2008evolutionary}, whereas more recent methods extend frequency-domain dependence measures to capture nonlinear relationships between oscillatory components \citep{pinto2023statistical,talento2024kencoh,redondo2025nonlinear}. These methods are principally designed to characterize frequency-domain dependence in multivariate time series. Our objective is related, but we instead develop a latent factor model for estimating separable covariance structures in complex-valued arrays using sparse mode-specific complex loading matrices. SSFA is not restricted to time-series applications and provides a general framework for structured covariance estimation in complex-valued arrays.

Embedding complex-valued data in the real domain provides one approach to covariance estimation. For example, the canonical block-skew-circulant (BSC) embedding represents a complex matrix $\Xb\in\CC^{n\times p}$ by $\Phib_{\Xb}\in\RR^{2n\times 2p}$, whose two row blocks are $[\operatorname{Re}(\Xb)\;\;-\operatorname{Im}(\Xb)]$ and $[\operatorname{Im}(\Xb)\;\;\operatorname{Re}(\Xb)]$, respectively \citep{hellings2015block}. Real-valued covariance estimation methods can then be applied to $\Phib_{\Xb}$. Methods such as PCA preserve the BSC structure, so the low-rank covariance estimate can be mapped back to a valid complex-valued covariance matrix \citep{hellings2015two}. This property need not be retained under sparse loading estimation, however, because penalizing the real and imaginary components separately can distort the modulus and phase of a complex loading. SSFA instead estimates and regularizes the loading matrices directly in the complex domain.

Complex-valued extensions of factor analysis and independent component analysis (ICA) have been developed for vector-valued data \citep{sardarabadi2017complex,urban2023oscillating}, but extending complex factor analysis to sparse array-valued settings is not straightforward. Complex ICA instead represents $\xb\in\CC^p$ as $\xb=\Ab\zb$, where $\zb\in\CC^p$ contains independent non-Gaussian latent sources and $\Ab\in\CC^{p\times p}$ is an unknown mixing matrix. Real-valued ICA methods use measures of non-Gaussianity, including kurtosis and mutual information \citep{hyvarinen2000independent}, with analogous methods available for complex-valued data \citep{bingham2000fast,eriksson2006complex}. Complex factor analysis and ICA are designed for vector-valued observations. Moreover, complex PCA and factor analysis do not impose sparsity on the loading matrices, and complex ICA does not impose sparsity on the mixing matrix. SSFA combines complex-valued factor modeling, array structure, and sparse mode-specific loading estimation within a single framework.

SSFA is built on an array-valued extension of the proper complex multivariate normal distribution \citep{And95,picinbono1996second,2011_Hoff,urban2023oscillating}. In the zero-mean Gaussian setting, properness implies circular symmetry and allows the dependence structure to be characterized by a Hermitian covariance matrix; see Section~\ref{sec:comp-array}. The complex array normal model specifies a separable covariance structure across modes. SSFA sets the mean array to zero and represents each mode-specific covariance matrix using a low-rank Hermitian factor-analytic structure with sparse complex loading matrices. Its likelihood preserves the multiway structure of the data and permits candidate separability structures to be compared through different vectorizations or matricizations. By estimating the parameters directly in the complex domain, SSFA avoids real-domain embeddings and preserves amplitude and phase information.

Separable factor analysis (SFA) has been developed for real-valued arrays, but existing methods estimate dense mode-specific loading matrices using a standard EM algorithm \citep{fosdick2014separable}. We develop a lasso-regularized PX-EM procedure for estimating sparse mode-specific loading matrices directly in the complex domain. Because the mode-specific covariance matrices are identifiable only up to multiplicative rescaling, we introduce a balancing condition that fixes their relative scales and permits comparable regularization across modes. The PX-EM working model assigns a nonidentity covariance matrix to the latent factors, enlarging the parameter space relative to the standard SFA formulation. This PX-EM formulation can facilitate faster convergence \citep{liu1998parameter} and includes real-valued SFA as a special case.

The PX-EM procedure has a tractable mode-wise structure. Each mode-specific update whitens all other modes, reducing the array problem to a complex-valued matrix factor-analysis problem. This avoids direct estimation of the full vectorized covariance matrix. The conditional update for the penalized loading matrix admits a closed-form complex soft-thresholding solution. We also construct a common regularization grid across modes. The loading update exactly minimizes its conditional penalized surrogate, whereas the residual covariance matrix is updated by a plug-in rule. We use the plug-in update because it produced more stable sparse solutions and rank selection in our numerical experiments.

SSFA has several practical advantages for analyzing LFP data. First, recordings from some brain regions may be unavailable because of electrode misplacement. SSFA uses animals with complete recordings to impute these missing observations while accounting for latent dependence among brain regions. This provides a model-based alternative to imputing each missing recording using the mean or median of the observed recordings from that region. Second, SSFA enables comparison of candidate separability structures involving brain region, frequency, time, and combinations of these modes. In our application, the matricization that combines brain region and frequency into one mode and retains time as the other yields the most interpretable dependence structure. Finally, we use SSFA to address scientific questions about frequency-specific dependence among brain regions in mice. Across a range of simulation settings, SSFA also outperforms vectorization-based methods, including complex PCA.

In summary, our contributions are threefold. First, we develop a complex-valued factor analysis model and extend it to sparse separable factor analysis for complex-valued arrays (Section~\ref{sec:vec-fct-mdl}). Second, we develop a lasso-regularized PX-EM procedure with a closed-form sparse loading update, a plug-in residual covariance update, and an identifiability-preserving balancing condition (Section~\ref{sec:ssfa}). Third, we demonstrate the practical utility of SSFA through simulation studies and an analysis of LFP data, including covariance-structure comparison and model-based imputation of missing recordings (Sections~\ref{sec:sim_analysis} and~\ref{sec:real_analysis}).

\section{Motivating application}
\label{sec:motivating-data}

\subsection{Data and methods}

We motivate the proposed methodology using brain-wide LFP recordings from mice. Electrodes implanted in multiple brain regions record electrical activity under different experimental conditions. Applying the fast Fourier transform yields complex coefficients indexed by brain region, frequency, and time window. A key scientific goal is to characterize frequency-specific dependence between brain regions and its evolution over time. Because this dependence may be driven by a smaller number of latent mechanisms, factor models are well suited for this analysis.

A common approach converts the Fourier coefficients into real-valued summaries before fitting a factor model. These summaries typically include frequency-specific power within each brain region and coherence between pairs of brain regions. Power measures the energy of a signal at a given frequency, whereas coherence measures the strength of linear association between two signals at that frequency. These feature vectors can be high-dimensional. For example, recordings from $13$ brain regions at $100$ frequencies produce $\bigl(13+\binom{13}{2}\bigr)\times100=9100$ real-valued power and coherence features for each time window.

We instead model the Fourier coefficients directly. This preserves their complex-valued representation and avoids separately constructing amplitude and dependence summaries. It also provides a more parsimonious representation. In the preceding example, each time window contains $13\times100=1300$ complex-valued Fourier coefficients. More importantly, a complex-valued factor model directly estimates the low-rank covariance structure across brain regions and frequencies. The evolution of this dependence over time can then be modeled under different separability assumptions; Section~\ref{sec:comp-array} introduces the corresponding models and estimation procedures.

A covariance estimate from a complex-valued factor model yields estimates of the coherence and phase-offset matrices. Let $\widehat{\Sigmab}$ denote the estimated covariance matrix and define $\widehat{\Rb} = \bigl(\widehat{\Sigmab}\circ\Ib \bigr)^{-1/2} \widehat{\Sigmab} \bigl(\widehat{\Sigmab}\circ\Ib \bigr)^{-1/2}$, where $\circ$ is the Hadamard product. The magnitude coherence and phase-offset matrices are $|\widehat{\Rb}|$ and $\arg(\widehat{\Rb})$, respectively, with both operations applied elementwise. We refer to magnitude coherence simply as coherence. The off-diagonal coherence values lie in $[0,1]$, with larger values indicating stronger linear associations. Phase offsets lie in $(-\pi,\pi]$ and describe relative phase, with lead--lag interpretations depending on the cross-spectrum convention. Phase offsets should also be interpreted jointly with coherence because they are less informative when the corresponding coherence is small.

\begin{figure}[htbp]
  \centering
  \includegraphics[
    width=\textwidth,
    height=0.9\textheight,
    keepaspectratio
  ]{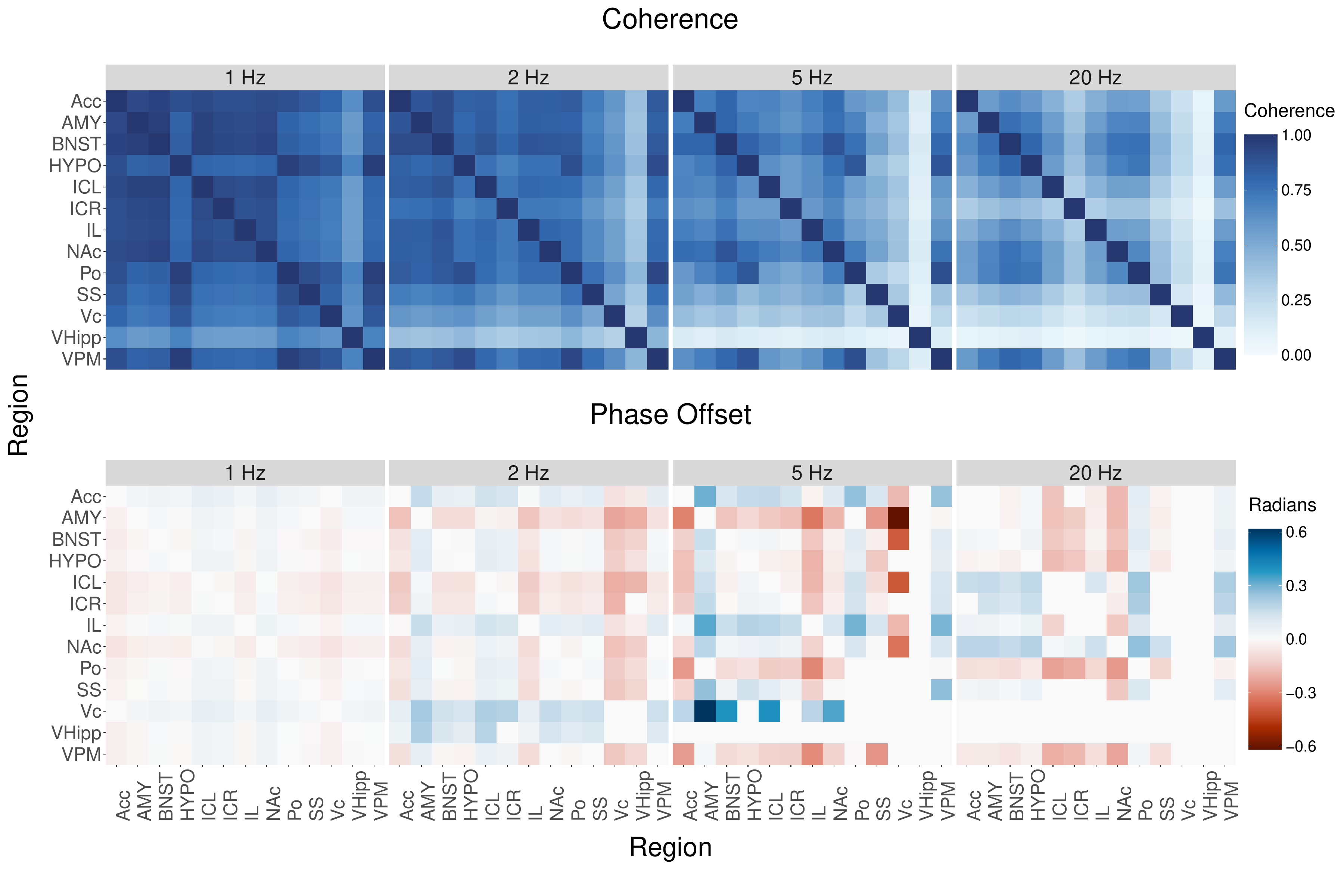}
  \caption{Estimated coherence matrices (top panels) and phase-offset matrices (bottom panels) at 1, 2, 5, and 20 Hz from the SSFA model fitted to the full 600-second time series in Section~\ref{sec:real_data_covariance_analysis}.}
  \label{fig:example_coherence_phase}
\end{figure}

\begin{figure}[htbp]
  \centering
  \includegraphics[
    width=\textwidth,
    height=0.9\textheight,
    keepaspectratio
  ]{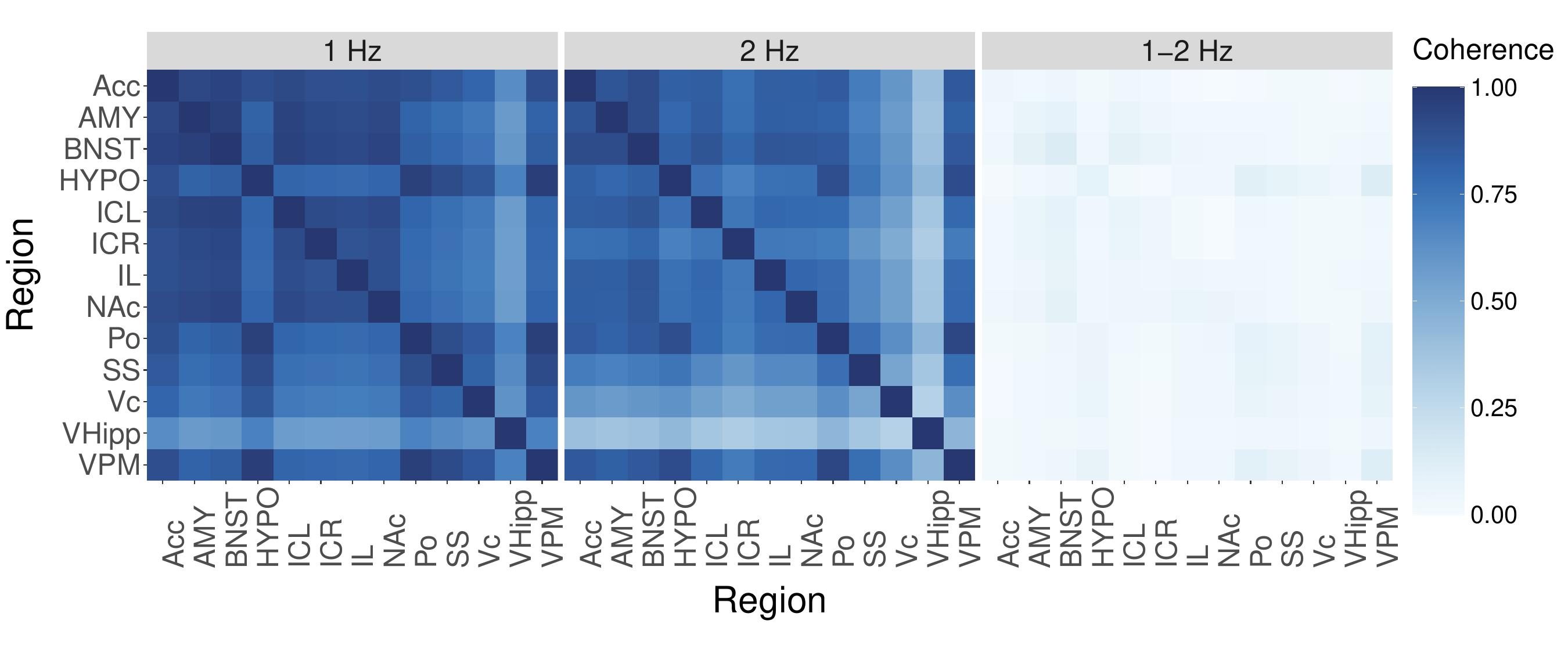}
  \caption{Estimated coherence matrices at 1 and 2 Hz (first two panels) and the estimated canonical cross-coherence matrix between 1 and 2 Hz (third panel) from the SSFA model fitted to the full 600-second time series in Section~\ref{sec:real_data_covariance_analysis}.}
  \label{fig:example_cross_coherence}
\end{figure}

\subsection{Application}

We use the fitted SSFA model to preview three aspects of the LFP application: within-frequency dependence, cross-frequency dependence, and imputation. Figure~\ref{fig:example_coherence_phase} shows the estimated coherence $|\widehat{\Rb}|$ and phase-offset $\arg(\widehat{\Rb})$ matrices at $1$, $2$, $5$, and $20$ Hz for the anterior cingulate cortex (Acc), amygdala (AMY), bed nucleus of the stria terminalis (BNST), hypothalamus (HYPO), left and right insular cortices (ICL and ICR), infralimbic cortex (IL), nucleus accumbens (NAc), posterior thalamic nucleus (Po), somatosensory cortex (SS), visual cortex (Vc), ventral hippocampus (VHipp), and ventral posteromedial thalamic nucleus (VPM). At $2$ Hz, for example, the VHipp--AMY pair has low coherence but a relatively large phase offset. This offset is less informative than a phase offset of similar magnitude for the Acc--AMY pair, whose coherence is substantially larger. We examine the estimated factors, coherence, and phase structure in detail in Section~\ref{sec:real_data_covariance_analysis}.

SSFA also distinguishes within-frequency from cross-frequency dependence. For a second-order stationary multivariate time series, Fourier coefficients at distinct frequencies are asymptotically uncorrelated \citep{jentsch2015test}. We use this property as a qualitative benchmark. Figure~\ref{fig:example_cross_coherence} compares the estimated coherence at $1$ and $2$ Hz with the canonical cross-coherence between these frequencies. The substantially weaker cross-frequency coherence is consistent with the expected near-orthogonality of distinct frequencies. Because brain region and frequency are modeled jointly, this structure is estimated rather than imposed through a separate frequency-mode covariance.

The estimated covariance structure also supports model-based imputation of missing recordings. In LFP experiments, measurements from some brain regions may be missing because electrodes are occasionally misplaced and record from unintended locations. Restricting the analysis to a common set of observed brain regions discards either the affected region or the affected animal, resulting in information loss. Under the assumption that the recordings are missing completely at random, SSFA imputes the missing Fourier coefficients conditionally on the observed signals from the same animal. The imputed Fourier coefficients are transformed to reconstruct the corresponding time-domain signals. We evaluate the imputation performance of SSFA in Section~\ref{sec:realdata_imputation_analysis}.

\section{Methodology for complex-valued vector data}
\label{sec:vec-fct-mdl}

\subsection{Complex-valued factor analysis}
\label{sec:vec_model_derivation}

We begin with a complex-valued latent variable model that serves as the building block for the separable factor model developed later. Let $\xb\in\CC^p$, $\zb\in\CC^k$, and $\eb\in\CC^p$ denote the observed vector, latent factor, and residual error, respectively, where $k\ll p$. The complex factor model is
\begin{align}
  \label{eq:fact-mdl}
  \xb
  =
  \Lambdab\zb+\eb,
  \qquad
  \Lambdab\in\CC^{p\times k},
  \qquad
  \zb\sim\CC\Ncal_k(\zero,\Ib_k),
  \qquad
  \eb\sim\CC\Ncal_p(\zero,\Psib),
\end{align}
where $\zb$ and $\eb$ are independent, $\Lambdab$ is the loading matrix, and $\Psib$ is a diagonal residual covariance matrix with positive diagonal entries. The notation $\CC\Ncal_c(\mb,\Vb)$ denotes the proper complex normal distribution with mean $\mb$ and covariance matrix $\Vb$; its analytic form is given in Appendix~\ref{sec:c-vec-mat-norm}.

A general complex Gaussian distribution is characterized by its mean, covariance matrix, and pseudo-covariance matrix \citep{urban2023oscillating}. For a complex random vector $\xb$ with mean $\mb$, the covariance and pseudo-covariance matrices are
\begin{align*}
  \Vb
  =
  \EE
  \left[
    (\xb-\mb)(\xb-\mb)^\star
  \right],
\qquad
  \Cb
  =
  \EE
  \left[
    (\xb-\mb)(\xb-\mb)^\top
  \right],
\end{align*}
where $^\star$ denotes the complex conjugate-transpose operator, $\Vb$ is Hermitian and $\Cb$ is symmetric. We focus on zero-mean proper complex Gaussian distributions, for which $\Cb=\zero$. In this setting, properness implies circular symmetry, so $\xb$ and $e^{\mathrm{i}\theta}\xb$ have the same distribution for every $\theta\in\RR$ \citep{picinbono1996second}. This restriction is natural because global phase rotations are a special case of the unitary nonidentifiability of the factor representation. Proper complex Gaussian models have also been used in previous works \citep{And95}.

We express the complex factor model for $n$ independent observations in matrix-variate complex normal form. For a matrix $\Yb\in\CC^{c_1\times c_2}$, the notation $\Yb\sim\CC\Ncal_{c_1\times c_2}(\Mb,\Vb_1,\Vb_2)$ means that $  \operatorname{vec}(\Yb) \sim \CC\Ncal_{c_1c_2} \left(\operatorname{vec}(\Mb), \Vb_2\otimes\Vb_1 \right)$; its analytic form is given in Appendix~\ref{sec:c-vec-mat-norm}. Let $\xb_1,\ldots,\xb_n\in\CC^p$ denote the $n$ independent observations. After augmenting each observation $\xb_i$ with its latent factor $\zb_i\in\CC^k$, the complete data are $\{(\xb_i,\zb_i):i=1,\ldots,n\}$. Define
\begin{align}
  \label{eq:data-mat}
  \Xb
  &=
  [\xb_1,\ldots,\xb_n]
  \in\CC^{p\times n},
  &
  \Zb
  &=
  [\zb_1,\ldots,\zb_n]
  \in\CC^{k\times n},
  &
  \Eb
  &=
  [\eb_1,\ldots,\eb_n]
  \in\CC^{p\times n}.
\end{align}
Because the pairs $(\xb_i,\zb_i)$ are independent copies of $(\xb,\zb)$ in \eqref{eq:fact-mdl}, the complete-data model is
\begin{align}
  \label{eq:X_vec_dataset_distn}
  \Xb
  =
  \Lambdab\Zb+\Eb,
  \qquad
  \Zb
  \sim
  \CC\Ncal_{k\times n}
  \left(
    \zero,
    \Ib_k,
    \Ib_n
  \right),
  \qquad
  \Eb
  \sim
  \CC\Ncal_{p\times n}
  \left(
    \zero,
    \Psib,
    \Ib_n
  \right).
\end{align}
Marginalizing over $\Zb$ in \eqref{eq:X_vec_dataset_distn} gives
\begin{align}\label{eq:X_vec_marg_model}
  \Xb
  \sim
  \CC\Ncal_{p\times n}
  \left(
    \zero,
    \Sigmab,
    \Ib_n
  \right),
  \qquad
  \Sigmab
  =
  \Lambdab\Lambdab^\star+\Psib,
\end{align}
with parameter vector $\thetab=(\Lambdab,\Psib)$.

We develop a PX-EM algorithm for maximum likelihood estimation of $\thetab$. The working model replaces $\Lambdab$ and $\zb$ in \eqref{eq:fact-mdl} with $\widetilde{\Lambdab}$ and $\widetilde{\zb}$ and introduces a $k\times k$ Hermitian positive-definite working covariance matrix $\Gammab$ for the latent factors:
\begin{align}
  \label{eq:px-fct-mdl}
  \xb
  =
  \tilde{\Lambdab}\tilde{\zb}+\eb,
  \qquad
  \tilde{\Lambdab}\in\CC^{p\times k},
  \qquad
  \tilde{\zb}
  \sim
  \CC\Ncal_k
  \left(
    \zero,
    \Gammab
  \right),
  \qquad
  \eb
  \sim
  \CC\Ncal_p
  \left(
    \zero,
    \Psib
  \right).
\end{align}
Marginalizing over $\tilde{\zb}$ in \eqref{eq:px-fct-mdl} implies that the covariance matrix of $\xb$  is $\tilde{\Lambdab}\Gammab\tilde{\Lambdab}^\star+\Psib$. Let $\operatorname{chol}(\Ab)$ denotes the lower-triangular Cholesky factor satisfying $\Ab=\operatorname{chol}(\Ab)\operatorname{chol}(\Ab)^\star$. If $\Rb=\operatorname{chol}(\Gammab)$, then the parameter-reduction map is $  \Lambdab = \tilde{\Lambdab}\Rb$. This map preserves the observed covariance matrix and returns the working model to the original parameterization in \eqref{eq:fact-mdl}. For the $n$ observations, the PX-EM complete-data working model is
\begin{align}
  \label{eq:X_vec_PXEM_dataset_distn}
  \Xb
  =
  \tilde{\Lambdab}\tilde{\Zb}+\Eb,
  \qquad
  \tilde{\Zb}
  \sim
  \CC\Ncal_{k\times n}
  \left(
    \zero,
    \Gammab,
    \Ib_n
  \right),
  \qquad
  \Eb
  \sim
  \CC\Ncal_{p\times n}
  \left(
    \zero,
    \Psib,
    \Ib_n
  \right).
\end{align}
The working parameter vector is $\tilde{\thetab} =(\tilde{\Lambdab},\Gammab,\Psib)$.

Each PX-EM iteration begins by computing the conditional moments of the latent factors under the current working parameter estimates. Let $\widetilde{\thetab}^{(t)} = (\widetilde{\Lambdab}^{(t)},\Gammab^{(t)},\Psib^{(t)})$ denote the working parameter estimate at the end of iteration $t$. The E-step computes
\begin{align}
  \label{eq:PX-EM_expected_z_given_x}
  \tilde{\Zb}^{(t)}
  =
  \EE
  \left(
    \tilde{\Zb}
    \mid
    \Xb,\tilde{\thetab}^{(t)}
  \right),
  \qquad
  \widetilde{\Sbb}_{zz}^{(t)}
  =
  \EE
  \left(
    \tilde{\Zb}\tilde{\Zb}^\star
    \mid
    \Xb,\tilde{\thetab}^{(t)}
  \right).
\end{align}
Detailed derivations and the analytic expressions of these conditional moments are provided in Proposition \ref{prop:pxem_posterior_moments}.

The M-step is implemented through two conditional maximization (CM) steps. The first CM-step holds $\Psib=\Psib^{(t)}$ fixed and updates the working latent covariance matrix and loading matrix as
\begin{align}
  \label{eq:main-2}
  \Gammab^{(t+1)}
  =
  \frac{1}{n}
  \widetilde{\Sbb}_{zz}^{(t)},
  \qquad
  \widetilde{\Lambdab}^{(t+1)}
  =
  \Xb
  \widetilde{\Zb}^{(t)\star}
  \left(
    \widetilde{\Sbb}_{zz}^{(t)}
  \right)^{-1}.
\end{align}
Let $\widetilde{\Hb}_{zz}^{(t)} = \operatorname{chol} \left(\widetilde{\Sbb}_{zz}^{(t)} \right)$. Applying the parameter-reduction map yields the unpenalized reduced loading estimate
\begin{align}
  \Ab^{(t)}
  =
  \widetilde{\Lambdab}^{(t+1)}
  \operatorname{chol}
  \left(
    \Gammab^{(t+1)}
  \right)
  =
  \frac{1}{\sqrt{n}}\,
  \Xb
  \widetilde{\Zb}^{(t)\star}
  \widetilde{\Hb}_{zz}^{(t)-\star}.
  \label{eq:main-4}
\end{align}
For the unpenalized model, we set $\Lambdab^{(t+1)}=\Ab^{(t)}$. The second CM-step holds $\Lambdab=\Lambdab^{(t+1)}$ fixed and updates the diagonal residual covariance matrix as
\begin{align}
  \label{eq:main-3}
  \Psib^{(t+1)}
  =
  \left[
    \frac{1}{n}\Xb\Xb^\star
    -
    \Lambdab^{(t+1)}
    \Lambdab^{(t+1)\star}
  \right]
  \circ
  \Ib_p.
\end{align}
Detailed derivations of the two CM updates are provided in Appendices~\ref{sec:app_fa_desc_m_step} and~\ref{sec:app_Psi_estimate}.

These updates are analogous to those for real-valued factor analysis, with transposes replaced by conjugate-transposes. The main distinction arises after imposing sparsity, where the loading update shrinks the modulus of each complex loading while preserving its phase. We develop this sparse extension next.

\subsection{Complex-valued sparse factor analysis}
\label{sec:sparse_vec_model_derivation}

Sparse loading estimates are often easier to interpret in biomedical applications, where each latent factor may involve only a subset of the observed variables. We obtain such estimates by imposing an elementwise complex lasso penalty on the reduced loading matrix $\Lambdab$. The sparse PX-EM procedure retains the E-step and unpenalized working-parameter updates derived in the previous subsection.

We now describe iteration $t+1$ of the sparse PX-EM procedure. We first perform the PX-EM E-step and the unpenalized working-parameter updates in \eqref{eq:main-2}. The parameter-reduction map in \eqref{eq:main-4} then yields the unpenalized reduced loading estimate $\Ab^{(t)}$. The expansion and reduction steps adapt the orientation and scale of the loading matrix before sparsity is imposed. We apply complex soft-thresholding to $\Ab^{(t)}$ to obtain $\Lambdab^{(t+1)}$, which is used to update the residual covariance matrix.

First, we hold $\Psib=\Psib^{(t)}$ fixed and update $\Lambdab$ by minimizing the penalized reduced objective
\begin{align}
  \mathcal{L}_{\rho}
  \left(
    \Lambdab
    \mid
    \thetab^{(t)}
  \right)
  &=
  \ell_{\Lambdab}
  \left(
    \Lambdab
    \mid
    \thetab^{(t)}
  \right)
  +
  \rho\lVert\Lambdab\rVert_{1,\CC},
  \label{eq:penalized-reduced-loading-objective}
\end{align}
where the reduced loading loss $\ell_{\Lambdab}$ and elementwise complex lasso penalty $\rho\lVert\Lambdab\rVert_{1,\CC}$ are, respectively,  defined as
\begin{align*}
  \ell_{\Lambdab}
  \left(
    \Lambdab
    \mid
    \thetab^{(t)}
  \right)
  =
  \operatorname{tr}
  \left[
    \left(
      \Lambdab-\Ab^{(t)}
    \right)^\star
    \left(
      \Psib^{(t)}
    \right)^{-1}
    \left(
      \Lambdab-\Ab^{(t)}
    \right)
    \right], \qquad
\rho\lVert\Lambdab\rVert_{1,\CC}
  =
  \rho
  \sum_{r=1}^{p}
  \sum_{c=1}^{k}
  |\lambda_{rc}|.
\end{align*}
The tuning parameter satisfies $\rho \geq 0$, $\lambda_{rc}$ denotes the $(r,c)$th entry of $\Lambdab$, and $|\lambda_{rc}|$ denotes its modulus. The following corollary to Lemma~\ref{lem:complex_soft_thresholding} gives its unique minimizer conditional on $\Psib=\Psib^{(t)}$.

\begin{corollary}[Closed-form sparse loading update]
\label{cor:sparse_pxem_loading_update}

Let $\Ab^{(t)}=(a_{rc}^{(t)})$ be the unpenalized reduced loading estimate in \eqref{eq:main-4}, and let $\psi_{rr}^{(t)}$ be the $r$th diagonal entry of $\Psib^{(t)}$. The unique minimizer of $\mathcal{L}_{\rho} \left(\Lambdab\mid\thetab^{(t)} \right)$ in \eqref{eq:penalized-reduced-loading-objective} over $\Lambdab\in\CC^{p\times k}$ is $\Lambdab^{(t+1)}$, with entries
\begin{align}
  \lambda_{rc}^{(t+1)}
  &=
  \mathcal{T}_{\rho\psi_{rr}^{(t)}/2}
  \left(
    a_{rc}^{(t)}
  \right),
  \qquad
  r=1,\ldots,p,
  \quad
  c=1,\ldots,k,
  \label{eq:sparse_lam_mod_update}
\end{align}
where the complex soft-thresholding operator is given by
\begin{align}
  \mathcal{T}_{\rho\psi_{rr}^{(t)}/2}
  \left(
    a_{rc}^{(t)}
  \right)
  &=
  \begin{cases}
    \displaystyle
    \left(
      1-
      \frac{\rho\psi_{rr}^{(t)}}
           {2|a_{rc}^{(t)}|}
    \right)
    a_{rc}^{(t)},
    &
    |a_{rc}^{(t)}|
    >
    \dfrac{\rho\psi_{rr}^{(t)}}{2},
    \\[3mm]
    0,
    &
    |a_{rc}^{(t)}|
    \leq
    \dfrac{\rho\psi_{rr}^{(t)}}{2}.
  \end{cases}
  \label{eq:sparse_lam_piecewise_update}
\end{align}
All entries of $\Lambdab^{(t+1)}$ can be updated simultaneously.
\end{corollary}
The update in \eqref{eq:sparse_lam_piecewise_update} shrinks the modulus of each nonzero complex loading while preserving its phase.

We update the diagonal residual covariance matrix holding $\Lambdab=\Lambdab^{(t+1)}$ fixed. Specifically, we replace the unpenalized reduced estimate $\Ab^{(t)}$ in \eqref{eq:main-3} with the sparse loading estimate $\Lambdab^{(t+1)}$ to obtain
\begin{align}
  \Psib^{(t+1)}
  &=
  \left[
    \frac{1}{n}\Xb\Xb^\star
    -
    \Lambdab^{(t+1)}
    \Lambdab^{(t+1)\star}
  \right]
  \circ\Ib_p.
  \label{eq:sparse-psi-update}
\end{align}
Therefore, the sparsity structure of $\Lambdab^{(t+1)}$ is retained, and each diagonal entry of $\Psib^{(t+1)}$ represents the empirical marginal variance not explained by the retained sparse factor structure. We use this plug-in update because it produced more stable sparse solutions and rank selection in our numerical experiments. We refer to the resulting algorithm as the sparse PX-EM procedure. Appendix~\ref{sec:lasso_derivation_app} derives the reduced loading objective from the conditional expected complete-data log-likelihood, the complex soft-thresholding update, and the plug-in residual covariance update. Section~\ref{sec:rho_grid} describes the regularization grid and selection of $\rho$ for the array-valued model.

\section{Methodology for complex-valued array data}
\label{sec:ssfa}

\subsection{Complex array-normal distribution}
\label{sec:comp-array}

We return to the application introduced in Section~\ref{sec:motivating-data} and extend the complex factor model in Section~\ref{sec:vec-fct-mdl} to array-valued data. The training data consist of Fourier-transformed LFP recordings from $n$ mice over multiple time windows. For $i=1,\ldots,n$, let $\Xcal_i\in\CC^{r\times f\times t}$ denote the recording from mouse $i$, where $r$, $f$, and $t$ are the numbers of brain regions, frequencies, and time windows, respectively. Our goal is to estimate dependence across brain regions and frequencies and characterize its evolution over time.

SFA captures these dependencies through low-rank factor-analytic representations of the mode-specific covariance matrices under separability assumptions \citep{fosdick2014separable}. The simplest approach applies the complex factor model in \eqref{eq:fact-mdl} to the vectorized arrays. At the other extreme, we can impose a separable covariance structure across the brain-region, frequency, and time modes. An intermediate approach matricizes each $\Xcal_i$ as an $rf\times t$ matrix and imposes separability across the combined brain-region--frequency mode and the time mode. These covariance representations correspond to different scientific interpretations. For example, if the Fourier coefficients are approximately uncorrelated across time windows, then the $rf\times t$ representation provides a more parsimonious and interpretable description than directly modeling the covariance of $\operatorname{vec}(\Xcal_i)$.

The complex array normal distribution provides a common framework for these covariance representations. We define it by combining the real-valued array normal construction \citep{2011_Hoff} with the matrix-variate complex normal distribution \citep{And95}. Let $\Ycal\in\CC^{c_1\times\cdots\times c_d}$ be a $d$-mode complex-valued array, where $c_j$ is the dimension of mode $j$. The vectorization operator maps $\Ycal$ to $\yb=\operatorname{vec}(\Ycal)\in\CC^{\prod_{\ell=1}^d c_\ell}$ by stacking its entries with the first-mode index varying fastest, followed by the second-mode index, and so on \citep{2009_Kolda}. We use the following definition throughout the paper; Appendix~\ref{sec:c-array-norm} provides additional details.

\begin{definition}[Complex array normal distribution]
  Let $\Mcal\in\CC^{c_1\times\cdots\times c_d}$, and let $\Vb_j\in\CC^{c_j\times c_j}$ be Hermitian positive definite for $j=1,\ldots,d$. A $d$-mode array $\Ycal$ follows a proper complex array normal distribution with mean $\Mcal$ and mode-specific covariance matrices $\Vb_1,\ldots,\Vb_d$ if
\begin{align}
  \label{eq:complex-array-normal}
  \operatorname{vec}(\Ycal)
  \sim
  \CC\Ncal_{\prod_{j=1}^d c_j}
  \left(
    \operatorname{vec}(\Mcal),
    \Vb_d\otimes\cdots\otimes\Vb_1
  \right),
\end{align}
where $\otimes$ denotes the Kronecker product. We write $\Ycal\sim \CC\Ncal_{c_1\times\cdots\times c_d} (\Mcal,\Vb_1,\ldots,\Vb_d)$.
\end{definition}

\subsection{Complex-valued separable factor analysis}
\label{sec:complex-sfa}

The complex array normal distribution in \eqref{eq:complex-array-normal} enables an extension of SFA to the complex domain \citep{fosdick2014separable}.
For $j=1,\ldots,d$, let $\Lambdab_j\in\CC^{p_j\times k_j}$ be the mode-$j$ loading matrix, where $k_j\ll p_j$, and let $\Psib_j$ be a $p_j\times p_j$ diagonal matrix with positive diagonal entries.  The complex-domain SFA model specifies the distribution of $\dot{\Xcal}\in\CC^{p_1\times\cdots\times p_d}$ as
\begin{align}
  \dot{\Xcal}
  \sim
  \CC\Ncal_{p_1\times\cdots\times p_d}
  \left(
    \zero,
    \Sigmab_1,\ldots,\Sigmab_d
  \right),
  \qquad
  \Sigmab_j
  =
  \Lambdab_j\Lambdab_j^\star+\Psib_j,
  \qquad
  j=1,\ldots,d.
  \label{eq:SSFA_model_intro}
\end{align}
Each mode-specific covariance matrix $\Sigmab_j$ is the sum of the low-rank factor component $\Lambdab_j\Lambdab_j^\star$ and the diagonal residual covariance matrix $\Psib_j$.

Maximum likelihood estimation under \eqref{eq:SSFA_model_intro} does not reduce directly to the PX-EM procedure in Section~\ref{sec:vec-fct-mdl}. Unlike the single latent-variable block in \eqref{eq:fact-mdl} and \eqref{eq:px-fct-mdl}, expansion of the Kronecker covariance in \eqref{eq:SSFA_model_intro} introduces multiple latent components corresponding to combinations of factor and residual terms across modes. To avoid these additional latent components, we update one mode at a time. For each mode, we whiten the observed arrays using the covariance matrices of the remaining modes and matricize the whitened arrays along the selected mode.

For each mode $j$, whitening all other modes reduces the array-valued model to a matrix-variate factor model for mode $j$. Suppose
$\Xcal_1,\ldots,\Xcal_n\in\CC^{p_1\times\cdots\times p_d}$ are independent copies of $\dot{\Xcal}$ in \eqref{eq:SSFA_model_intro}. Stack them into the $(d+1)$-mode array $\Xcal\in\CC^{p_1\times\cdots\times p_d\times n}$, whose $i$th slice along mode $d+1$ is $\Xcal_i$. Then, $  \Xcal \sim \CC\Ncal_{p_1\times\cdots\times p_d\times n} \left(\zero, \Sigmab_1,\ldots,\Sigmab_d,\Ib_n \right)$. For a fixed $j\in\{1,\ldots,d\}$, let $\Lb_\ell=\operatorname{chol}(\Sigmab_\ell)$ and define
\begin{align}
  \Mb_{\ell\mid j}
  &=
  \begin{cases}
    \Ib_{p_j},
    & \ell=j,\\
    \Lb_\ell^{-1},
    & \ell\neq j,
  \end{cases}
  \qquad
  \ell=1,\ldots,d.
  \label{eq:mode-wise-whitening-matrices}
\end{align}
The mode-$j$ whitened array is
\begin{align}
  \breve{\Xcal}_{(j)}
  =
  \Xcal
  \times_1\Mb_{1\mid j}
  \cdots
  \times_d\Mb_{d\mid j}
  \times_{d+1}\Ib_n.
  \label{eq:whitened-mode-j-array}
\end{align}
It follows that $\breve{\Xcal}_{(j)} \sim \CC\Ncal_{p_1\times\cdots\times p_d\times n} \left(\zero, \Ib_{p_1},\ldots,\Ib_{p_{j-1}}, \Sigmab_j, \Ib_{p_{j+1}},\ldots,\Ib_{p_d}, \Ib_n \right)$ \citep{2011_Hoff}.  Let $\breve{\Xb}_{(j)}\in\CC^{p_j\times n_j}$ denote the mode-$j$ matricization of $\breve{\Xcal}_{(j)}$, where $n_j=n\prod_{\ell\neq j}p_\ell$. Because all covariance matrices other than $\Sigmab_j$ are identities,
\begin{align}
  \breve{\Xb}_{(j)}
  \sim
  \CC\Ncal_{p_j\times n_j}
  \left(
    \zero,
    \Sigmab_j,
    \Ib_{n_j}
  \right),
  \qquad
  \Sigmab_j
  =
  \Lambdab_j\Lambdab_j^\star+\Psib_j.
  \label{eq:mode-j-mat}
\end{align}
Therefore, $\breve{\Xb}_{(j)}$ has the matrix-variate structure induced by the vector-valued factor model in \eqref{eq:X_vec_dataset_distn}, with $n_j$ mode-$j$ fibers replacing the original sample size $n$.

\subsection{Complex-valued sparse separable factor analysis (SSFA)}
\label{sec:reg-est-ssfa}

SSFA extends the SFA model in \eqref{eq:SSFA_model_intro} by imposing an elementwise lasso penalty on each mode-specific loading matrix. For a fixed mode $j\in\{1,\ldots,d\}$, the matrix-variate model in \eqref{eq:mode-j-mat} admits the latent-variable representation
\begin{align}
  \label{eq:sfa-j}
  \breve{\Xb}_{(j)}
  =
  \Lambdab_j\Zb_{(j)}+\Eb_{(j)}, \qquad
  \Zb_{(j)}
  \sim
  \CC\Ncal_{k_j\times n_j}
  \left(
    \mathbf{0},
    \Ib_{k_j},
    \Ib_{n_j}
  \right), \qquad
  \Eb_{(j)}
  \sim
  \CC\Ncal_{p_j\times n_j}
  \left(
    \mathbf{0},
    \Psib_j,
    \Ib_{n_j}
  \right),
\end{align}
where the matrix $\Zb_{(j)}$ contains the latent factors for the $n_j$ mode-$j$ fibers, and $\Eb_{(j)}$ is the corresponding residual error matrix. We apply the PX-EM procedure from Section~\ref{sec:vec-fct-mdl} to this mode-specific factor model. The PX-EM working model introduces a covariance parameter for the latent factors, which is removed through parameter reduction, yielding an objective in terms of mode-$j$ parameter block $\thetab_j = (\Lambdab_j,\Psib_j)$.

The PX-EM procedure updates the mode-specific parameter blocks sequentially. Denote the SSFA parameter estimate at the end of iteration $t$ by
\begin{align}
  \label{eq:th-t}
  \thetab^{(t)}  = (\thetab_1^{(t)}, \ldots, \thetab_d^{(t)})
  =
  \left\{
    (\Lambdab_1^{(t)},\Psib_1^{(t)}),
    \ldots,
    (\Lambdab_d^{(t)},\Psib_d^{(t)})
  \right\}.
\end{align}
Within iteration $t+1$, define the intermediate iterates
\begin{align}
  \thetab^{[t,j]}
  &=
  \thetab^{(t+j/d)},
  \qquad
  j=0,\ldots,d.
  \label{eq:upd-par}
\end{align}
Immediately before the mode-$j$ update in iteration $t+1$, the current parameter estimate is $\thetab^{[t,j-1]}$. The parameter blocks for modes $1,\ldots,j-1$ have been updated, whereas those for modes $j,\ldots,d$ retain their values from iteration $t$. In particular, $\thetab^{[t,0]}=\thetab^{(t)}$ and $\thetab^{[t,d]}=\thetab^{(t+1)}$.

The mode-$j$ parameter block update uses the quantities introduced in Section~\ref{sec:complex-sfa}. Define the current mode-specific covariance estimates and their Cholesky factors by
\begin{align*}
  \Sigmab_\ell^{[t,j-1]}
  =
  \Lambdab_\ell^{[t,j-1]}
  \Lambdab_\ell^{[t,j-1]\star}
  +
  \Psib_\ell^{[t,j-1]},
  \qquad
    \Lb_\ell^{[t,j-1]}
  =
  \operatorname{chol}
  \left(
    \Sigmab_\ell^{[t,j-1]}
  \right),
  \qquad
  \ell=1,\ldots,d.
\end{align*}
Following \eqref{eq:mode-wise-whitening-matrices}, set $\Mb_{\ell \mid j}^{[t,j-1]} = \left(\Lb_\ell^{[t,j-1]}\right)^{-1}$ for $\ell\neq j$ and $\Mb_{j \mid j}^{[t,j-1]}=\Ib_{p_j}$. Substituting these matrices into \eqref{eq:whitened-mode-j-array} gives
\begin{align}
  \label{eq:ssfa-whitened-mode-j-array}
  \breve{\Xcal}_{(j)}^{[t]}
  =
  \Xcal
  \times_1\Mb_{1\mid j}^{[t,j-1]}
  \cdots
  \times_d\Mb_{d\mid j}^{[t,j-1]}
  \times_{d+1}\Ib_n,
\end{align}
and let $\breve{\Xb}_{(j)}^{[t]}\in\CC^{p_j\times n_j}$ denote its mode-$j$ matricization. The mode-$j$ E-step computes the conditional moments
\begin{align}
  \label{eq:ssfa-e-step-stats}
  \Zb_{(j)}^{[t]}
  =
  \EE\left(
    \Zb_{(j)}
    \mid
    \breve{\Xb}_{(j)}^{[t]},
    \thetab^{[t,j-1]}
  \right), \qquad
  \Sbb_{j,zz}^{[t]}
  =
  \EE\left(
    \Zb_{(j)}\Zb_{(j)}^\star
    \mid
    \breve{\Xb}_{(j)}^{[t]},
    \thetab^{[t,j-1]}
  \right).
\end{align}
Their analytic forms are obtained from the corresponding vector-valued results by replacing $(\Xb,\Zb,\Lambdab,\Psib,n)$ with $(\breve{\Xb}_{(j)}^{[t]},\Zb_{(j)},\Lambdab_j,\Psib_j,n_j)$. Appendix~\ref{sec:reduced_form_pxem_obj} provides additional details.

We next derive the sparse loading update for mode $j$. Applying the reduced-form PX-EM representation in \eqref{eq:penalized-reduced-loading-objective} to the latent-variable model in \eqref{eq:sfa-j} yields the unpenalized mode-$j$ reduced loading estimate
\begin{align}
  \label{eq:A-j-definition}
  \Ab_j^{[t]}
  =
  \frac{1}{\sqrt{n_j}} \,
  \breve{\Xb}_{(j)}^{[t]}
  \Zb_{(j)}^{[t]\star}
\left\{  \operatorname{chol} \left(\Sbb_{j,zz}^{[t]} \right) \right\}^{-\star}.
\end{align}
The corresponding mode-$j$ reduced loading loss is
\begin{align}
  \label{eq:5-loss}
  \ell_j
  \left(
    \Lambdab_j
    \mid
    \thetab^{[t,j-1]}
  \right)
  =
  \operatorname{tr}
  \left[
    \left(
      \Lambdab_j-\Ab_j^{[t]}
    \right)^\star
    \left(
      \Psib_j^{[t,j-1]}
    \right)^{-1}
    \left(
      \Lambdab_j-\Ab_j^{[t]}
    \right)
  \right].
\end{align}
We update the mode-$j$ loading matrix by solving
\begin{align}
  \label{eq:5}
  \Lambdab_j^{[t,j]}
  =
  \argmin_{\Lambdab_j}
  \left\{
    \ell_j
    \left(
      \Lambdab_j
      \mid
      \thetab^{[t,j-1]}
    \right)
    +
    \rho
    \|\Lambdab_j\|_{1,\CC}
  \right\}, \qquad \|\Lambdab_j\|_{1,\CC} =\sum_{c=1}^{k_j}\sum_{r=1}^{p_j}|\lambda_{j,rc}|,
\end{align}
where $\rho\geq0$ controls the degree of sparsity. Applying the complex soft-thresholding operator in \eqref{eq:sparse_lam_piecewise_update} gives
\begin{align}
  \label{eq:lam-soft-thresh-j}
  \lambda_{j,rc}^{[t,j]}
  =
  \mathcal{T}_{
    \rho\psi_{j,rr}^{[t,j-1]}/2
  }
  \left(
    a_{j,rc}^{[t]}
  \right),
  \qquad
  r=1,\ldots,p_j,
  \quad
  c=1,\ldots,k_j,
\end{align}
where $a_{j,rc}^{[t]}$ is the $(r,c)$th entry of $\Ab_j^{[t]}$ and $\psi_{j,rr}^{[t,j-1]}$ is the $r$th diagonal entry of $\Psib_j^{[t,j-1]}$. Each nonzero loading preserves the phase of $a_{j,rc}^{[t]}$, and all entries of $\Lambdab_j$ may be updated simultaneously.

We update the diagonal residual covariance matrix using the sparse loading estimate following \eqref{eq:main-3} as
\begin{align}
  \label{eq:ssfa-psi-update}
  \Psib_j^{[t,j]}
  =
  \left[
    \frac{1}{n_j}
    \breve{\Xb}_{(j)}^{[t]}
    \breve{\Xb}_{(j)}^{[t]\star}
    -
        \Lambdab_j^{[t,j]}     \Lambdab_j^{[t,j] \star}
  \right]
  \circ
  \Ib_{p_j}.
\end{align}
This update retains the sparsity structure of $\Lambdab_j^{[t,j]}$. It estimates $\Psib_j^{[t,j]}$ as the diagonal component of the empirical covariance of the whitened mode-$j$ matricization $\breve{\Xb}_{(j)}^{[t]}$ after removing the covariance explained by the retained sparse mode-$j$ factor structure. When $\rho=0$, \eqref{eq:lam-soft-thresh-j} gives $\Lambdab_j^{[t,j]}=\Ab_j^{[t]}$, and \eqref{eq:ssfa-psi-update} reduces to the unpenalized residual covariance update. During iteration $t+1$, the algorithm cycles sequentially through the $d$ modes, moving from $\thetab^{[t,0]}=\thetab^{(t)}$ to $\thetab^{[t,d]}=\thetab^{(t+1)}$. The mode-wise cycle is repeated until the convergence criterion is satisfied. Algorithm~\ref{algo} summarizes the complete PX-EM procedure.

\begin{algorithm}[!th]
\caption{Sparse Separable Factor Analysis Estimation Procedure}
\label{algo}
  \begin{algorithmic}
    \For{$j = 1, \ldots, d$}
      \State Initialize estimates $\Lambdab_j^{[0]}$ and $\Psib_j^{[0]}$ via pseudo random number generator.
    \EndFor

    \Repeat
      \For{$j = 1, \ldots, d$}
        \State Let $\Lambdab_l^{[t,j-1]}$ and $\Psib_l^{[t,j-1]}$ be the most up-to-date estimates of $\Lambdab_l$ and $\Psib_l$ for
        \State $l = 1, \ldots, d$.
        \State \textbf{1.)} Compute the Tucker product in \eqref{eq:ssfa-whitened-mode-j-array}:
        \State \hspace{7.2mm} $\breve{\Xcal}_{(j)}^{[t]}
          =
          \Xcal
          \times_1\Mb_{1\mid j}^{[t,j-1]}
          \cdots
          \times_d\Mb_{d\mid j}^{[t,j-1]}
          \times_{d+1}\Ib_n$.
        \State \textbf{2.)} Reshape $\breve{\Xcal}_{(j)}^{[t]}$ to the mode-$j$ matricization $\breve{\Xb}_{(j)}^{[t]}$.
        \State \textbf{3.) E step:}
        \State \hspace{3.6mm} Compute
        \State \hspace{7.2mm} $\Fb_{(j)}^{[t]} = \Ib_{k_j} + \Lambdab_j^{[t,j-1]\star} \left(\Psib_j^{[t,j-1]}\right)^{-1} \Lambdab_j^{[t,j-1]}$,
        \vspace{1mm}
        \State \hspace{7.2mm} $\Zb_{(j)}^{[t]} = \left(\Fb_{(j)}^{[t]}\right)^{-1} \Lambdab_j^{[t,j-1]\star} \left(\Psib_j^{[t,j-1]}\right)^{-1} \breve{\Xb}_{(j)}^{[t]}$,
        \vspace{1mm}
        \State \hspace{7.2mm} $\Sbb_{jzz}^{[t]} = n_j \left(\Fb_{(j)}^{[t]}\right)^{-1} + \Zb_{(j)}^{[t]} \Zb_{(j)}^{[t]\star}$.
        \State \textbf{4.) Penalized loading step:}
        \State \hspace{3.6mm} Compute $\Ab_j^{[t]}
          =
          \frac{1}{\sqrt{n_j}} \,
          \breve{\Xb}_{(j)}^{[t]}
          \Zb_{(j)}^{[t]\star}
          \left\{  \operatorname{chol} \left(\Sbb_{j,zz}^{[t]} \right) \right\}^{-\star}$.
        \State \hspace{3.6mm} Set $\lambda_{j,rc}^{[t,j]}
          =
          \mathcal{T}_{
            \rho\psi_{j,rr}^{[t,j-1]}/2
          }
          \left(a_{j,rc}^{[t]}\right),
          \quad r = 1, \ldots, p_j, \quad c = 1, \ldots, k_j$.
        \State \textbf{5.) Plug-in residual covariance step:}
        \State \hspace{3.6mm} Set $\Psib_j^{[t,j]} = \left[ \left( \frac{1}{n_j} \breve{\Xb}_{(j)}^{[t]} \breve{\Xb}_{(j)}^{[t]\star} \right) - \Lambdab_j^{[t,j]} \Lambdab_j^{[t,j]\star} \right] \circ \Ib_{p_j}$.
      \EndFor
    \Until{convergence}
  \end{algorithmic}
\end{algorithm}

\subsection{Scale identifiability of the mode-specific covariance matrices}
\label{sec:identifiability_balancing}

Under the SSFA model in \eqref{eq:SSFA_model_intro}, if $\Xcal \sim \CC\Ncal_{p_1\times\cdots\times p_d} \left(\zero, \Sigmab_1,\ldots,\Sigmab_d \right)$, then
then $  \operatorname{Cov} \left\{\operatorname{vec}(\Xcal) \right\} = \Sigmab_d\otimes\cdots\otimes\Sigmab_1$. The Kronecker product is identifiable, but its mode-specific factors are identifiable only up to multiplicative rescaling. In particular, let $c_1,\ldots,c_d>0$ satisfy $\prod_{j=1}^d c_j=1$. Then,
\begin{align*}
  (c_d\Sigmab_d)\otimes\cdots\otimes(c_1\Sigmab_1)
  =
    \prod_{j=1}^d c_j
  \left(
    \Sigmab_d\otimes\cdots\otimes\Sigmab_1
  \right)
  =
  \Sigmab_d\otimes\cdots\otimes\Sigmab_1.
\end{align*}
The factor-analytic form is also preserved because $  c_j\Sigmab_j = \left(c_j^{1/2}\Lambdab_j \right) \left(c_j^{1/2}\Lambdab_j \right)^\star + c_j\Psib_j$.

We remove this scale indeterminacy by requiring the smallest diagonal entries of the rescaled residual covariance matrices to be equal. The following proposition shows that this condition selects a unique rescaling.

\begin{proposition}
\label{prop:balancing_factor_identifiability_main}

For $j=1,\ldots,d$, suppose $\Sigmab_j=\Lambdab_j\Lambdab_j^\star+\Psib_j$, where $\Psib_j$ is diagonal with positive diagonal entries. There exists a unique vector $(c_1,\ldots,c_d)$ satisfying $c_j>0$ for $j=1,\ldots,d$ and $\prod_{j=1}^d c_j=1$ such that the rescaled parameters $  \widetilde{\Lambdab}_j = c_j^{1/2}\Lambdab_j$ and $  \widetilde{\Psib}_j = c_j\Psib_j$ satisfy
\begin{align}
  \min\left\{
    \operatorname{diag}(\widetilde{\Psib}_1)
  \right\}
  =
  \cdots
  =
  \min\left\{
    \operatorname{diag}(\widetilde{\Psib}_d)
  \right\}.
  \label{eq:min-diag-eq}
\end{align}
The rescaling constants are
\begin{align*}
  c_j
  =
  \frac{\overline{\psi}_{\min}}{\psi_{j,\min}},
  \qquad
  \psi_{j,\min}
  =
  \min\left\{
    \operatorname{diag}(\Psib_j)
  \right\},
  \qquad
  \overline{\psi}_{\min}
  =
  \left(
    \prod_{\ell=1}^d\psi_{\ell,\min}
  \right)^{1/d}.
\end{align*}
Moreover, if
$\widetilde{\Sigmab}_j = \widetilde{\Lambdab}_j\widetilde{\Lambdab}_j^\star + \widetilde{\Psib}_j$, then $\widetilde{\Sigmab}_d\otimes\cdots\otimes \widetilde{\Sigmab}_1 = \Sigmab_d\otimes\cdots\otimes\Sigmab_1$.
\end{proposition}
The proof of Proposition~\ref{prop:balancing_factor_identifiability_main} is provided in Appendix \ref{app:identifiability_balancing}.

Following Proposition~\ref{prop:balancing_factor_identifiability_main}, we rescale the loading and residual covariance matrices at the end of each iteration of Algorithm~\ref{algo}. At the end of iteration $t$, define
\begin{align}
  \psi_{j,\min}^{(t)}
  =
  \min\left\{
    \operatorname{diag}\left(
      \Psib_j^{(t)}
    \right)
  \right\},
  \qquad
  c_j^{(t)}
  =
  \frac{
    \overline{\psi}_{\min}^{(t)}
  }{
    \psi_{j,\min}^{(t)}
  },
  \qquad
  j=1,\ldots,d,
  \label{eq:iden-cs-1}
\end{align}
where $  \overline{\psi}_{\min}^{(t)} = \left(\prod_{j=1}^d \psi_{j,\min}^{(t)} \right)^{1/d}$.  The rescaled parameters are
\begin{align}
  \widetilde{\Lambdab}_j^{(t)}
  =
  \left(
    c_j^{(t)}
  \right)^{1/2}
  \Lambdab_j^{(t)},
  \qquad
  \widetilde{\Psib}_j^{(t)}
  =
  c_j^{(t)}
  \Psib_j^{(t)},
  \qquad
  j=1,\ldots,d.
  \label{eq:iden-cs-2}
\end{align}
Their minimum diagonal entries satisfy
\begin{align*}
  \min\left\{
    \operatorname{diag}\left(
      \widetilde{\Psib}_j^{(t)}
    \right)
  \right\}
  &=
  c_j^{(t)}
  \psi_{j,\min}^{(t)}
  =
  \overline{\psi}_{\min}^{(t)},
  \qquad
  j=1,\ldots,d.
\end{align*}
The corresponding covariance matrices are
\begin{align*}
  \widetilde{\Sigmab}_j^{(t)}
  &=
  \widetilde{\Lambdab}_j^{(t)}
  \widetilde{\Lambdab}_j^{(t)\star}
  +
  \widetilde{\Psib}_j^{(t)}
  =
  c_j^{(t)}
  \Sigmab_j^{(t)},
  \qquad
  j=1,\ldots,d.
\end{align*}
Because $\prod_{j=1}^d c_j^{(t)}=1$, this transformation leaves the overall Kronecker covariance unchanged. We then replace the current parameter values with their rescaled counterparts.

Proposition~\ref{prop:balancing_factor_identifiability_main} shows that the constraint in \eqref{eq:min-diag-eq} determines a unique rescaling vector $(c_1,\ldots,c_d)$, yielding a unique tuple $(c_1\Sigmab_1,\ldots,c_d\Sigmab_d)$ within each scale-equivalence class. This balancing step also facilitates the use of a common penalty parameter across modes. By \eqref{eq:lam-soft-thresh-j}, the threshold applied to row $r$ of the mode-$j$ loading matrix is $\rho\psi_{j,rr}^{(t)}/2$. The transformation in \eqref{eq:iden-cs-2} places the residual covariance matrices on a common reference scale, making regularization more comparable across modes.

\subsection{Heuristic construction of the tuning-parameter grid}
\label{sec:rho_grid}

We use a common tuning parameter $\rho$ across modes. The balancing transformation in \eqref{eq:iden-cs-2} places the mode-specific parameters on a common scale and permits a unified regularization path. The thresholding rule in \eqref{eq:lam-soft-thresh-j} sets $\lambda_{j,rc}^{[t,j]}=0$ whenever
\begin{align}
  \label{eq:zero_out_eq_rho_grid}
  |a_{j,rc}^{[t]}|
  \le
  \psi_{j,rr}^{[t,j-1]} \, \rho / 2,
  \qquad r = 1, \ldots, p_j, \qquad c = 1, \ldots, k_j, \qquad j = 1, \ldots, d.
\end{align}
For fixed $t$, $j$, $r$, and $c$, the boundary between zero and nonzero updates is $\rho=2|a_{j,rc}^{[t]}|/\psi_{j,rr}^{[t,j-1]}$. Because this boundary varies across loading entries and iterations, it does not directly define a fixed tuning-parameter grid. We therefore use reference estimates to construct the grid heuristically.

First, we replace $\psi_{j,rr}^{[t,j-1]}$ in \eqref{eq:zero_out_eq_rho_grid} with the corresponding estimate under the fully sparse model with  $\Lambdab_1=\cdots=\Lambdab_d=\zero$. Let $\widehat{\Psib}_j^0$ denote the resulting estimate of $\Psib_j$ in this model. With all loading matrices fixed at zero, the residual covariance update in \eqref{eq:ssfa-psi-update} becomes
\begin{align*}
  \widehat{\Psib}_j^0
  =
  \frac{1}{n_j}
  \breve{\Xb}_{(j)}
  \breve{\Xb}_{(j)}^\star
  \circ
  \Ib_{p_j},
  \qquad
  j=1,\ldots,d.
\end{align*}
We obtain $\widehat{\Psib}_1^0,\ldots,\widehat{\Psib}_d^0$ by running Algorithm~\ref{algo} with all loading matrices fixed at zero until convergence. In our numerical experiments, this run typically converges in fewer than five iterations.

Next, we fit the unpenalized model by setting $\rho=0$. Let $\widehat{\Lambdab}_j^0$ and $\widehat{\Ab}_j^0$ denote the converged values of $\Lambdab_j^{[t,j]}$ \eqref{eq:lam-soft-thresh-j} and $\Ab_j^{[t]}$ \eqref{eq:A-j-definition}, respectively. If fitting the model at $\rho=0$ is computationally infeasible, we approximate these quantities by running Algorithm~\ref{algo} to convergence at a prespecified small positive value of $\rho$. These estimates provide unshrunk or approximately unshrunk reference loading values. When $\rho=0$, \eqref{eq:lam-soft-thresh-j} gives $\widehat{\lambda}_{j,rc}^0=\widehat a_{j,rc}^0$. We use the converged unpenalized loading estimate as a reference for the iteration-specific pre-thresholding loading and the fully sparse residual estimate as a reference for the iteration-specific residual variance:
\begin{align*}
  |a_{j,rc}^{[t]}|
  \approx
  |\widehat a_{j,rc}^0|
  =
  |\widehat{\lambda}_{j,rc}^0|,
  \qquad
  \psi_{j,rr}^{[t,j-1]}
  \approx
  \widehat{\psi}_{j,rr}^0.
\end{align*}
Substitution into \eqref{eq:zero_out_eq_rho_grid} gives the approximate boundary $|\widehat{\lambda}_{j,rc}^0| \approx \widehat{\psi}_{j,rr}^0\rho/2$.

These approximations motivate the upper endpoint $\rho_{\max}$ of the common tuning-parameter grid. For each mode $j=1,\ldots,d$, let $\rho_{j(\max)}$ denote the approximate smallest value of $\rho$ that shrinks every entry of $\Lambdab_j$ to zero. The preceding boundary implies that
\begin{align}
  \label{eq:rho-grid-max}
  \rho_{j(\max)}
  =
  \max_{1\leq r\leq p_j, \, 1\leq c\leq k_j}
  \frac{
    2|\widehat{\lambda}_{j,rc}^{0}|
  }{
    \widehat{\psi}_{j,rr}^0
  },
  \qquad
  j=1,\ldots,d.
\end{align}
We define the endpoints of the common grid by
\begin{align}
  \label{eq:rho-max-min}
  \rho_{\max}
  =
  \min_{1\leq j\leq d}
  \rho_{j(\max)},
  \qquad
  \rho_{\min}
  =
  \epsilon\rho_{\max},
  \qquad
  \epsilon\in(0,1).
\end{align}
Therefore, $\rho_{\max}$ is the smallest mode-specific threshold and is approximately sufficient to shrink all loading entries to zero in at least one mode. The value of $\epsilon$ is user-specified; we set $\epsilon=0.1$ in our implementation. We construct a decreasing sequence of values evenly spaced on the logarithmic scale from $\rho_{\max}$ to $\rho_{\min}$.

\section{Simulated data analysis}
\label{sec:sim_analysis}

\subsection{Simulation setup}
\label{sec:sim_setup}

The simulation study considers six settings that vary the order and dimension of the complex-valued observations. For a $d$-mode observation array with dimension $\mathbf p=(p_1,\ldots,p_d)$, the mode-specific loading matrices satisfy $\Lambdab_j\in\CC^{p_j\times k_j}$, where $\mathbf k=(k_1,\ldots,k_d)$ denotes the vector of loading ranks. The residual error covariance matrices $\Psib_j$ are $p_j\times p_j$ diagonal matrices with positive entries for $j=1,\ldots,d$. Table~\ref{tab:simulation-scenarios} gives the detailed dimensions for six settings: two vector-valued, two matrix-valued, and two three-dimensional array-valued cases. 

\begin{table}[ht]
\caption{Simulation settings. For each setting, $\Lambdab_j\in\CC^{p_j\times k_j}$ and $\Psib_j\in\RR_{+}^{p_j\times p_j}$ is diagonal for $j=1,\ldots,d$.}
\label{tab:simulation-scenarios}
\centering
\renewcommand{\arraystretch}{0.90}
\begin{tabular}{ccccc}
\hline
Setting & Observation type & Observation size & Dimension vector $\mathbf p$ & Rank vector $\mathbf k$ \\
\hline
1 & Vector & $\xb_i\in\CC^{25}$ & $(25)$ & $(3)$ \\
2 & Vector & $\xb_i\in\CC^{50}$ & $(50)$ & $(4)$ \\
3 & Matrix & $\Xb_i\in\CC^{25\times 25}$ & $(25,25)$ & $(4,3)$ \\
4 & Matrix & $\Xb_i\in\CC^{50\times 50}$ & $(50,50)$ & $(4,3)$ \\
5 & Array & $\Xcal_i\in\CC^{25\times 25\times 25}$ & $(25,25,25)$ & $(4,3,2)$ \\
6 & Array & $\Xcal_i\in\CC^{50\times 50\times 50}$ & $(50,50,50)$ & $(4,3,2)$ \\
\hline
\end{tabular}
\end{table}

The parameters were simulated as follows. For each setting, the diagonal entries of the $\Psib_j$ matrices were sampled independently from a $\mathrm{Unif}(0.1,0.5)$ distribution. For the nonzero entries of the $\Lambdab_j$ matrices, the moduli were sampled independently from $\mathrm{Unif}(0.5,2)$, and the phases were sampled independently from $\mathrm{Unif}(-3,3)$. Each $\Lambdab_j$ was constructed so that each column contained a contiguous block of nonzero entries, with block sizes approximately equal across columns. Consecutive columns were assigned two overlapping nonzero entries. After generating $\Lambdab_j$ and $\Psib_j$ for $j=1,\ldots,d$, we computed the mode-specific covariance matrices as $\Sigmab_j = \Lambdab_j\Lambdab_j^\star + \Psib_j$ for $ j=1,\ldots,d$, and used them to generate the simulated observations as described below. Figure \ref{fig:Lambda_Sigma_example} shows an example of the moduli of a sparse loading matrix and the corresponding mode-specific covariance matrix.

\begin{figure}[htbp]
  \centering

  \begin{subfigure}[t]{0.38\textwidth}
    \centering
    \includegraphics[
      width=\linewidth,
      height=0.32\textheight,
      keepaspectratio
    ]{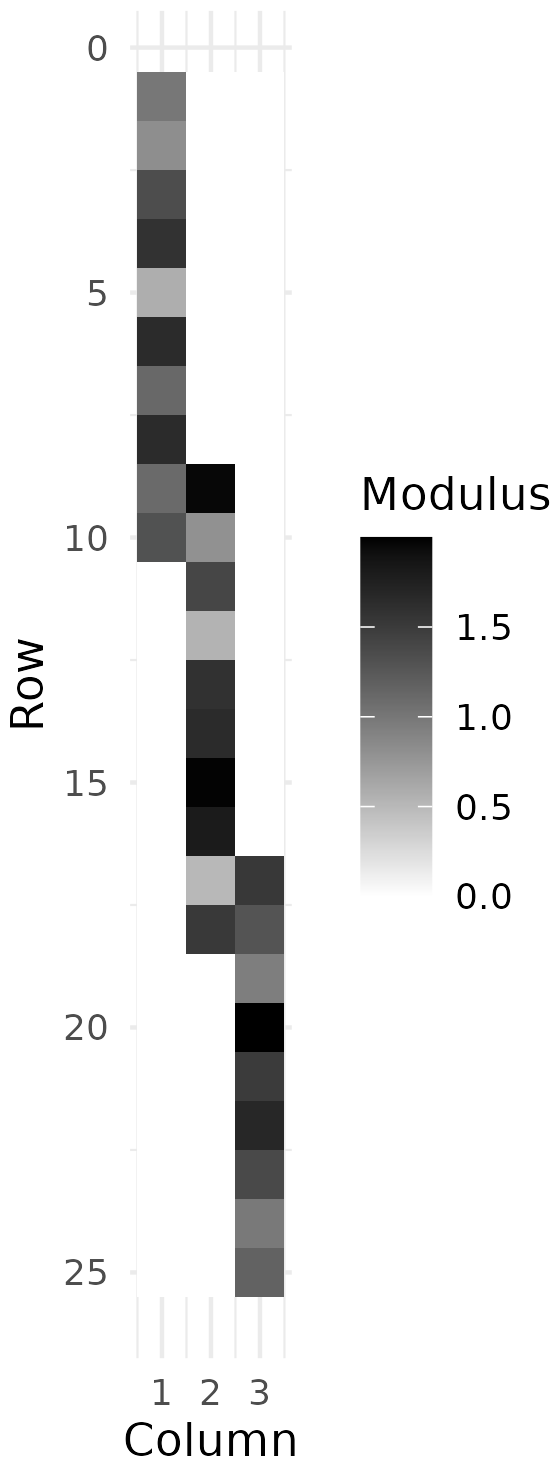}
    \caption{Rank-three sparse loading matrix.}
    \label{fig:Lambda_example}
  \end{subfigure}
  \hfill
  \begin{subfigure}[t]{0.60\textwidth}
    \centering
    \includegraphics[
      width=\linewidth,
      height=0.32\textheight,
      keepaspectratio
    ]{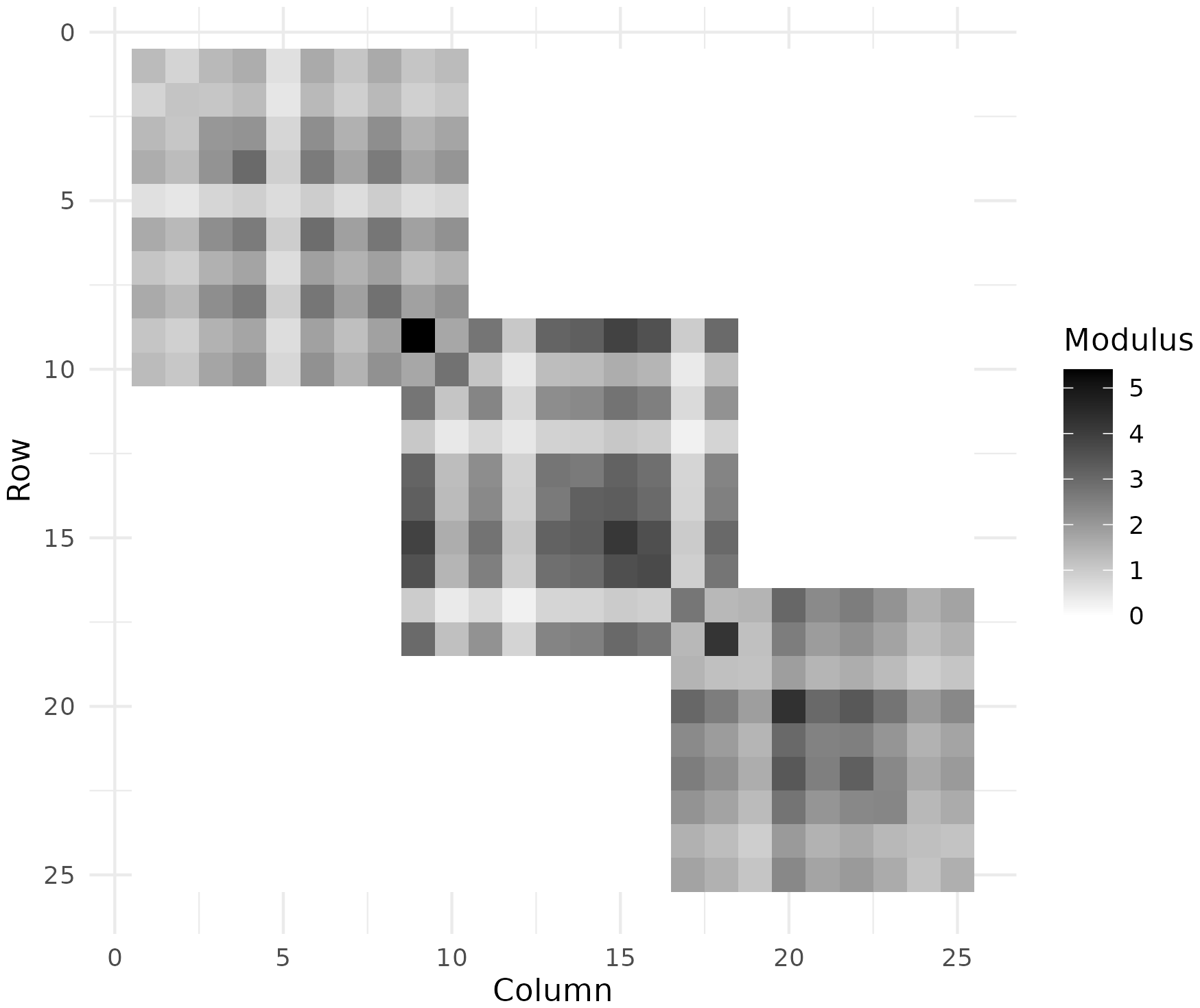}
    \caption{Corresponding mode-specific covariance matrix.}
    \label{fig:Sigma_example}
  \end{subfigure}

  \caption{Example of the moduli of a rank-three sparse loading matrix and the corresponding mode-specific covariance matrix used in the simulation study.}
  \label{fig:Lambda_Sigma_example}
\end{figure}

To assess robustness, we simulated the observations from a misspecified model relative to the circular complex normal distribution assumed by SSFA. The latent  factor vectors or arrays were generated from a real-valued standard normal distribution and then multiplied by the mode-specific complex Cholesky factors using the Tucker product. This transformation preserves the target Hermitian covariance structure but generally induces a nonzero pseudo-covariance matrix, so the resulting data need not follow a circular complex normal distribution. Let $p=\prod_{j=1}^d p_j$ and  $\Lb_j\in\CC^{p_j\times p_j}$ denote the Cholesky factor of the mode-specific covariance matrix $\Sigmab_j=\Lb_j\Lb_j^\star$. For each observation $i=1,\ldots,n$, we generated $\zb_i\in\RR^p$ from a real-valued standard normal distribution. For matrix-valued
and array-valued datasets, we reshaped $\zb_i$ into $\Zb_i\in\RR^{p_1\times p_2}$ and $\Zcal_i\in\RR^{p_1\times p_2\times p_3}$, respectively.
For $i=1,\ldots,n$, $i$th  vector-, matrix-, and array-valued observations are generated as
\begin{align}
   \xb_i = \Lb_1\zb_i, \qquad
   \Xb_i = \Lb_1\Zb_i\Lb_2^\top, \qquad
   \Xcal_i
=
\Zcal_i \times_1 \Lb_1 \times_2 \Lb_2 \times_3 \Lb_3.
\end{align}

For each simulation setting, we varied the sample size over $n\in\{5,10,20,40\}$. We also considered $n=100$ for the vector-valued settings, giving 26 combinations of setting and sample size. Each combination was replicated 100 times. The larger sample size was included only for the vector-valued settings because, for higher-order observations, the mode-$j$ matricization has effective sample size $np_{-j}$, where $p_{-j}=\prod_{\ell\neq j}p_\ell$. Therefore, moderate values of $n$ can still provide substantial information for estimating mode-specific parameters. In contrast, the vector-valued settings remain high-dimensional relative to the sample size, so a larger value of $n$ is needed to assess covariance estimation performance.

\subsection{Competing methods}
\label{sec:sim_competing_methods}

We compare SSFA with three covariance estimators for vector-valued data. The competitors are the empirical covariance estimator, complex PCA, and sparse PCA applied after embedding the complex-valued data in the real domain. Because we are unaware of directly comparable covariance estimation methods for complex-valued arrays, we do not include an array-variate competitor. Each competing method estimates the covariance matrix of the vectorized observations. We therefore vectorize the observations in every simulation setting before fitting these methods. We first describe the implementation of SSFA and then those of the three competitors.

We fit SSFA to each simulated dataset over a grid of 51 tuning parameter values, consisting of $\rho=0$ and 50 positive values constructed using the heuristic in Section~\ref{sec:rho_grid}. The algorithm was initialized with five columns in each loading matrix $\Lambdab_j$ for $j=1,\ldots,d$. For each value of $\rho$, we computed the EBIC as described in Appendix~\ref{sec:ebic_description}. We selected the largest value of $\rho$ whose EBIC was within one median absolute deviation of the minimum EBIC across the 51 fitted models. The selected model provided $\widehat{\Lambdab}_j$, $\widehat{\Psib}_j$, and $\widehat{\Sigmab}_j$ for $j=1,\ldots,d$. We estimated the effective rank of $\Lambdab_j$ by the number of columns of $\widehat{\Lambdab}_j$ containing at least one nonzero entry.

We next describe the empirical covariance and complex PCA competitors. For each simulated dataset, we form the matrix $\Xb\in\CC^{p\times n}$, where $p=\prod_{j=1}^d p_j$ and the $i$th column is the vectorized $i$th observation. The simulated data are centered, so the empirical covariance estimator is $\widehat{\Sigmab}_{\mathrm{emp}} = n^{-1}\Xb\Xb^\star$. Complex PCA imposes an additional low-rank assumption on the covariance structure. For a fixed rank $k$, we compute the eigendecomposition of $\widehat{\Sigmab}_{\mathrm{emp}}$ and estimate $\Lambdab$ using the leading $k$ eigenvectors scaled by the square roots of their corresponding eigenvalues. The estimated residual error covariance satisfies $\widehat \Psib = \widehat \sigma^2 \Ib_p$, where $\widehat \sigma^2$ is estimated as the average of the remaining $(p-k)$ eigenvalues \citep{bishop2006pattern}. We fit complex PCA over  $k\in\{1,2,\ldots,10\}$ and select $k$ using the EBIC criterion described in Appendix~\ref{sec:ebic_description}. The selected model gives the estimates $\widehat{\Lambdab}$, $\widehat{\sigma}^2$, and $\widehat{\Sigmab}_{\mathrm{PCA}}$.

We use sparse PCA as an adapted vectorized sparse competitor. Sparse PCA is designed for real-valued vector data and does not directly provide a covariance estimator for complex-valued observations. Therefore, we use a construction that embeds the data in the real domain and maps the estimated sparse loading directions back to the complex domain. For $\Xb\in\CC^{p\times n}$, define $\breve{\Xb}\in\RR^{2p\times n}$ by taking $\operatorname{Re}(\Xb)$ and $\operatorname{Im}(\Xb)$ as its first and second row blocks, respectively. For each rank $k\in\{1,2,3,4,5\}$ and sparsity parameter
$\alpha\in\{10^{-2},10^{-1},1,2\}$, we apply the \texttt{SparsePCA} function in \texttt{scikit-learn} to $\breve{\Xb}$ \citep{zou2006sparse,mairal2009online,JMLR:v12:pedregosa11a}. We select the fitted model with the smallest EBIC. The resulting real-valued loading matrix is mapped to a complex-valued sparse loading matrix $\widehat{\Ub}\in\CC^{p\times k}$ by treating its first $p$ rows as the real part and its last $p$ rows as the imaginary part.

We estimate the corresponding complex-valued scores by $  \widehat{\Zb} = \left(\widehat{\Ub}^{\star}\widehat{\Ub} + \lambda \Ib_k \right)^{-1} \widehat{\Ub}^{\star}\Xb$, where $\lambda$ is a small ridge-stabilization parameter. Let $\widehat{\Gammab}=n^{-1}\widehat{\Zb}\widehat{\Zb}^{\star}$. To retain the sparsity pattern of each loading column, we scale the columns individually and define $  \widehat{\Lambdab} = \widehat{\Ub} \operatorname{diag}(\widehat{\Gammab})^{1/2}$. We then construct a covariance estimate of the form $  \widehat{\Sigmab}_{\mathrm{SPCA}} = \widehat{\Lambdab}\widehat{\Lambdab}^{\star} + \widehat{\sigma}^{2}\Ib_p$, where $\widehat{\sigma}^{2}$ is set to the mean of the nonzero eigenvalues of $\widehat{\Sigmab}_{\mathrm{emp}} -\widehat{\Lambdab}\widehat{\Lambdab}^{\star}$. This construction is used only to obtain a sparse vectorized benchmark. It is not intended as a native complex-valued sparse PCA model.

\subsection{Loading matrix estimation accuracy}
\label{subspace-err-metrics}

We evaluate estimation accuracy of Algorithm~\ref{algo} using two complementary metrics. The first is a standard subspace error based on the Davis--Kahan theorem \citep{yu2015useful}. For each mode $j=1, \ldots, d$, let $\Pb_j = \Lambdab_j(\Lambdab_j^\star \Lambdab_j)^{+}\Lambdab_j^\star$ be the orthogonal projector onto the column space of $\Lambdab_j$, where $+$ denotes the Moore--Penrose pseudoinverse. We similarly define $\widehat{\Pb}_j$ as the orthogonal projection matrix onto the column space of $\widehat \Lambdab_j$. The null-space projection error is
\begin{align}
  \label{eq:null_proj_err}
  \mathrm{err}_{\Pb_j}^2
  =
  \|(\Ib_{p_j}-\widehat{\Pb}_j)\Pb_j\|_F^2,
  \qquad j=1,\ldots,d,
\end{align}
where $\|\cdot\|_F$ denotes the Frobenius norm. This metric quantifies the amount of the true column space of $\Lambdab_j$ that lies outside the estimated column space of $\widehat{\Lambdab}_j$. However, it does not account for the relative importance of the singular directions of $\Lambdab_j$, as measured by their squared singular values.

The second metric measures the fraction of the important principal directions of the true loading matrix captured by the estimated column space. This metric is less standard but useful because the singular directions of $\Lambdab_j$ differ in their contributions to the covariance structure. We define the weighted null-space projection error by
\begin{align}
  \label{eq:null_err}
  \mathrm{err}_{\Lambdab_j}^2
  =
  {\|(\Ib_{p_j}-\widehat{\Pb}_j)\Lambdab_j\|_F^2} /
       {\|\Lambdab_j\|_F^2},
  \qquad j=1,\ldots,d,
\end{align}
where we assume that $\Lambdab_j \neq \zero$. This metric measures the fraction of the squared Frobenius norm of $\Lambdab_j$ that lies outside the estimated column space of $\widehat{\Lambdab}_j$. Because$\|\Lambdab_j\|_F^2$ equals the sum of its squared singular values, the $i$th summand in $\mathrm{err}_{\Lambdab_j}^2$ has a weight proportional to the squared $i$th singular value of $\Lambdab_j$. Appendices~\ref{sec:supp-null-proj-true-proj} and \ref{sec:supp-null-proj-lambda} give the exact representations of $\mathrm{err}_{\Pb_j}^2$ in \eqref{eq:null_proj_err} and $\mathrm{err}_{\Lambdab_j}^2$ in \eqref{eq:null_err}, respectively.

\subsection{Simulation results}

We summarize simulation results using three performance metrics. The first two are the subspace error metrics: $\mathrm{err}_{\Pb_j}^2$ in \eqref{eq:null_proj_err} and  $\mathrm{err}_{\Lambdab_j}^2$ in \eqref{eq:null_err}. These metrics evaluate recovery of the column space of the mode-specific loading matrix for $j=1,\ldots,d$. The third metric is the relative Frobenius error for covariance estimation, defined by
\begin{align}
  \label{eq:err}
  \mathrm{err}_{\xib}
  =
  \frac{\|\widehat{\xib}-\xib\|_F}{\|\xib\|_F}, \qquad \xib \in \{\Sigmab_1,\ldots,\Sigmab_d,\Sigmab\}, \qquad \Sigmab=\Sigmab_d\otimes\cdots\otimes\Sigmab_1.
\end{align}
We first compare all four methods using $\mathrm{err}_{\Sigmab}$ because the competing methods are applied only to the vectorized observations. For complex PCA and sparse PCA, loading-matrix subspace errors are restricted to the vector-valued simulations. For SSFA, we also report mode-specific covariance errors and loading-matrix subspace errors. For each combination of simulation setting and sample size, we report the median and median absolute deviation of each metric across the 100 simulated datasets.

SSFA and its competitors have similar full covariance errors $\mathrm{err}_{\Sigmab}$ for vector-valued data, but the advantage of SSFA becomes substantial for matrix- and three-dimensional array-valued observations (Table~\ref{table:estimating_sigma_all_methods}). In the vector-valued settings, SSFA, empirical covariance, and complex PCA have comparable errors, whereas sparse PCA performs worse. In the higher-order settings, SSFA has much smaller $\mathrm{err}_{\Sigmab}$ values than the vectorization-based competitors across all sample sizes. The errors generally decrease as $n$ increases. The smaller errors achieved by SSFA indicate that exploiting separability and mode-specific factor structure improves covariance estimation for complex-valued arrays.

\begin{table}[h!]
  \caption{
    Comparison across methods of the median and median absolute deviation of $\text{err}_{\Sigmab}$ defined in \eqref{eq:err}. The median absolute deviations are in parentheses. Empirical denotes the empirical covariance estimation method. A ``--'' indicates that the entry could not be computed because of prohibitive computational cost.
    \label{table:estimating_sigma_all_methods}}
  \resizebox{\textwidth}{!}{%
  \begin{tabular}{clcccc}
    \hline
    Setting & Size & Complex PCA & Sparse PCA & SSFA & Empirical \\
    \hline
    \multirow{5}{8em}{$\xb_i \in \CC^{25}$}
      &   $n=5$ & 0.828 (0.154) & 1.058 (0.219) & 0.811 (0.148) & 0.813 (0.144) \\
      &  $n=10$ & 0.597 (0.103) & 0.651 (0.149) & 0.590 (0.097) & 0.598 (0.100) \\
      &  $n=20$ & 0.427 (0.070) & 0.380 (0.065) & 0.419 (0.063) & 0.429 (0.063) \\
      &  $n=40$ & 0.303 (0.055) & 0.276 (0.039) & 0.301 (0.052) & 0.307 (0.051) \\
      & $n=100$ & 0.197 (0.027) & 0.179 (0.030) & 0.191 (0.028) & 0.199 (0.026) \\
    \hline
    \multirow{5}{8em}{$\xb_i \in \CC^{50}$}
      &   $n=5$ &0.994 (0.186) & 1.620 (0.316) & 0.939 (0.177) & 0.942 (0.181) \\
      &  $n=10$ & 0.743 (0.105) & 0.918 (0.162) & 0.688 (0.085) & 0.698 (0.087) \\
      &  $n=20$ & 0.501 (0.061) & 0.478 (0.084) & 0.488 (0.058) & 0.501 (0.056) \\
      &  $n=40$ & 0.361 (0.046) & 0.314 (0.044) & 0.351 (0.042) & 0.362 (0.044) \\
      & $n=100$ & 0.226 (0.022) & 0.210 (0.028) & 0.220 (0.024) & 0.227 (0.021) \\
    \hline
    \multirow{4}{8em}{$\Xb_i \in \CC^{25 \times 25}$}
      &  $n=5$ & 15.472 (1.344) & 14.273 (1.780) & 0.271 (0.031) & 1.782 (0.184) \\
      & $n=10$ &  8.139 (0.513) &  8.539 (0.530) & 0.203 (0.023) & 1.281 (0.079) \\
      & $n=20$ &  4.250 (0.182) &  4.074 (0.567) & 0.168 (0.020) & 0.913 (0.048) \\
      & $n=40$ &  0.917 (0.048) &  1.398 (0.064) & 0.124 (0.020) & 0.649 (0.037) \\
    \hline
    \multirow{4}{8em}{$\Xb_i \in \CC^{50 \times 50}$}
      &  $n=5$ & 35.103 (3.132) & 7.352 (0.524) & 0.208 (0.028) & 1.990 (0.169) \\
      & $n=10$ & 18.017 (0.999) & 7.793 (0.588) & 0.194 (0.020) & 1.381 (0.079) \\
      & $n=20$ &  9.263 (0.441) & 7.113 (0.342) & 0.139 (0.016) & 0.985 (0.048) \\
      & $n=40$ &  4.723 (0.168) & 3.883 (0.473) & 0.113 (0.015) & 0.698 (0.036) \\
    \hline
    \multirow{4}{8em}{$\Xcal_i \in \CC^{25 \times 25 \times 25}$}
      & $n=5$ & 141.700 (7.482) & - & 0.168 (0.018) & 2.933 (0.144) \\
      & $n=10$ & 72.256 (3.215) & - & 0.145 (0.027) & 2.067 (0.110) \\
      & $n=20$ & 36.810 (1.003) & - & 0.152 (0.035) & 1.443 (0.053) \\
      & $n=40$ & 18.879 (0.601) & - & 0.153 (0.033) & 1.039 (0.033) \\
    \hline
    \multirow{4}{8em}{$\Xcal_i \in \CC^{50 \times 50 \times 50}$}
      &  $n=5$ & - & - & - & - \\
      & $n=10$ & - & - & - & - \\
      & $n=20$ & - & - & - & - \\
      & $n=40$ & - & - & - & - \\
    \hline
  \end{tabular}
  }
\end{table}

Complex PCA and SSFA have similar loading-matrix subspace errors, $\mathrm{err}_{\Pb_1}^2$ and $\mathrm{err}_{\Lambdab_1}^2$, across all sample sizes for the vector-valued simulations (Table~\ref{table:estimating_loading_errors_all_methods}). The empirical covariance estimator is not included because it does not produce loading matrix estimates. For the remaining three methods, the errors generally decrease as $n$ increases. Sparse PCA has substantially larger $\mathrm{err}_{\Pb_1}^2$ and $\mathrm{err}_{\Lambdab_1}^2$ values than complex PCA and SSFA, indicating that imposing sparsity after embedding the data in the real domain does not recover the complex loading structure as accurately as direct estimation in the complex domain. These results indicate that complex PCA and SSFA have comparable loading-space estimation accuracy for vector-valued data.

\begin{table}[h!]
  \centering
  \renewcommand{\arraystretch}{0.90}
  \caption{
    Median and median absolute deviation of the loading-matrix errors $\mathrm{err}_{\Pb_1}$ and $\mathrm{err}_{\Lambdab_1}$ defined in \eqref{eq:null_proj_err} and \eqref{eq:null_err} for the vector-valued simulations. Median absolute deviations are shown in parentheses. \label{table:estimating_loading_errors_all_methods}
  }
  \begin{tabular}{clccc}
    \hline
    \multicolumn{5}{c}{$\mathrm{err}_{\Pb_1}$} \\
    \hline
    Setting & Size & Complex PCA & Sparse PCA & SSFA \\
    \hline
    \multirow{5}{8em}{$\xb_i \in \CC^{25}$}
      &   $n=5$ & 0.664 (0.089) & 1.358 (0.108) & 0.664 (0.086) \\
      &  $n=10$ & 0.389 (0.049) & 1.021 (0.405) & 0.388 (0.038) \\
      &  $n=20$ & 0.275 (0.029) & 0.299 (0.031) & 0.258 (0.029) \\
      &  $n=40$ & 0.191 (0.017) & 0.211 (0.022) & 0.174 (0.016) \\
      & $n=100$ & 0.117 (0.010) & 0.136 (0.012) & 0.103 (0.009) \\
    \hline
    \multirow{5}{8em}{$\xb_i \in \CC^{50}$}
      &   $n=5$ & 0.986 (0.102) & 1.651 (0.097) & 0.984 (0.101) \\
      &  $n=10$ & 0.610 (0.092) & 1.140 (0.304) & 0.580 (0.071) \\
      &  $n=20$ & 0.385 (0.017) & 0.444 (0.053) & 0.366 (0.020) \\
      &  $n=40$ & 0.255 (0.015) & 0.277 (0.022) & 0.241 (0.013) \\
      & $n=100$ & 0.158 (0.010) & 0.177 (0.013) & 0.146 (0.011) \\
    \hline
  \end{tabular}

  \begin{tabular}{clccc}
    \hline
    \multicolumn{5}{c}{$\mathrm{err}_{\Lambdab_1}$} \\
    \hline
    Setting & Size & Complex PCA & Sparse PCA & SSFA \\
    \hline
    \multirow{5}{8em}{$\xb_i \in \CC^{25}$}
      &   $n=5$ & 0.385 (0.053) & 0.768 (0.079) & 0.383 (0.050) \\
      &  $n=10$ & 0.222 (0.028) & 0.527 (0.234) & 0.218 (0.024) \\
      &  $n=20$ & 0.155 (0.019) & 0.168 (0.018) & 0.142 (0.017) \\
      &  $n=40$ & 0.106 (0.009) & 0.117 (0.012) & 0.098 (0.009) \\
      & $n=100$ & 0.065 (0.004) & 0.076 (0.006) & 0.057 (0.004) \\
    \hline
    \multirow{5}{8em}{$\xb_i \in \CC^{50}$}
      &   $n=5$ & 0.486 (0.049) & 0.813 (0.062) & 0.486 (0.047) \\
      &  $n=10$ & 0.299 (0.047) & 0.571 (0.142) & 0.286 (0.036) \\
      &  $n=20$ & 0.189 (0.010) & 0.214 (0.026) & 0.180 (0.010) \\
      &  $n=40$ & 0.125 (0.007) & 0.136 (0.009) & 0.119 (0.006) \\
      & $n=100$ & 0.077 (0.005) & 0.087 (0.005) & 0.072 (0.005) \\
    \hline
  \end{tabular}
\end{table}

A key feature of SSFA, differentiating it from the vectorized competitors, is its ability to estimate mode-specific covariance and loading matrices. Tables~\ref{table:ssfa_mode_one_sim}--\ref{table:ssfa_mode_three_sim} summarize these mode-specific results. Table~\ref{table:ssfa_mode_one_sim} reports mode-one results for all six settings, Table~\ref{table:ssfa_mode_two_sim} reports mode-two results for the matrix-valued and three-dimensional array-valued settings, and Table~\ref{table:ssfa_mode_three_sim} reports mode-three results for the three-dimensional array-valued settings. Across modes, the covariance errors $\mathrm{err}_{\Sigmab_j}$ and loading-matrix errors $\mathrm{err}_{\Lambdab_j}$ and $\mathrm{err}_{\Pb_j}$ generally decrease as $n$ increases.

The lasso penalty on the mode-specific loading matrices also allows SSFA to estimate the effective rank of each $\Lambdab_j$ for $j=1, \ldots, d$. We estimate the effective rank of $\Lambdab_j$ as the number of columns of $\widehat{\Lambdab}_j$ with at least one nonzero entry. The rank estimates are less accurate for vector-valued data, but improve substantially for matrix-valued and three-dimensional array-valued observations. In these higher-order settings, the estimated ranks are close to the true ranks, especially as the dimension and sample size increase. Overall, these results show that SSFA accurately estimates mode-specific covariance matrices, recovers the column spaces associated with the mode-specific loading matrices, and provides useful effective rank estimates in higher-order array settings.

\begin{table}[h!]
  \caption{
    SSFA mode-one estimation results. Entries report medians, with median absolute deviations shown in parentheses. The metrics $\mathrm{err}_{\Sigmab_1}$, $\mathrm{err}_{\Lambdab_1}$, and $\mathrm{err}_{\Pb_1}$ are defined in \eqref{eq:err}, \eqref{eq:null_err}, and \eqref{eq:null_proj_err}, respectively. \label{table:ssfa_mode_one_sim}}
  \resizebox{\textwidth}{!}{%
  \begin{tabular}{clcrrccc}
    \hline
    scenario & size & $\text{err}_{\Sigmab_1}$ & $\text{err}_{\Lambdab_1}$ & $\text{err}_{\Pb_1}$ & proportion zero $\widehat{\Lambdab}_1$ & rank $\widehat{\Lambdab}_1$ & true rank $\Lambdab_1$ \\
    \hline
    \multirow{5}{8em}{$\xb_i \in \CC^{25}$}
      &   $n=5$ & 0.811 (0.148) & 0.383 (0.050) & 0.664 (0.086) & 0.000 (0.000) & 4.000 (0.000) & \multirow{4}{2em}{3} \\
      &  $n=10$ & 0.590 (0.097) & 0.218 (0.024) & 0.388 (0.038) & 0.040 (0.016) & 5.000 (0.000) & \\
      &  $n=20$ & 0.419 (0.063) & 0.142 (0.017) & 0.258 (0.029) & 0.056 (0.008) & 5.000 (0.000) & \\
      &  $n=40$ & 0.301 (0.052) & 0.098 (0.009) & 0.174 (0.016) & 0.064 (0.016) & 5.000 (0.000) & \\
      & $n=100$ & 0.191 (0.028) & 0.057 (0.004) & 0.103 (0.009) & 0.096 (0.024) & 5.000 (0.000) & \\
    \hline
    \multirow{5}{8em}{$\xb_i \in \CC^{50}$}
      &   $n=5$ & 0.939 (0.177) & 0.486 (0.047) & 0.984 (0.101) & 0.000 (0.000) & 4.000 (0.000) & \multirow{4}{2em}{4} \\
      &  $n=10$ & 0.688 (0.085) & 0.286 (0.036) & 0.580 (0.071) & 0.012 (0.004) & 5.000 (0.000) & \\
      &  $n=20$ & 0.488 (0.058) & 0.180 (0.010) & 0.366 (0.020) & 0.020 (0.004) & 5.000 (0.000) & \\
      &  $n=40$ & 0.351 (0.042) & 0.119 (0.006) & 0.241 (0.013) & 0.024 (0.008) & 5.000 (0.000) & \\
      & $n=100$ & 0.220 (0.024) & 0.072 (0.005) & 0.146 (0.011) & 0.044 (0.016) & 5.000 (0.000) & \\
    \hline
    \multirow{4}{8em}{$\Xb_i \in \CC^{25 \times 25}$}
      &  $n=5$ & 0.212 (0.039) & 0.061 (0.005) & 0.128 (0.011) & 0.068 (0.020) & 5.000 (0.000) & \multirow{4}{2em}{4} \\
      & $n=10$ & 0.154 (0.022) & 0.039 (0.002) & 0.082 (0.006) & 0.120 (0.032) & 5.000 (0.000) & \\
      & $n=20$ & 0.122 (0.027) & 0.027 (0.002) & 0.056 (0.004) & 0.140 (0.030) & 4.000 (0.000) & \\
      & $n=40$ & 0.089 (0.020) & 0.019 (0.001) & 0.040 (0.003) & 0.150 (0.030) & 4.000 (0.000) & \\
    \hline
    \multirow{4}{8em}{$\Xb_i \in \CC^{50 \times 50}$}
      &  $n=5$ & 0.160 (0.024) & 0.044 (0.002) & 0.089 (0.005) & 0.082 (0.046) & 5.000 (0.000) & \multirow{4}{2em}{4} \\
      & $n=10$ & 0.121 (0.021) & 0.029 (0.001) & 0.058 (0.003) & 0.177 (0.033) & 4.000 (0.000) & \\
      & $n=20$ & 0.091 (0.023) & 0.020 (0.001) & 0.040 (0.002) & 0.180 (0.030) & 4.000 (0.000) & \\
      & $n=40$ & 0.070 (0.018) & 0.014 (0.001) & 0.027 (0.001) & 0.235 (0.045) & 4.000 (0.000) & \\
    \hline
    \multirow{4}{8em}{$\Xcal_i \in \CC^{25 \times 25 \times 25}$}
      &  $n=5$ & 0.070 (0.024) & 0.010 (0.001) & 0.021 (0.002) & 0.240 (0.030) & 4.000 (0.000) & \multirow{4}{2em}{4} \\
      & $n=10$ & 0.058 (0.022) & 0.007 (0.001) & 0.015 (0.001) & 0.315 (0.065) & 4.000 (0.000) & \\
      & $n=20$ & 0.056 (0.027) & 0.005 (0.001) & 0.011 (0.001) & 0.410 (0.060) & 4.000 (0.000) & \\
      & $n=40$ & 0.053 (0.027) & 0.004 (0.001) & 0.008 (0.001) & 0.470 (0.070) & 4.000 (0.000) & \\
    \hline
    \multirow{4}{8em}{$\Xcal_i \in \CC^{50 \times 50 \times 50}$}
      &  $n=5$ & 0.054 (0.017) & 0.005 (0.000) & 0.011 (0.001) & 0.427 (0.055) & 4.000 (0.000) & \multirow{4}{2em}{4} \\
      & $n=10$ & 0.054 (0.022) & 0.004 (0.000) & 0.008 (0.001) & 0.518 (0.077) & 4.000 (0.000) & \\
      & $n=20$ & 0.054 (0.021) & 0.003 (0.000) & 0.007 (0.001) & 0.600 (0.052) & 4.000 (0.000) & \\
      & $n=40$ & 0.056 (0.024) & 0.003 (0.000) & 0.006 (0.001) & 0.657 (0.035) & 4.000 (0.000) & \\
    \hline
  \end{tabular}}
\end{table}

\begin{table}[h!]
  \caption{
    SSFA mode-two estimation results. Entries report medians, with median absolute deviations shown in parentheses. The metrics $\mathrm{err}_{\Sigmab_2}$, $\mathrm{err}_{\Lambdab_2}$, and $\mathrm{err}_{\Pb_2}$ are defined in \eqref{eq:err}, \eqref{eq:null_err}, and \eqref{eq:null_proj_err}, respectively.
    \label{table:ssfa_mode_two_sim}}
  \resizebox{\textwidth}{!}{%
  \begin{tabular}{clcrrccc}
    \hline
    scenario & size & $\text{err}_{\Sigmab_2}$ & $\text{err}_{\Lambdab_2}$ & $\text{err}_{\Pb_2}$ & proportion zero $\widehat{\Lambdab}_2$ & rank $\widehat{\Lambdab}_2$ & true rank $\Lambdab_2$ \\
    \hline
    \multirow{4}{8em}{$\Xb_i \in \CC^{25 \times 25}$}
      &  $n=5$ & 0.206 (0.035) & 0.054 (0.005) & 0.097 (0.008) & 0.088 (0.024) & 5.000 (0.000) & \multirow{4}{2em}{3} \\
      & $n=10$ & 0.142 (0.032) & 0.037 (0.004) & 0.065 (0.007) & 0.111 (0.031) & 4.000 (1.000) & \\
      & $n=20$ & 0.104 (0.025) & 0.025 (0.002) & 0.045 (0.004) & 0.120 (0.027) & 3.000 (0.000) & \\
      & $n=40$ & 0.078 (0.016) & 0.017 (0.001) & 0.031 (0.003) & 0.120 (0.027) & 3.000 (0.000) & \\
    \hline
    \multirow{4}{8em}{$\Xb_i \in \CC^{50 \times 50}$}
      &  $n=5$ & 0.147 (0.028) & 0.039 (0.002) & 0.068 (0.003) & 0.096 (0.044) & 5.000 (0.000) & \multirow{4}{2em}{3} \\
      & $n=10$ & 0.112 (0.025) & 0.026 (0.001) & 0.045 (0.003) & 0.153 (0.037) & 3.000 (0.000) & \\
      & $n=20$ & 0.086 (0.022) & 0.018 (0.001) & 0.031 (0.002) & 0.160 (0.027) & 3.000 (0.000) & \\
      & $n=40$ & 0.063 (0.018) & 0.012 (0.001) & 0.022 (0.001) & 0.213 (0.040) & 3.000 (0.000) & \\
    \hline
    \multirow{4}{8em}{$\Xcal_i \in \CC^{25 \times 25 \times 25}$}
      &  $n=5$ & 0.067 (0.025) & 0.010 (0.001) & 0.017 (0.001) & 0.213 (0.040) & 3.000 (0.000) & \multirow{4}{2em}{3} \\
      & $n=10$ & 0.057 (0.025) & 0.007 (0.001) & 0.012 (0.001) & 0.267 (0.073) & 3.000 (0.000) & \\
      & $n=20$ & 0.055 (0.024) & 0.005 (0.001) & 0.009 (0.001) & 0.360 (0.067) & 3.000 (0.000) & \\
      & $n=40$ & 0.051 (0.026) & 0.004 (0.001) & 0.007 (0.001) & 0.440 (0.067) & 3.000 (0.000) & \\
    \hline
    \multirow{4}{8em}{$\Xcal_i \in \CC^{50 \times 50 \times 50}$}
      &  $n=5$ & 0.051 (0.018) & 0.005 (0.000) & 0.009 (0.001) & 0.383 (0.050) & 3.000 (0.000) & \multirow{4}{2em}{3} \\
      & $n=10$ & 0.048 (0.018) & 0.004 (0.000) & 0.007 (0.001) & 0.440 (0.073) & 3.000 (0.000) & \\
      & $n=20$ & 0.047 (0.018) & 0.003 (0.001) & 0.006 (0.001) & 0.547 (0.060) & 3.000 (0.000) & \\
      & $n=40$ & 0.047 (0.019) & 0.003 (0.001) & 0.005 (0.001) & 0.590 (0.030) & 3.000 (0.000) & \\
    \hline
  \end{tabular}}
\end{table}

\begin{table}[h!]
  \caption{
    SSFA mode-three estimation results. Entries report medians, with median absolute deviations shown in parentheses. The metrics $\mathrm{err}_{\Sigmab_3}$, $\mathrm{err}_{\Lambdab_3}$, and $\mathrm{err}_{\Pb_3}$ are defined in \eqref{eq:err}, \eqref{eq:null_err}, and \eqref{eq:null_proj_err}, respectively.
    \label{table:ssfa_mode_three_sim}}
  \resizebox{\textwidth}{!}{%
  \begin{tabular}{clcrrccc}
    \hline
    scenario & size & $\text{err}_{\Sigmab_3}$ & $\text{err}_{\Lambdab_3}$ & $\text{err}_{\Pb_3}$ & proportion zero $\widehat{\Lambdab}_3$ & rank $\widehat{\Lambdab}_3$ & true rank $\Lambdab_3$ \\
    \hline
    \multirow{4}{8em}{$\Xcal_i \in \CC^{25 \times 25 \times 25}$}
      &  $n=5$ & 0.052 (0.021) & 0.009 (0.001) & 0.013 (0.001) & 0.160 (0.040) & 2.000 (0.000) & \multirow{4}{2em}{2} \\
      & $n=10$ & 0.048 (0.020) & 0.007 (0.001) & 0.009 (0.001) & 0.200 (0.060) & 2.000 (0.000) & \\
      & $n=20$ & 0.047 (0.021) & 0.005 (0.000) & 0.007 (0.001) & 0.280 (0.060) & 2.000 (0.000) & \\
      & $n=40$ & 0.046 (0.022) & 0.004 (0.001) & 0.005 (0.001) & 0.340 (0.080) & 2.000 (0.000) & \\
    \hline
    \multirow{4}{8em}{$\Xcal_i \in \CC^{50 \times 50 \times 50}$}
      &  $n=5$ & 0.050 (0.018) & 0.005 (0.000) & 0.007 (0.001) & 0.280 (0.065) & 2.000 (0.000) & \multirow{4}{2em}{2} \\
      & $n=10$ & 0.042 (0.019) & 0.004 (0.001) & 0.005 (0.001) & 0.355 (0.065) & 2.000 (0.000) & \\
      & $n=20$ & 0.044 (0.018) & 0.003 (0.001) & 0.005 (0.001) & 0.405 (0.055) & 2.000 (0.000) & \\
      & $n=40$ & 0.047 (0.017) & 0.003 (0.001) & 0.004 (0.001) & 0.450 (0.030) & 2.000 (0.000) & \\
    \hline
  \end{tabular}}
\end{table}

\section{Real-world data analyses}
\label{sec:real_analysis}

We analyzed LFP recordings from 15 mice with electrodes implanted in 13 brain regions. Nine mice had recordings missing from one or more regions because of electrode misplacement. We first imputed these recordings using an unpenalized vectorized complex factor model. To evaluate imputation, we used the six mice with complete recordings and sequentially masked either one region or one pair of regions, predicting the omitted recordings from the remaining regions.

After imputation, we fit SSFA to all 15 mice under three covariance structures. The first combines brain region and frequency into one mode and treats time as a second mode. The second treats brain region, frequency, and time as separate modes. The third again combines brain region and frequency but treats time windows as independent observations. We describe the experimental setting and preprocessing, evaluate imputation, compare the three SSFA formulations, and benchmark SSFA against existing methods.

\subsection{Experimental setting}
\label{sec:realdata_exp_setting}

Fourteen mice were recorded on three experimental days, and one mouse was recorded only on days two and three. Each day included a 600-second baseline period followed by four post-administration phases: 150 seconds of incubation, 300 seconds in phase 1, 300 seconds in phase 2, and 1800 seconds in phase 3. Mice received the control condition on days one and three and the treatment condition on day two. Baseline activity was recorded before administration, followed by 40 minutes of post-administration recording.

We divided each LFP signal into one-second windows containing 1000 observations sampled at 1000 Hz. For each window and brain region, we applied a Fourier transform and retained the coefficients from 0 to 100 Hz at 1 Hz resolution. Each window therefore contained $p=13\times101=1313$ brain-region--frequency features. After vectorizing across brain regions, each one-second window has feature dimension $p = 13 \times 101 = 1313$. Complete recordings from all 13 regions were available for six mice. For the remaining nine mice, electrode misplacement caused one or more brain regions to be missing across all time windows.

\subsection{Imputation of missing brain-region recordings}
\label{sec:realdata_imputation_analysis}

We first evaluated the ability of the vectorized complex factor model in \eqref{eq:fact-mdl} to impute missing brain-region recordings. Because the objective was prediction rather than interpretation, we imposed neither sparsity nor separability. The model uses low-rank dependence among the Fourier-domain features to predict missing regions from the observed regions.

We created validation cases using the six mice with complete recordings. For each mouse, we first treated one brain region at a time as missing and imputed its Fourier coefficients. This procedure was repeated by treating each pair of brain regions as missing. This produced $6 \times \left\{13 + {13 \choose 2}\right\} = 546$ validation cases, corresponding to 13 single-region omissions and 78 two-region omissions for each mouse.

We compared three imputation methods under a subset-specific complex normal model. We stratified the data by experimental day and phase, giving 15 day-phase subsets. Within subset $s$, we modeled the $p$-dimensional vectorized Fourier coefficients for each time window as $\xb_{s,t,m}\sim\CC\Ncal_p(\mub_s,\Sigmab_s)$ for $s=1,\ldots,15$, where $m$ indexes the mouse, $t=1,\ldots,n_{s,m}$ indexes the time window for mouse $m$ in subset $s$, $\mub_s\in\CC^p$ is the subset-specific mean vector, and $\Sigmab_s\in\CC^{p\times p}$ is the subset-specific covariance matrix. The three imputation methods differ in how $\Sigmab_s$ is specified or estimated. For all methods, $\mub_s$ is estimated by the empirical mean of the observed Fourier-domain features across all available mice and time windows in subset $s$, denoted by $\widehat{\mub}_s$.

The first method uses subset-specific mean imputation, which ignores dependence among the Fourier-domain features and is equivalent to setting $\Sigmab_s=\Ib_p$ for $s=1,\ldots,15$. A missing value of feature $\ell$ is replaced by $\widehat{\mu}_{s,\ell}$. The second method estimates $\Sigmab_s$ using the empirical complex covariance matrix. Each entry of $\widehat{\Sigmab}_s$ is computed from all available mice and time windows in subset $s$.

The third method uses the unpenalized vectorized complex factor model in \eqref{eq:fact-mdl}. For subset $s$, we assume $\Sigmab_s=\Lambdab_s\Lambdab_s^\star+\Psib_s$, where $\Lambdab_s\in\CC^{p\times k}$ and $\Psib_s\in\RR_{+}^{p\times p}$ is diagonal. We fit the model separately within each subset for $k\in\{10,25,50,100,150,250\}$ factors. For each $s$ and $k$, we applied Algorithm~\ref{algo} with $\rho=0$ to the centered observations $\xb_{s,t,m}-\widehat{\mub}_s$ across all mice and time windows in subset $s$, yielding $\widehat{\Lambdab}_s$ and $\widehat{\Psib}_s$.

For each subset and imputation method, we imputed the missing Fourier-domain features using the conditional mean under the fitted complex normal model. Let $\xb_{s,t,m}^{(\mathrm{mis})}$ and $\xb_{s,t,m}^{(\mathrm{obs})}$ denote the missing and observed components of $\xb_{s,t,m}$, respectively. Accordingly, partition $\widehat{\mub}_s$, $\widehat{\Sigmab}_s$, and $\xb_{s,t,m}$ as
\begin{align*}
  \widehat{\mub}_s =
  \begin{bmatrix}
    \widehat{\mub}_s^{(\mathrm{mis})} \\
    \widehat{\mub}_s^{(\mathrm{obs})}
  \end{bmatrix}, \qquad
  \widehat{\Sigmab}_s =
  \begin{bmatrix}
    \widehat{\Sigmab}_s^{(\mathrm{mis},\mathrm{mis})}
    &
    \widehat{\Sigmab}_s^{(\mathrm{mis},\mathrm{obs})}
    \\
    \widehat{\Sigmab}_s^{(\mathrm{obs},\mathrm{mis})}
    &
    \widehat{\Sigmab}_s^{(\mathrm{obs},\mathrm{obs})}
  \end{bmatrix}, \qquad
  \xb_{s,t,m} =
  \begin{bmatrix}
    \xb_{s,t,m}^{(\mathrm{mis})} \\
    \xb_{s,t,m}^{(\mathrm{obs})}
  \end{bmatrix}.
\end{align*}
The imputed values of $\xb_{s,t,m}^{(\mathrm{mis})}$ are
\begin{align}
\label{eq:realdata_conditional_mean_impute}
\widehat{\xb}_{s,t,m}^{(\mathrm{mis})}
=
\widehat{\mub}_s^{(\mathrm{mis})}
+
\widehat{\Sigmab}_s^{(\mathrm{mis},\mathrm{obs})}
\left\{
\widehat{\Sigmab}_s^{(\mathrm{obs},\mathrm{obs})}
\right\}^{-1}
\left(
  \xb_{s,t,m}^{(\mathrm{obs})}
  -
  \widehat{\mub}_s^{(\mathrm{obs})}
\right).
\end{align}

To evaluate imputation accuracy, we applied the inverse Fourier transform to the imputed coefficients and compared the reconstructed LFP recordings with the observed recordings. For mouse $m$ and omission set $r$, we computed the log relative norm error
\begin{align}
  \label{eq:impute_log_err}
  \mathrm{logerr}_{m,r}
  =
  \log\left(
  \frac{
  \left\lVert \Xb_{m,r} - \widehat{\Xb}_{m,r} \right\rVert
  }{
  \left\lVert \Xb_{m,r} \right\rVert
  }
  \right),
  \qquad m=1,\ldots,6,
  \qquad r=1,\ldots,91,
\end{align}
where $\Xb_{m,r}$ and $\widehat{\Xb}_{m,r}$ denote the observed and imputed LFP recordings, respectively, for mouse $m$ and omission set $r$.

The unpenalized vectorized complex factor model with larger ranks outperformed mean and empirical-covariance imputation (Table~\ref{table:SSFA_loocv_impute_by_region}). For each brain region, the summary includes all cases in which that region was omitted, either alone or as part of a pair, yielding 78 validation cases per region. Imputation error for SSFA decreased as $k$ increased. While $k=250$ often produced the smallest median error, its improvement over $k=150$ was marginal and required substantially more computation. Therefore, we used $k=150$ to impute the missing recordings for the subsequent analyses.

\begin{table}[ht]
\caption{
  Median log relative norm error $\mathrm{logerr}_{m,r}$ from \eqref{eq:impute_log_err}, stratified by brain region. Smaller values indicate better imputation. Results are shown for mean imputation, empirical-covariance imputation, and the unpenalized vectorized complex factor model with $k\in\{10,25,50,100,150,250\}$. Median absolute deviations are reported in parentheses.
}
\label{table:SSFA_loocv_impute_by_region}
  \begin{subtable}{\textwidth}
    \resizebox{0.875\textwidth}{!}{%
    \begin{tabular}{rcccccc}
      \hline
      Method & AMY & Acc & BNST & HYPO & ICL & ICR \\
      \hline
      Empirical &  0.207 (0.229) &  0.402 (0.820) & -0.312 (0.414) & -0.331 (0.451) &  0.524 (0.814) &  0.372 (0.278) \\
      \hline
      Mean & -0.001 (0.000) &  0.000 (0.000) &  0.000 (0.000) &  0.000 (0.000) &  0.000 (0.000) &  0.000 (0.000) \\
      \hline
      \multicolumn{7}{l}{SSFA with $\rho=0$} \\
      \hline
      $k=10$ & -0.362 (0.180) & -0.363 (0.200) & -0.426 (0.228) & -0.396 (0.227) & -0.331 (0.301) & -0.323 (0.260) \\
      \hline
      $k=25$ & -0.409 (0.238) & -0.422 (0.205) & -0.493 (0.259) & -0.471 (0.241) & -0.372 (0.318) & -0.351 (0.273) \\
      \hline
      $k=50$ & -0.486 (0.261) & -0.522 (0.213) & -0.613 (0.315) & -0.571 (0.250) & -0.416 (0.322) & -0.396 (0.317) \\
      \hline
      $k=100$ & -0.560 (0.264) & -0.610 (0.264) & -0.700 (0.349) & -0.683 (0.303) & -0.447 (0.343) & -0.416 (0.338) \\
      \hline
      $k=150$ & -0.640 (0.253) & -0.680 (0.288) & -0.800 (0.345) & -0.762 (0.313) & -0.502 (0.359) & -0.455 (0.369) \\
      \hline
      $k=250$ & -0.661 (0.291) & -0.708 (0.285) & -0.828 (0.324) & -0.782 (0.309) & -0.516 (0.370) & -0.467 (0.414) \\
      \hline
    \end{tabular}}
  \end{subtable}

  \begin{subtable}{\textwidth}
    \resizebox{\textwidth}{!}{%
    \begin{tabular}{rccccccc}
      \hline
      Method & IL & NAc & Po & SS & VHipp & VPM & Vc \\
      \hline
      Empirical & -0.347 (0.398) & -0.498 (0.295) & -0.457 (0.370) & -0.243 (0.293) & -0.129 (0.139) & -0.455 (0.383) & -0.221 (0.204) \\
      \hline
      Mean &  0.000 (0.000) &  0.000 (0.000) &  0.000 (0.000) &  0.000 (0.000) &  0.000 (0.000) &  0.000 (0.000) & -0.001 (0.000) \\
      \hline
      \multicolumn{8}{l}{SSFA with $\rho=0$} \\
      \hline
      $k=10$ & -0.405 (0.133) & -0.374 (0.135) & -0.382 (0.210) & -0.213 (0.144) & -0.189 (0.132) & -0.391 (0.218) & -0.225 (0.128) \\
      \hline
      $k=25$ & -0.462 (0.171) & -0.445 (0.160) & -0.442 (0.272) & -0.241 (0.157) & -0.197 (0.134) & -0.456 (0.253) & -0.243 (0.130) \\
      \hline
      $k=50$ & -0.529 (0.257) & -0.545 (0.171) & -0.510 (0.275) & -0.279 (0.176) & -0.204 (0.135) & -0.538 (0.307) & -0.267 (0.153) \\
      \hline
      $k=100$ & -0.573 (0.364) & -0.656 (0.220) & -0.603 (0.299) & -0.320 (0.195) & -0.226 (0.152) & -0.660 (0.320) & -0.294 (0.170) \\
      \hline
      $k=150$ & -0.637 (0.371) & -0.738 (0.214) & -0.747 (0.323) & -0.410 (0.264) & -0.231 (0.118) & -0.763 (0.325) & -0.327 (0.153) \\
      \hline
      $k=250$ & -0.618 (0.389) & -0.773 (0.211) & -0.776 (0.308) & -0.431 (0.282) & -0.228 (0.108) & -0.795 (0.307) & -0.340 (0.163) \\
      \hline
    \end{tabular}}
  \end{subtable}
\end{table}

\subsection{SSFA analysis under different separability assumptions}
\label{sec:real_data_covariance_analysis}

\subsubsection{Data organization, preprocessing, and model-fitting}

We transformed the imputed recordings to the frequency domain by applying the Fourier transform to each one-second window of every mouse-day recording. This yielded Fourier-domain measurements from 13 brain regions at frequencies $0,1,\ldots,100$ Hz for all 15 mice. Nine mice included one or more imputed brain-region recordings, whereas six had complete observed recordings.

We fit SSFA to two representations of the 44 mouse-day recordings, treating the recordings as independent and pooling the two injection conditions. The first combines brain region and frequency into one mode and treats time as the second mode. This yields observations $\Xb_i\in\CC^{1313\times600}$, $i=1,\ldots,44$, and the dataset $\Xcal_{\mathrm{mat}}\in\CC^{1313\times600\times44}$, where $1313=13\times101$ and the 600 columns correspond to one-second time windows. The second treats brain region, frequency, and time as separate modes, yielding observations $\Xcal_i\in\CC^{13\times101\times600}$, $i=1,\ldots,44$, and the dataset $\Xcal_{\mathrm{arr}}\in\CC^{13\times101\times600\times44}$. Both representations retain the 600 windows as a time mode, allowing SSFA to estimate dependence across time windows.

We centered the two datasets before model fitting. For each representation, we computed the empirical mean across the 44 mouse-day observations and subtracted it from each observation. We divided each centered observation by the number of features in the corresponding representation.

We fit SSFA using Algorithm~\ref{algo}. For $\Xcal_{\mathrm{mat}}$, the maximum numbers of factors were 150 for the brain-region--frequency mode and 5 for the time mode. For $\Xcal_{\mathrm{arr}}$, the maximum number of factors were 5, 20, and 5 for the brain-region, frequency, and time modes, respectively. For each dataset, we fit SSFA over 21 values of $\rho$, including $\rho=0$ and 20 values selected automatically using the heuristic in Section~\ref{sec:rho_grid}. The selected model had the smallest EBIC.

\subsubsection{Estimated covariance and factor structures}

The SSFA fit to $\Xcal_{\mathrm{mat}}$ estimates the brain-region--frequency- and time-mode covariance matrices $\widehat{\Sigmab}_1$ and $\widehat{\Sigmab}_2$ of dimensions $1313\times1313$ and $600\times600$, respectively. Let $(\widehat{\Lambdab}_1,\widehat{\Psib}_1)$ and $(\widehat{\Lambdab}_2,\widehat{\Psib}_2)$ denote the estimated loading and residual variance matrices for the two modes. Then,  $\widehat{\Sigmab}_1=\widehat{\Lambdab}_1\widehat{\Lambdab}_1^\star+\widehat{\Psib}_1$ and $\widehat{\Sigmab}_2=\widehat{\Lambdab}_2\widehat{\Lambdab}_2^\star+\widehat{\Psib}_2$. If $\widehat{\tau}_{tt}^2$ is the $t$th diagonal entry of $\widehat{\Sigmab}_2$, then the covariance of the brain-region--frequency features at time window $t$ is $\widehat{\tau}_{tt}^2\widehat{\Sigmab}_1$. Because multiplication by $\widehat{\tau}_{tt}^2$ does not affect the corresponding complex correlation matrix, the coherence and phase-offset matrices can be computed directly from $\widehat{\Sigmab}_1$.

We order the 1313 features by frequency, with the 13 brain regions nested within each frequency. For $i\in\{0,\ldots,100\}$, let $(\widehat{\Sigmab}_1)_{i+1,i+1} \in\CC^{13\times13}$ denote the diagonal block corresponding to frequency $i$. Define
\begin{align*}
  \widehat{\Vb}_i
  =
  (\widehat{\Sigmab}_1)_{i+1,i+1}
  \circ\Ib_{13},\qquad
  \widehat{\Rb}_i
  =
  \widehat{\Vb}_i^{-1/2}
  (\widehat{\Sigmab}_1)_{i+1,i+1}
  \widehat{\Vb}_i^{-1/2}.
\end{align*}
The coherence and phase-offset matrices at frequency $i$ are $|\widehat{\Rb}_i|$ and $\arg\{\widehat{\Rb}_i\}$, respectively, where the modulus and argument are applied elementwise. Coherence values range from zero to one, with values near zero and one indicating weak and strong linear dependence, respectively, between pairs of brain regions. Phase offsets should be interpreted jointly with coherence because they are less informative when the corresponding coherence is small.

The SSFA fit to $\Xcal_{\mathrm{arr}}$ estimates the brain-region-, frequency-, and time-mode covariance matrices $\widehat{\Sigmab}_1$, $\widehat{\Sigmab}_2$, and $\widehat{\Sigmab}_3$ of dimensions $13\times13$, $101\times101$, and $600\times600$, respectively. Let $(\widehat{\Lambdab}_j,\widehat{\Psib}_j)$ denote the estimated loading and residual variance matrices for mode $j$. Then, $\widehat{\Sigmab}_j=\widehat{\Lambdab}_j\widehat{\Lambdab}_j^\star+\widehat{\Psib}_j$
for $j=1,2,3$. Let $\widehat{\sigma}_{\ell\ell}^2$ and $\widehat{\tau}_{tt}^2$ denote the $\ell$th and $t$th diagonal entries of $\widehat{\Sigmab}_2$ and $\widehat{\Sigmab}_3$, respectively. The covariance among brain regions at frequency index $\ell$ and time window $t$ is $\widehat{\sigma}_{\ell\ell}^2\widehat{\tau}_{tt}^2\widehat{\Sigmab}_1$. Therefore, frequency and time affect the scale of the brain-region covariance but not its correlation structure, implying that the brain-region coherence and phase-offset matrices are constant across frequencies and time windows. In contrast, $\Xcal_{\mathrm{mat}}$ allows brain-region dependence to vary with frequency.

The two SSFA fits produce different coherence and phase-offset estimates. The fit to $\Xcal_{\mathrm{mat}}$ captures frequency-specific brain-region dependence, with distinct coherence blocks at lower frequencies and stronger coherence and phase structure at 20 Hz  (Figure~\ref{fig:example_coherence_phase}). In contrast, the fully separable fit to $\Xcal_{\mathrm{arr}}$ produces a single coherence and phase-offset structure across frequencies and time windows (Figure~\ref{fig:sep_coh_w_phase}). Its coherence estimate is closer to the higher-frequency estimates from $\Xcal_{\mathrm{mat}}$ but does not capture the lower-frequency patterns.

\begin{figure}[htbp]
  \centering
  \includegraphics[
    width=\textwidth,
    height=0.9\textheight,
    keepaspectratio
  ]{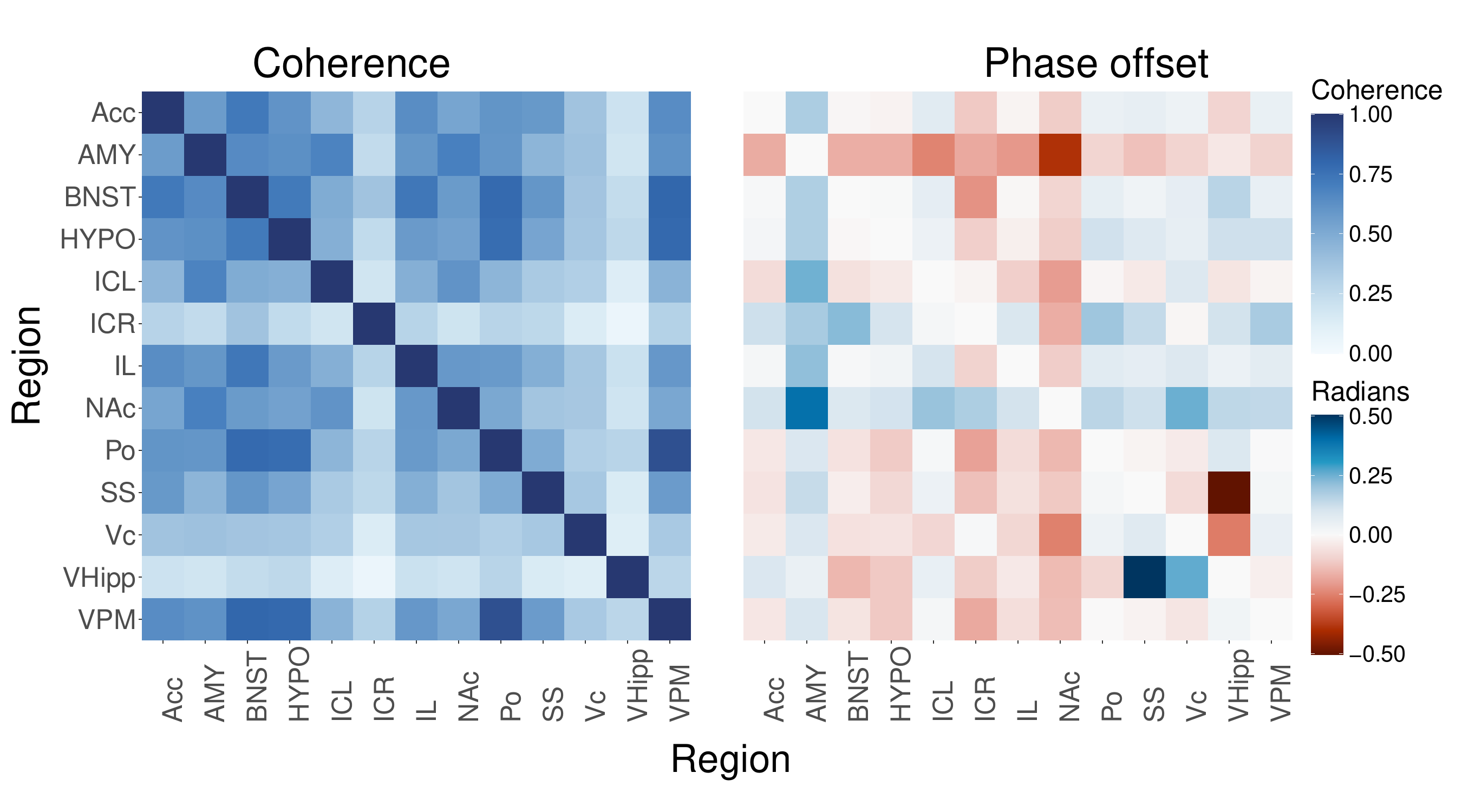}
  \caption{
  Brain-region mode-specific coherence and phase-offset matrices from the fully separable SSFA model fitted to $\Xcal_{\mathrm{arr}}$. The model was fitted to the post-administration data under both the treatment and control conditions.
  }
  \label{fig:sep_coh_w_phase}
\end{figure}

We next examine the estimated time-mode dependence from the SSFA fits to $\Xcal_{\mathrm{mat}}$ and $\Xcal_{\mathrm{arr}}$.  Figure~\ref{fig:time_factors_comp} shows the moduli of the two largest-magnitude time-mode loading vectors from the two fits. The first and final 300 seconds correspond to phases 1 and 2,
respectively. For both representations, the loading vectors indicate weaker dependence among the one-second windows during phase 1 and stronger dependence during the phase 2. This structure emerged without using phase labels during model fitting.

These results highlight the flexibility of SSFA: it directly estimates the time-mode covariance while modeling the brain-region and frequency covariances either jointly or separately (Figures~\ref{fig:example_coherence_phase} and \ref{fig:sep_coh_w_phase}). The weaker temporal dependence during phase 1 in Figure~\ref{fig:time_factors_comp} motivates the vectorized comparison in Section \ref{sec:realdata_vectorized_comparison}.

\begin{figure}[htbp]
  \centering
  \includegraphics[
    width=\textwidth,
    height=0.9\textheight,
    keepaspectratio
  ]{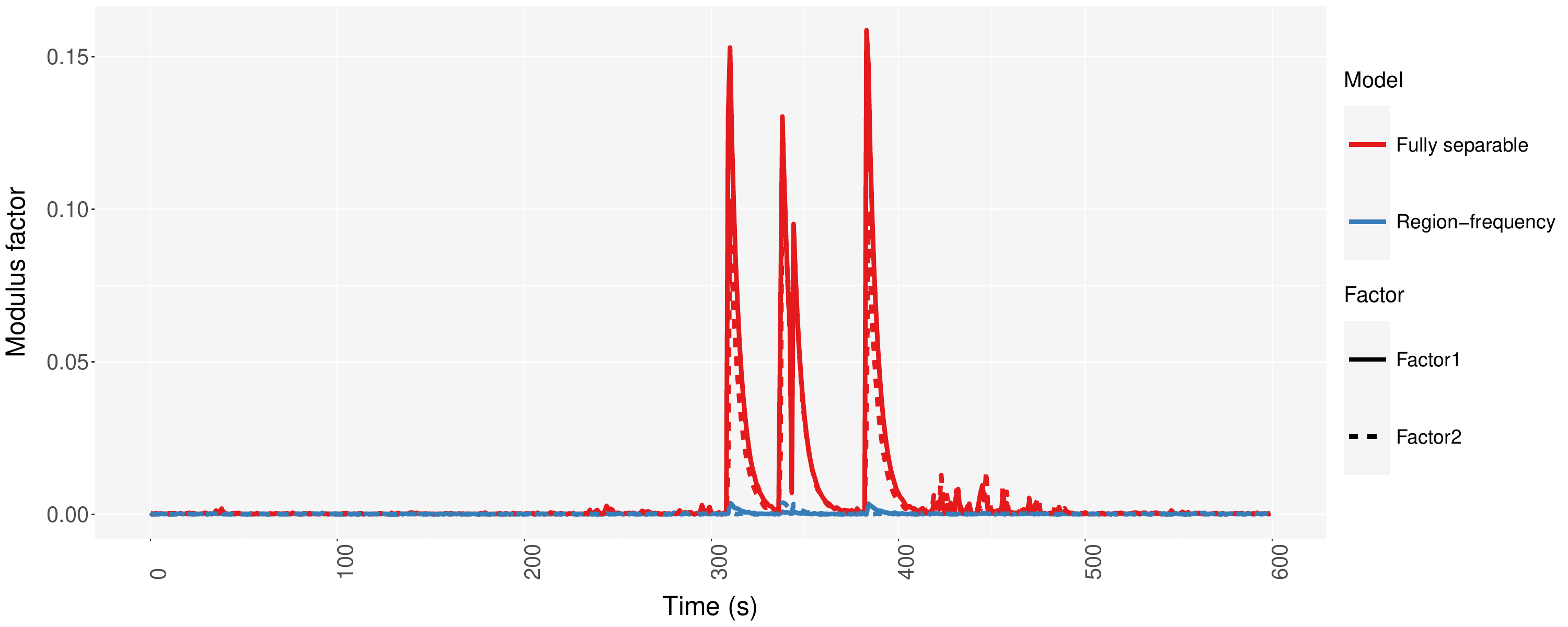}
  \caption{
    Moduli of the two largest-magnitude time-mode loading vectors from the SSFA fits to $\Xcal_{\mathrm{mat}}$ and $\Xcal_{\mathrm{arr}}$. The models were fitted to the post-administration data under both the treatment and control conditions. The first and final 300 seconds correspond to phases 1 and 2, respectively.
  }
  \label{fig:time_factors_comp}
\end{figure}

\subsubsection{Interpretation of coherence and phase estimates}
\label{sec:interpret}

As frequency increases, coherence generally decreases, while phase relationships among appreciably coherent region pairs become more localized (Figure~\ref{fig:example_coherence_phase}). At 1 Hz, coherence is broadly elevated and phase offsets are near zero, indicating widespread, approximately synchronous dependence. At 2 Hz, coherence remains strong across many regions, but Vc and VHipp are less integrated; Vc also exhibits the clearest phase offsets relative to several coherent regions. At 5 Hz, dependence is more concentrated, particularly among HYPO, Po, and VPM and among Acc, AMY, BNST, and NAc, with more differentiated phase relationships. At 20 Hz, coherence is lower but remains appreciable among selected regions, including AMY, BNST, HYPO, IL, NAc, Po, and VPM, while the largest phase offsets involve ICL, NAc, Po, and VPM. Lead-lag interpretations are restricted to appreciably coherent pairs and depend on the adopted sign convention \citep{becerra2011robust,tillman2018intrinsic,androulakis2018modulation}.

This pattern is broadly consistent with longstanding findings from scalp EEG and intracranial LFP recordings. Across different brain states, higher-frequency activity tends to reflect more localized interactions, whereas lower-frequency activity more often supports long-range coordination across brain regions. The estimated dependence structure therefore suggests a transition from broad low-frequency synchrony to weaker but more selectively organized higher-frequency dependence \citep{von2000different,buzsaki2012origin,arnulfo2020long}.

\subsection{Comparison with vectorized methods}
\label{sec:realdata_vectorized_comparison}

\subsubsection{Dataset construction and model-fitting}

The weaker temporal dependence during the 300-second phase 1 motivated a vectorized comparison of SSFA with complex PCA and sparse PCA. We restricted the comparison to this period on the two vehicle-injection days so that injection treatment, lighting condition, and experimental phase were held fixed. We combined brain region and frequency into a single feature dimension and treated the one-second windows as approximately independent observations.

Denote the resulting dataset by $\Xb_{\mathrm{vec}}\in\CC^{1313\times8700}$, whose columns contain the brain-region--frequency features for the one-second windows. It comprises 29 mouse-day recordings, with 14 from the first experimental day and 15 from the third, giving $8700=29\times300$ observations. We centered the dataset using the empirical mean across the 8700 windows and scaled each centered observation by the 1313 brain-region--frequency features.

We fit vectorized SSFA, complex PCA, and sparse PCA to $\Xb_{\mathrm{vec}}$. For SSFA, we used a maximum of 150 factors and the same fitting and selection procedure as for $\Xcal_{\mathrm{mat}}$ and $\Xcal_{\mathrm{arr}}$. For complex PCA and sparse PCA, we fixed the number of components at 150. We fit sparse PCA using \texttt{scikit-learn} over $\alpha\in\{10^{-4},10^{-3},10^{-2},0.1,0.5,1\}$ and selected the model with the smallest EBIC. For each method, we estimated the brain-region--frequency covariance matrix as
$\widehat{\Sigmab} = \widehat{\Lambdab}\widehat{\Lambdab}^{\star}+\widehat{\Psib}$, where $\widehat{\Lambdab}$ is the estimated loading matrix and $\widehat \Psib$ is the estimated residual covariance matrix. For complex PCA and sparse PCA, $\widehat{\Lambdab}$ and $\widehat{\sigma}^2$ were estimated as described in Section~\ref{sec:sim_competing_methods}, and we set $\widehat{\Psib}=\widehat{\sigma}^2\Ib_{1313}$.

\subsubsection{Within-frequency coherence and phase}

We order the 1313 features by frequency, with the 13 brain regions nested within each frequency. For $i,j\in\{0,\ldots,100\}$, let $\widehat{\Sigmab}_{i+1,j+1} \in \CC^{13\times13}$ denote the block corresponding to frequencies $i$ and $j$. The indices $i+1$ and $j+1$ account for the fact that the frequencies begin at $0$ Hz. The diagonal block $\widehat{\Sigmab}_{i+1,i+1}$ estimates the covariance among brain regions at frequency $i$.

Define the normalized within-frequency covariance matrix by
\begin{align*}
  \widehat{\Rb}_i
  &=
  \left(
    \widehat{\Sigmab}_{i+1,i+1}
    \circ
    \Ib_{13}
  \right)^{-1/2}
  \widehat{\Sigmab}_{i+1,i+1}
  \left(
    \widehat{\Sigmab}_{i+1,i+1}
    \circ
    \Ib_{13}
  \right)^{-1/2}.
\end{align*}
The coherence and phase-offset matrices at frequency $i$ are $|\widehat{\Rb}_i|$ and $\arg\{\widehat{\Rb}_i\}$, respectively. Figures~\ref{fig:coh_comp_dark_veh} and~\ref{fig:phase_comp_dark_veh} compare the estimated coherence and phase-offset matrices obtained using vectorized SSFA, complex PCA, and sparse PCA. Sparse PCA produces substantially sparser coherence matrices, suggesting that the real-domain embedding and sparsity penalty may suppress many dependencies among brain regions. In contrast, complex PCA and vectorized SSFA recover similar dependence patterns at lower frequencies. At 20 Hz, vectorized SSFA retains stronger coherent structure, whereas the corresponding estimates from complex PCA and sparse PCA are weaker.

The phase-offset estimates show a similar pattern. Sparse PCA yields little interpretable phase structure because many of its corresponding coherence estimates are small. Complex PCA and vectorized SSFA recover broadly similar phase relationships at lower frequencies, but vectorized SSFA retains more phase structure at 20 Hz, where its coherence estimates are larger. These results suggest that vectorized SSFA preserves more of the complex-valued dependence among brain regions at frequencies for which the competing methods produce weaker coherence estimates.

\begin{figure}[htbp]
  \centering
  \includegraphics[
    width=0.7\textwidth,
    height=0.7\textheight,
    keepaspectratio
  ]{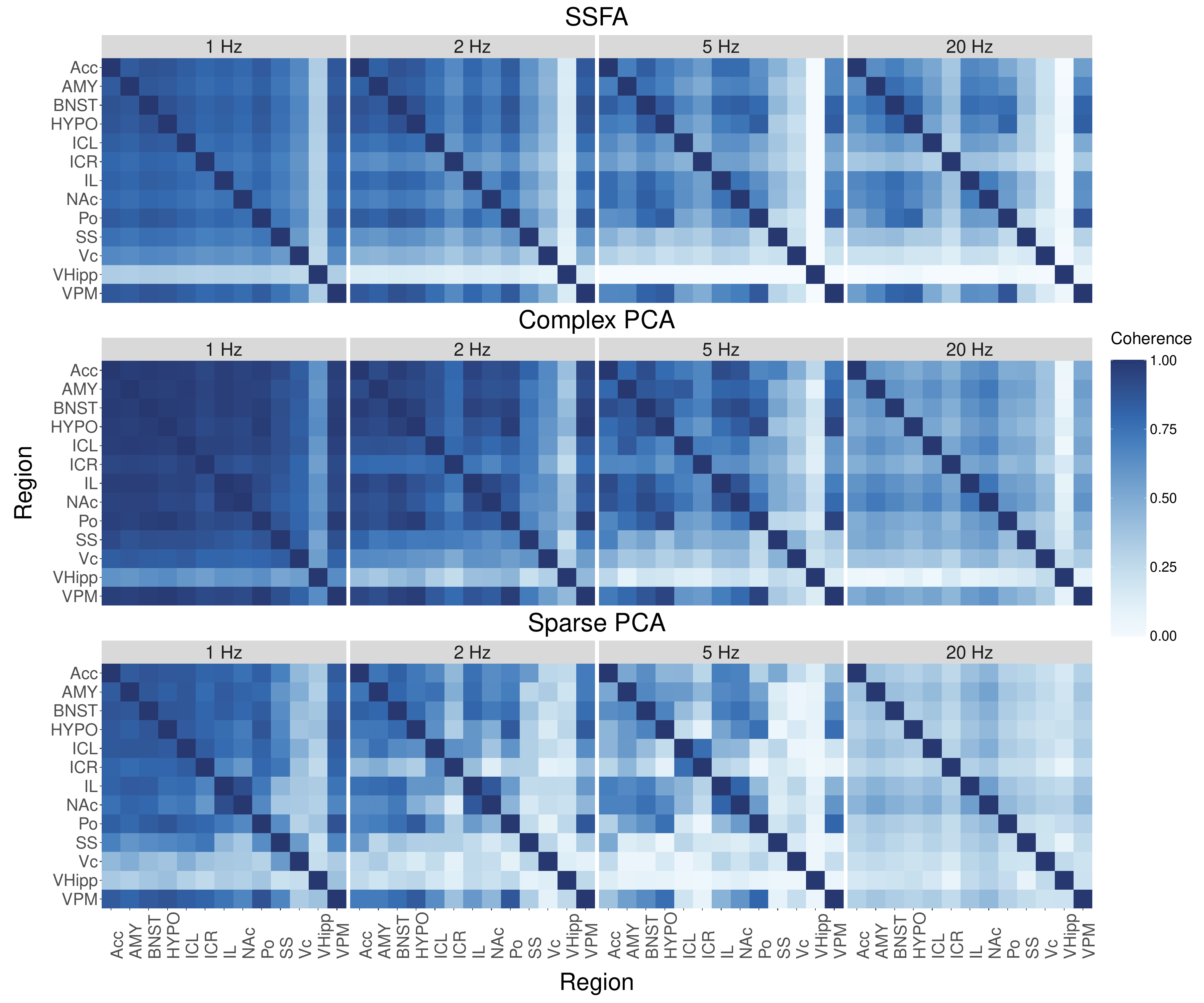}
  \caption{
  Estimated coherence matrices at selected frequencies from SSFA, complex PCA, and sparse PCA fitted to $\Xb_{\mathrm{vec}}$. The comparison uses phase 1 of the post-administration period from the two control days.
  }
  \label{fig:coh_comp_dark_veh}
\end{figure}

\begin{figure}[htbp]
  \centering
  \includegraphics[
    width=0.7\textwidth,
    height=0.7\textheight,
    keepaspectratio
  ]{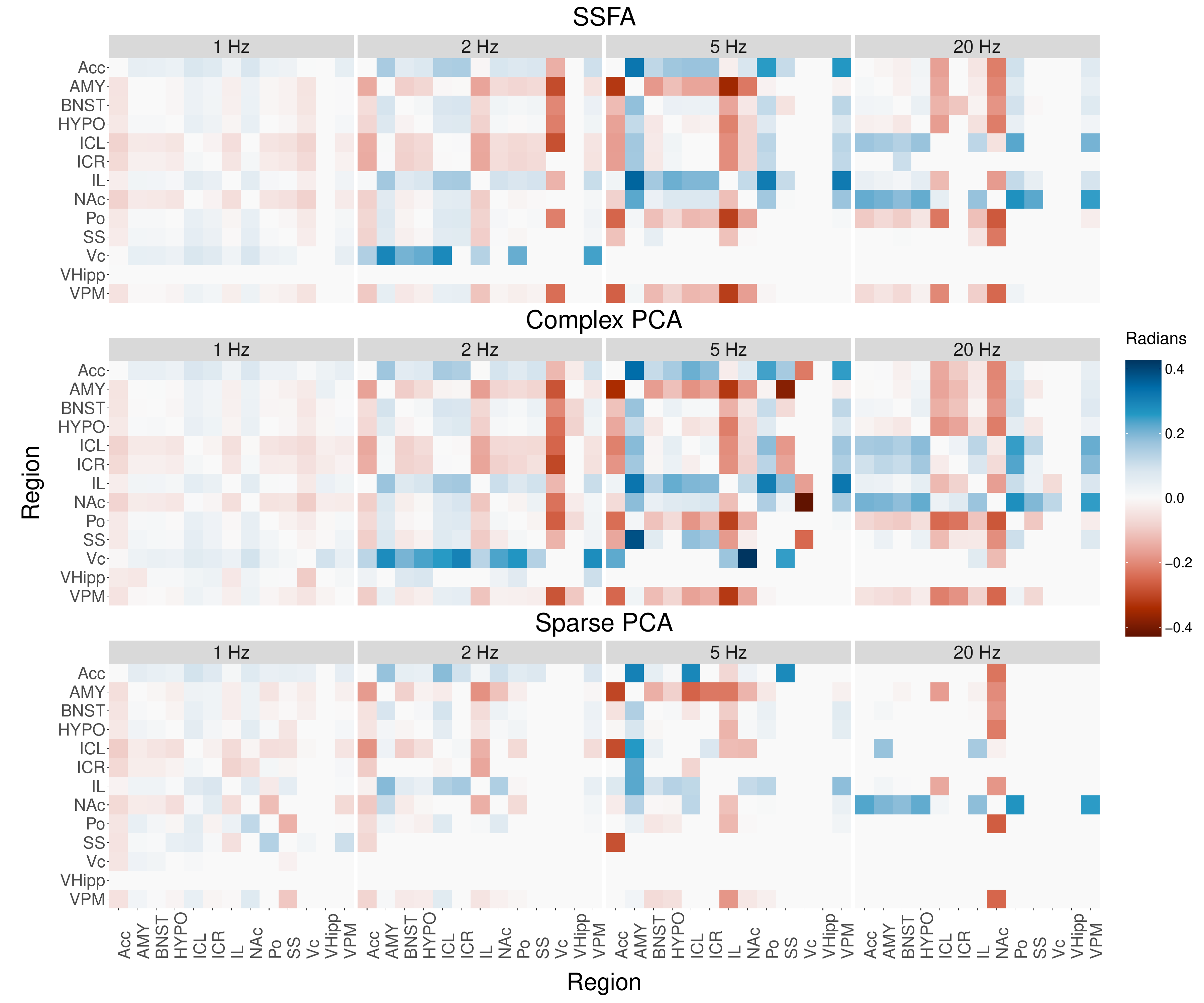}
  \caption{
  Estimated phase-offset matrices at selected frequencies from SSFA, complex PCA, and sparse PCA fitted to $\Xb_{\mathrm{vec}}$. The comparison uses phase 1 of the post-administration period from the two control days.
  }
  \label{fig:phase_comp_dark_veh}
\end{figure}

\subsubsection{Cross-frequency covariance and factor structures}

We next evaluate the estimated cross-frequency correlations. Under a local second-order stationarity approximation for the restricted dark-light interval, discrete Fourier transform coefficients at different Fourier frequencies are asymptotically uncorrelated \citep{jentsch2015test}. Therefore, in our setting, this result implies that the entries of $\widehat{\Sigmab}_{i+1,j+1}$ should be small when $i\neq j$. For frequencies $i,j\in\{0,\ldots,100\}$, define the normalized cross-frequency correlation block by
\begin{align*}
  \widehat{\Rb}_{i+1,j+1}
  =
  \left(\widehat{\Sigmab}_{i+1,i+1}\right)^{-1/2}
  \widehat{\Sigmab}_{i+1,j+1}
  \left(\widehat{\Sigmab}_{j+1,j+1}\right)^{-1/2}
  \in\CC^{13\times13},
\end{align*}
where $\left(\widehat{\Sigmab}_{i+1,i+1}\right)^{-1/2} \left(\widehat{\Sigmab}_{i+1,i+1}\right)^{-1/2} = \widehat{\Sigmab}_{i+1,i+1}^{-1}$.
For $i\neq j$, spectral theory suggests that $\widehat{\Rb}_{i+1,j+1}$
should be close to zero, up to finite-sample variability and estimation error. Figure~\ref{fig:cross_coh_dark_veh} shows the elementwise moduli
$|\widehat{\Rb}_{2,3}|$, corresponding to cross-frequency correlation between brain regions at 1 Hz and 2 Hz.
The estimates obtained using vectorized SSFA are generally closer to zero than those obtained using complex PCA and sparse PCA. This indicates that vectorized SSFA performs better in recovering the expected near-zero correlation between different Fourier frequencies.

The loading matrix provides a factor-level view of the cross-frequency covariance structure. Frequency-localized loading vectors induce covariance mainly within frequencies, whereas loading vectors with nonzero entries across several frequencies can induce cross-frequency covariance. Because individual loading vectors are not identifiable, we do not match columns across methods. For visualization, we order the columns of the loading matrix within each method by their Euclidean norms and display the six columns with the largest norms. Figure~\ref{fig:reg_freq_factors_dark_veh} shows the moduli of these loading vectors for each method. We display only the entries corresponding to the first seven frequencies for each brain region.

The loading vectors obtained using vectorized SSFA show greater agreement with the expected frequency-localized structure than those obtained using complex PCA and sparse PCA. Complex PCA produces loading patterns that are more diffuse across frequencies. Sparse PCA produces highly sparse loading vectors whose nonzero entries are less clearly concentrated within frequency-specific bands. These loading patterns help explain the smaller cross-frequency correlation estimates obtained using vectorized SSFA.

In summary, these comparisons demonstrate the importance of accounting for the complex-valued structure of the data. Sparse PCA imposes sparsity after real-domain embedding rather than directly on the complex loading entries. Complex PCA preserves the complex-valued representation, but its dense loading structure does not exploit the frequency-localized block structure suggested by spectral theory. In contrast, vectorized SSFA estimates sparse complex loading matrices directly in the complex domain and produces covariance estimates that are more consistent with the expected cross-frequency structure.

\begin{figure}[htbp]
  \centering
  \includegraphics[
    width=\textwidth,
    height=0.9\textheight,
    keepaspectratio
  ]{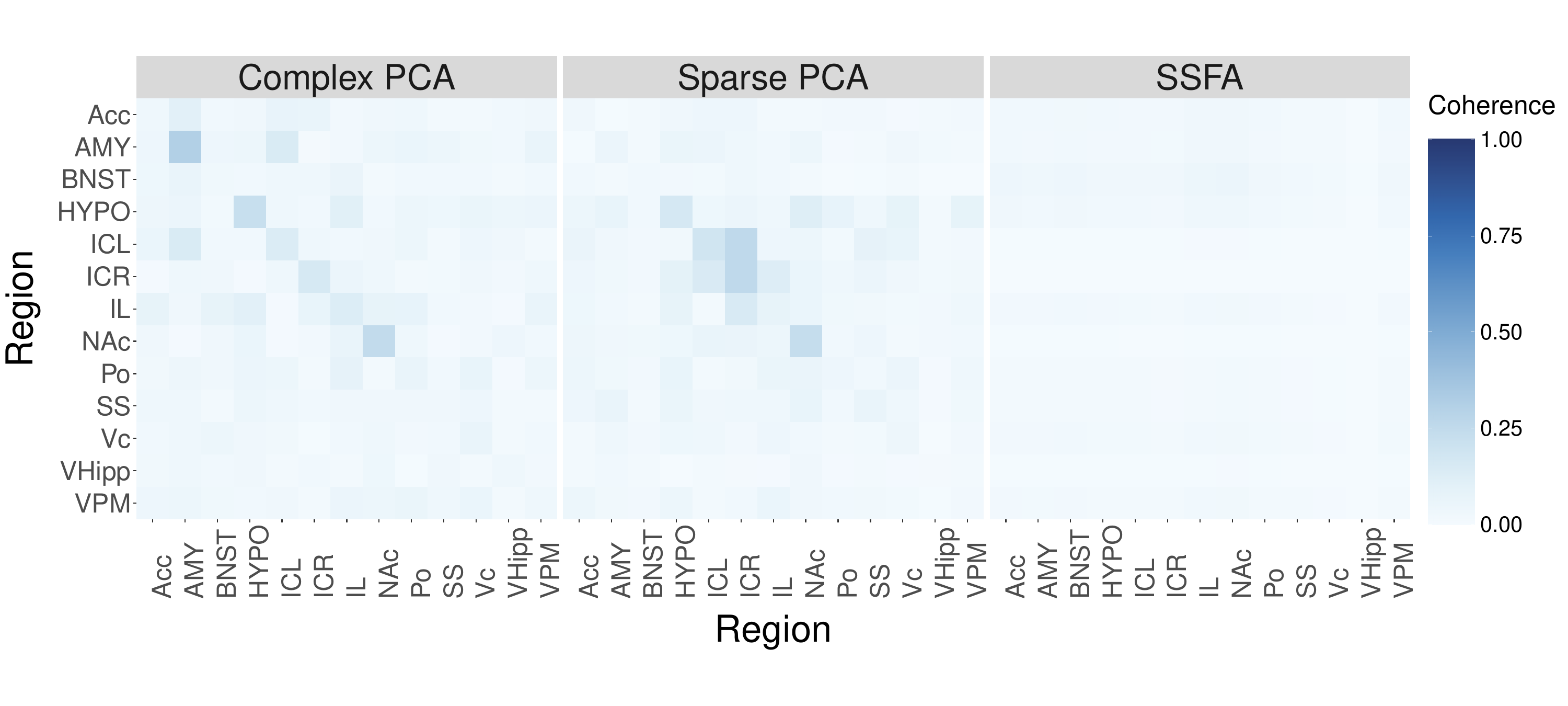}
  \caption{
  Elementwise moduli of the normalized cross-frequency correlation blocks between the brain-region features at 1 Hz and 2 Hz. Estimates are shown for SSFA, complex PCA, and sparse PCA fitted to $\Xb_{\mathrm{vec}}$.
  }
  \label{fig:cross_coh_dark_veh}
\end{figure}

\begin{figure}[htbp]
  \centering
  \includegraphics[
    width=\textwidth,
    height=0.9\textheight,
    keepaspectratio
  ]{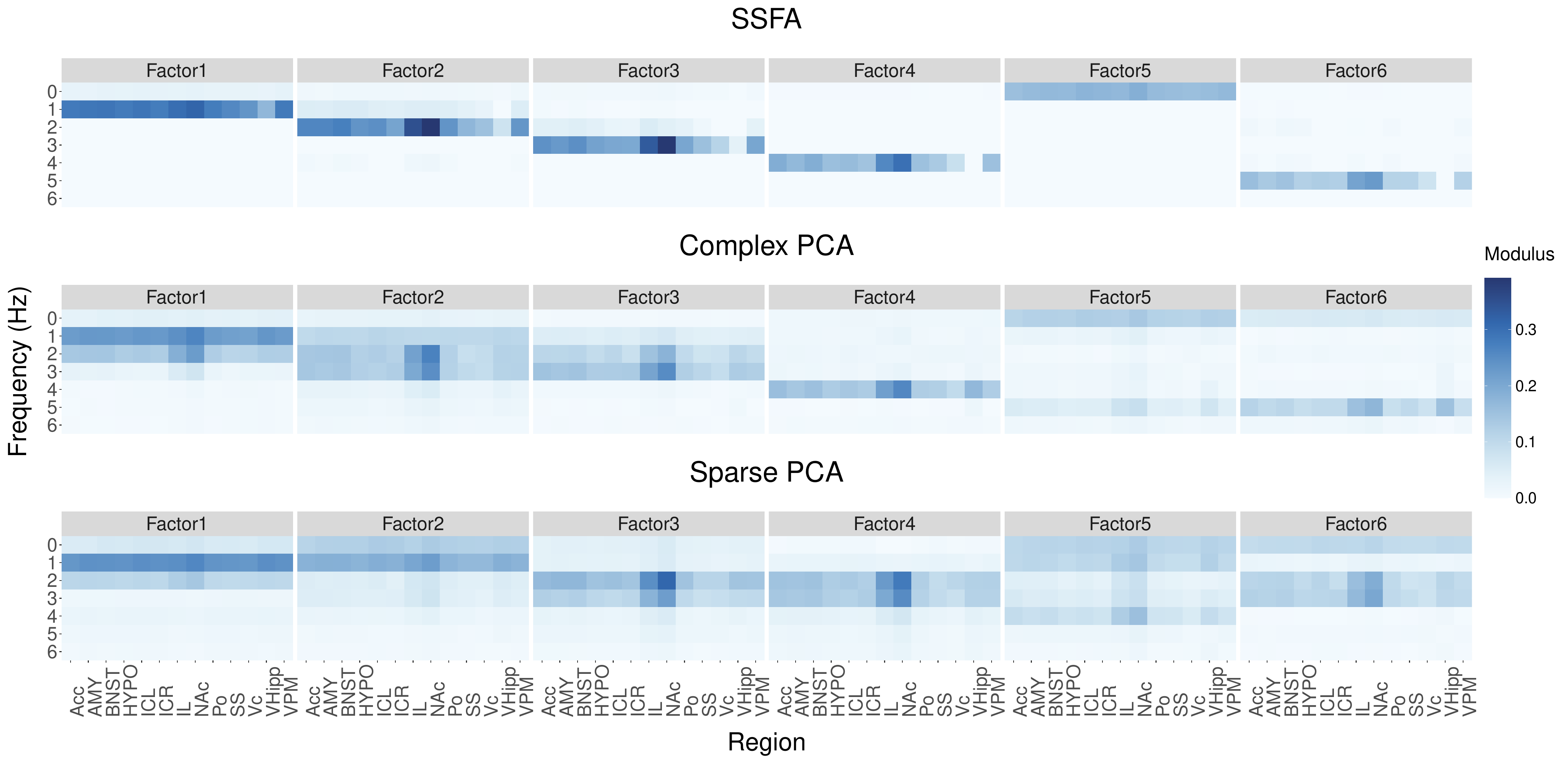}
  \caption{
  Moduli of the six largest-magnitude brain-region--frequency loading vectors from SSFA, complex PCA, and sparse PCA fitted to $\Xb_{\mathrm{vec}}$. Loading entries are shown for the first seven frequencies within each brain region.
  }
  \label{fig:reg_freq_factors_dark_veh}
\end{figure}

\section{Discussion}
\label{sec:discussion}

Several extensions could improve the scalability and flexibility of SSFA. A Bayesian extension of SSFA could replace the lasso penalty with shrinkage priors \citep{Srietal17}.  In Algorithm~\ref{algo}, the repeated mode-wise whitening step is the main computational bottleneck because the data must be rewhitened using updated covariance estimates for the remaining modes. One possible approximation is to prewhiten each mode once using diagonal covariance estimates obtained under the zero-loading model discussed in Section~\ref{sec:rho_grid}. This approach would avoid repeated whitening, permit the mode-specific parameters to be updated in parallel, and reduce computational cost. The imputation method developed in Section~\ref{sec:realdata_imputation_analysis} could also be incorporated into model estimation by alternating between parameter updates and conditional imputation. This extension would allow partially observed arrays to contribute directly to model fitting.

A natural alternative to SSFA replaces the separable covariance factorization with a Tucker-structured factor model \citep{2009_Kolda}. Under this formulation, each observation is represented by a single low-dimensional latent core tensor and mode-specific loading matrices. The resulting observation-specific factor scores can be used for reconstruction, visualization, and downstream regression. This model is better suited for dimension-reduction applications than for mode-specific covariance estimation.

\section*{Acknowledgment}
\label{sec:discussion}

This work was supported by Roy J. Carver Charitable Trust, National Institutes of Health grants 1DP2MH126377-01 and R01-HD121993 awarded to Rainbo Hultman, and National Science Foundation grants DMS-1854667 and DMS-2506058 awarded to Sanvesh Srivastava. Python implementations of SSFA and code for the simulation studies and real-data analysis are publicly available at \url{https://github.com/IHultman/SSFA}.

During the preparation of this manuscript, the authors used ChatGPT 5.5 to assist with grammar and the verification of mathematical proofs. The authors subsequently reviewed and revised all affected content and take full responsibility for the content of the manuscript.

\bibliographystyle{Chicago}
\bibliography{papers}

\newpage

\appendix

\begin{center}
    {\LARGE\bf Supplementary Material: Sparse Separable Factor Analysis in the Complex Domain with an Application to Local Field Potential Data}
\end{center}

\section{Complex normal distribution}

\subsection{Vector and matrix complex normal distribution}
\label{sec:c-vec-mat-norm}

For a complex vector $\yb\in\CC^c$, we write $\yb\sim\mathbb{C}\Ncal_c(\mb,\Vb)$ for the proper complex normal distribution with mean $\mb$, covariance matrix $\Vb$, and zero pseudo-covariance matrix \citep{urban2023oscillating}. Its density is
\begin{align}
  \label{eq:c-vec-pdf}
  f_{\yb}(\yb)
  =
  \pi^{-c}|\Vb|^{-1}
  \exp\left\{
    -(\yb-\mb)^\star\Vb^{-1}(\yb-\mb)
  \right\},
\end{align}
where $^\star$ denotes the complex conjugate-transpose operator \citep{And95}.

For $\Yb\in\CC^{c_1\times c_2}$, we write $\Yb \sim \mathbb{C}\Ncal_{c_1\times c_2} \left(\Mb,\Vb_1,\Vb_2\right)$, where $\Mb\in\CC^{c_1\times c_2}$ is the mean matrix and $\Vb_1\in\CC^{c_1\times c_1}$ and $\Vb_2\in\CC^{c_2\times c_2}$ are the mode-1 and mode-2 covariance matrices, respectively. This distribution is defined through $\yb=\operatorname{vec}(\Yb) \sim \mathbb{C}\Ncal_{c_1c_2}\left(\mb,\Vb_2\otimes\Vb_1 \right)$, where $\mb=\operatorname{vec}(\Mb)$ and  $\otimes$ denotes the Kronecker product \citep{And95}.
Applying \eqref{eq:c-vec-pdf} to the vectorized matrix gives
\begin{align}
  \label{eq:c-mat-pdf-orig}
  f_{\Yb}(\Yb)
  &=
  \pi^{-c_1c_2}
  |\Vb_2|^{-c_1}
  |\Vb_1|^{-c_2}
  \exp\left\{
    -\operatorname{tr}\left[
      (\Yb-\Mb)^\star
      \Vb_1^{-1}
      (\Yb-\Mb)
      \Vb_2^{-\top}
    \right]
  \right\}.
\end{align}
The identity $\operatorname{vec}(\Ab\Bb\Cb) = (\Cb^\top\otimes\Ab)\operatorname{vec}(\Bb)$ implies
\begin{align*}
  (\yb-\mb)^\star
  \left(\Vb_2^{-1}\otimes\Vb_1^{-1}\right)
  (\yb-\mb)
  =
  \operatorname{tr}\left[
    (\Yb-\Mb)^\star
    \Vb_1^{-1}
    (\Yb-\Mb)
    \Vb_2^{-\top}
  \right].
\end{align*}

\subsection{Array-valued complex normal distribution}
\label{sec:c-array-norm}

Following the real-valued array normal distribution \citep{2011_Hoff}, we define its proper complex analogue. Let $\Ycal\in\CC^{c_1\times\cdots\times c_d}$, and define $c=\prod_{\ell=1}^d c_\ell$. We say that $\Ycal$ follows a complex array normal distribution with mean $\Mcal\in\CC^{c_1\times\cdots\times c_d}$ and Hermitian positive-definite mode-specific covariance matrices $\Vb_j\in\CC^{c_j\times c_j}$, for $j=1,\ldots,d$, if
\begin{align}
  \label{eq:c-array-normal-definition}
  \yb
  =
  \operatorname{vec}(\Ycal)
  \sim
  \mathbb{C}\Ncal_c
  \left(
    \mb,
    \Vb_d\otimes\cdots\otimes\Vb_1
  \right),
  \qquad
  \mb=\operatorname{vec}(\Mcal).
\end{align}
We denote this distribution by $\Ycal \sim \mathbb{C}\Ncal_{c_1\times\cdots\times c_d} \left(\Mcal,\Vb_1,\ldots,\Vb_d \right)$.

The matricization of this distribution follows from the convention in \citet{2009_Kolda}. For mode $j$, define $c_{-j} = \prod_{\ell\neq j}c_\ell$ and $\Vb_{-j} = \Vb_d\otimes\cdots\otimes\Vb_{j+1} \otimes \Vb_{j-1}\otimes\cdots\otimes\Vb_1$. The mode-$j$ matricization of $\Ycal$ in \eqref{eq:c-array-normal-definition} satisfies
\begin{align}
  \label{eq:c-array-normal-matricization}
  \Yb_{(j)}
  \sim
  \mathbb{C}\Ncal_{c_j\times c_{-j}}
  \left(
    \Mb_{(j)},\Vb_j,\Vb_{-j}
  \right), \qquad j = 1, \ldots, d,
\end{align}
where $\Yb_{(j)}$ and $\Mb_{(j)}$ are the mode-$j$ matricizations of $\Ycal$ and $\Mcal$, respectively. Applying the complex matrix-normal density in
\eqref{eq:c-mat-pdf-orig}, the density can be written for any mode $j$ as
\begin{align}
  \label{eq:c-arr-pdf}
  f_{\Ycal}(\Ycal)
  &=
  \pi^{-c}
  |\Vb_j|^{-c_{-j}}
  |\Vb_{-j}|^{-c_j}
  \exp\left\{
    -\operatorname{tr}\left[
      (\Yb_{(j)}-\Mb_{(j)})^\star
      \Vb_j^{-1}
      (\Yb_{(j)}-\Mb_{(j)})
      \Vb_{-j}^{-\top}
    \right]
  \right\}
  \nonumber\\
  &=
  \pi^{-c}
  \prod_{\ell=1}^d
  |\Vb_\ell|^{-c_{-\ell}}
  \exp\left\{
    -\operatorname{tr}\left[
      (\Yb_{(j)}-\Mb_{(j)})^\star
      \Vb_j^{-1}
      (\Yb_{(j)}-\Mb_{(j)})
      \Vb_{-j}^{-\top}
    \right]
  \right\}.
\end{align}

\section{Factor analysis for complex-valued vectors}
\label{sec:app_fa_desc}

We first derive the PX-EM updates for the expanded parameter set $\tilde{\thetab}=(\tilde{\Lambdab},\Gammab,\Psib)$ under the working model in \eqref{eq:px-fct-mdl} of Section~\ref{sec:vec_model_derivation}. We then express the algorithm in terms of the reduced parameter set $\thetab=(\Lambdab,\Psib)$ for the factor model in \eqref{eq:fact-mdl} of Section~\ref{sec:vec_model_derivation}, where $\Lambdab=\tilde{\Lambdab}\Rb$,  $\Rb=\operatorname{chol}(\Gammab)$, and $\operatorname{chol}(\Ab)$ denotes the lower-triangular Cholesky factor satisfying $\Ab=\operatorname{chol}(\Ab)\operatorname{chol}(\Ab)^\star$. This reduced parameterization allows us to impose sparsity directly on $\Lambdab$ through a complex soft-thresholding operator.

\subsection{PX-EM objective}
\label{sec:app_PX_fa_desc}

Recall that the PX-EM working model in \eqref{eq:px-fct-mdl} and \eqref{eq:X_vec_PXEM_dataset_distn} is
\begin{align}
  \label{eq:app_pxem_working_mat_model}
  \Xb
  &=
  \tilde{\Lambdab}\tilde{\Zb}+\Eb,
  &
  \tilde{\Zb}
  &\sim
  \mathbb{C}\Ncal_{k\times n}
  \left(\mathbf{0},\Gammab,\Ib_n\right),
  &
  \Eb
  &\sim
  \mathbb{C}\Ncal_{p\times n}
  \left(\mathbf{0},\Psib,\Ib_n\right).
\end{align}
Therefore, the complete-data likelihood factors as
\begin{align*}
  f(\Xb,\tilde{\Zb};\tilde{\thetab})
  =
  f(\Xb\mid\tilde{\Zb};\tilde{\thetab})
  f(\tilde{\Zb};\tilde{\thetab}), \qquad
\Xb\mid\tilde{\Zb}
  \sim
  \mathbb{C}\Ncal_{p\times n}
  \left(
    \tilde{\Lambdab}\tilde{\Zb},
    \Psib,
    \Ib_n
  \right).
\end{align*}
Using the complex matrix-normal density in \eqref{eq:c-array-normal-matricization}, the two factors are
\begin{align*}
  f(\Xb\mid\tilde{\Zb};\tilde{\thetab})
  &=
  \pi^{-np}|\Psib|^{-n}
  \exp\left\{
    -\tr\left[
      (\Xb-\tilde{\Lambdab}\tilde{\Zb})^\star
      \Psib^{-1}
      (\Xb-\tilde{\Lambdab}\tilde{\Zb})
    \right]
  \right\},\\
  f(\tilde{\Zb};\tilde{\thetab})
  &=
  \pi^{-nk}|\Gammab|^{-n}
  \exp\left\{
    -\tr\left(
      \tilde{\Zb}^\star
      \Gammab^{-1}
      \tilde{\Zb}
    \right)
  \right\}.
\end{align*}

Treating $\Xb $ and $\tilde \Zb$ as fixed and omitting terms that do not depend on $\tilde \thetab = (\tilde{\Lambdab}, \Gammab, \Psib)$, the complete-data log-likelihood is
\begin{align*}
  \ell_{\mathrm{com}}(\tilde{\thetab})
  \propto\;&
  -n\log|\Gammab|
  -\tr\left(
    \tilde{\Zb}\tilde{\Zb}^\star\Gammab^{-1}
  \right)
  -n\log|\Psib|
  -\tr\left(
    \Xb^\star\Psib^{-1}\Xb
  \right)\\
  &-
  \tr\left(
    \tilde{\Lambdab}
    \tilde{\Zb}\tilde{\Zb}^\star
    \tilde{\Lambdab}^\star
    \Psib^{-1}
  \right)
  +
  \tr\left(
    \Xb^\star
    \Psib^{-1}
    \tilde{\Lambdab}
    \tilde{\Zb}
  \right) +
  \tr\left(
    \tilde{\Zb}^\star
    \tilde{\Lambdab}^\star
    \Psib^{-1}
    \Xb
  \right).
\end{align*}

\subsection{E-step}
At iteration $t$, if $\tilde{\thetab}^{(t)} = (\tilde{\Lambdab}^{(t)}, \Gammab^{(t)}, \Psib^{(t)})$ denotes the parameter estimates, then the E-step of the PX-EM procedure computes $-\EE \left\{ \ell_{\mathrm{com}}(\tilde{\Lambdab},\Gammab, \Psib) \mid \Xb, \tilde{\thetab}^{(t)}  \right\} $, which is proportional to
\begin{align}
  \label{eq:vec-comp-llk_app}
  \Qcal\left(
    \tilde{\Lambdab},
    \Gammab,
    \Psib
    \mid
    \tilde{\thetab}^{(t)}
  \right)
  =\;&
  n\log|\Gammab|
  +
  \tr\left(
    \widetilde{\Sbb}_{zz}^{(t)}
    \Gammab^{-1}
  \right)
  +
  n\log|\Psib|
  +
  \tr\left(
    \Xb^\star\Psib^{-1}\Xb
  \right)
  \nonumber\\
  &+
  \tr\left(
    \tilde{\Lambdab}
    \widetilde{\Sbb}_{zz}^{(t)}
    \tilde{\Lambdab}^\star
    \Psib^{-1}
  \right)
  -
  \tr\left(
    \Xb^\star
    \Psib^{-1}
    \tilde{\Lambdab}
    \tilde{\Zb}^{(t)}
  \right)
  \nonumber\\
  &-
  \tr\left(
    \tilde{\Zb}^{(t)\star}
    \tilde{\Lambdab}^\star
    \Psib^{-1}
    \Xb
  \right),
\end{align}
where the posterior expectations of the complete-data sufficient statistics are
\begin{align}
  \label{eq:sufficient_stat_defs}
  \tilde{\Zb}^{(t)}
  =
  \EE\left(
    \tilde{\Zb}
    \mid
    \Xb,\tilde{\thetab}^{(t)}
  \right),
  \qquad
  \widetilde{\Sbb}_{zz}^{(t)}
  =
  \EE\left(
    \tilde{\Zb}\tilde{\Zb}^\star
    \mid
    \Xb,\tilde{\thetab}^{(t)}
  \right).
\end{align}

\begin{proposition}[Posterior expectations of the complete-data sufficient statistics]
\label{prop:pxem_posterior_moments}

At iteration $t$, define
\begin{align*}
  \widetilde{\Sigmab}^{(t)}
  =
  \tilde{\Lambdab}^{(t)}
  \Gammab^{(t)}
  \tilde{\Lambdab}^{(t)\star}
  +
  \Psib^{(t)},\quad
  \widetilde{\Ab}^{(t)}
  =
  \Gammab^{(t)}
  \tilde{\Lambdab}^{(t)\star}
  \left(\widetilde{\Sigmab}^{(t)}\right)^{-1},\quad
  \widetilde{\Cb}^{(t)}
  =
  \Gammab^{(t)}
  -
  \widetilde{\Ab}^{(t)}
  \tilde{\Lambdab}^{(t)}
  \Gammab^{(t)}.
\end{align*}
Then, the conditional distribution of the latent factor matrix is
\begin{align*}
  \tilde{\Zb}\mid\Xb,\tilde{\thetab}^{(t)}
  \sim
  \mathbb{C}\Ncal_{k\times n}
  \left(
    \widetilde{\Ab}^{(t)}\Xb,
    \widetilde{\Cb}^{(t)},
    \Ib_n
  \right),
\end{align*}
and the posterior expectations of the complete-data sufficient statistics in \eqref{eq:sufficient_stat_defs} are
\begin{align}
  \label{eq:pxem_posterior_means}
  \tilde{\Zb}^{(t)}
  =
  \widetilde{\Ab}^{(t)}\Xb, \qquad
  \widetilde{\Sbb}_{zz}^{(t)}
  =
  n\widetilde{\Cb}^{(t)}
  +
  \tilde{\Zb}^{(t)}
  \tilde{\Zb}^{(t)\star}.
\end{align}
\end{proposition}

\begin{proof}
Using  \eqref{eq:app_pxem_working_mat_model}, define
\begin{align*}
\xb=\operatorname{vec}(\Xb),
\qquad
  \tilde{\zb}=\operatorname{vec}(\tilde{\Zb}),
  \qquad
\eb=\operatorname{vec}(\Eb),
  \qquad
\widetilde{\Sigmab}
=
\tilde{\Lambdab}\Gammab\tilde{\Lambdab}^{\star}
+
\Psib.
\end{align*}
The vectorized form of the working model in \eqref{eq:app_pxem_working_mat_model} is
\begin{align}
\label{eq:2}
\xb
=
(\Ib_n\otimes\tilde{\Lambdab})\tilde{\zb}
+
\eb, \qquad
\tilde{\zb}
\sim
\mathbb{C}\Ncal_{nk}
\left(
  \mathbf{0},
  \Ib_n\otimes\Gammab
  \right),
  \qquad
\eb
  \sim
\mathbb{C}\Ncal_{np}
\left(
  \mathbf{0},
  \Ib_n\otimes \Psib
  \right),
\end{align}
and $\xb$ is marginally distributed as $\mathbb{C}\Ncal_{np} \left(\mathbf{0}, \Ib_n\otimes\widetilde{\Sigmab} \right)$. Using \eqref{eq:2},
\begin{align}
  \label{eq:joint_pxem_vec_app}
  \begin{bmatrix}
    \xb\\
    \tilde{\zb}
  \end{bmatrix}
  \sim
  \mathbb{C}\Ncal_{n(p+k)}
  \left(
    \begin{bmatrix}
      \mathbf{0}\\
      \mathbf{0}
    \end{bmatrix},
    \begin{bmatrix}
      \Ib_n\otimes\widetilde{\Sigmab}
      &
      \Ib_n\otimes\tilde{\Lambdab}\Gammab
      \\
      \Ib_n\otimes\Gammab\tilde{\Lambdab}^{\star}
      &
      \Ib_n\otimes\Gammab
    \end{bmatrix}
  \right).
\end{align}

The conditional distribution formula for a proper complex normal vector gives
\begin{align}\label{eq:cond-mn-var}
  \EE(\tilde{\zb}\mid\xb)
  =
  \left[
    \Ib_n
    \otimes
    \Gammab\tilde{\Lambdab}^{\star}
    \widetilde{\Sigmab}^{-1}
  \right]\xb,\qquad
  \operatorname{Cov}(\tilde{\zb}\mid\xb)
  =
  \Ib_n
  \otimes
  \left(
    \Gammab
    -
    \Gammab\tilde{\Lambdab}^{\star}
    \widetilde{\Sigmab}^{-1}
    \tilde{\Lambdab}\Gammab
  \right).
\end{align}
Based on the forms of the conditional mean vector and covariance matrix above, define
\begin{align}\label{eq:def1}
  \widetilde{\Ab}
  =
  \Gammab\tilde{\Lambdab}^{\star}
  \widetilde{\Sigmab}^{-1},\qquad
  \widetilde{\Cb}
  =
  \Gammab
  -
  \Gammab\tilde{\Lambdab}^{\star}
  \widetilde{\Sigmab}^{-1}
  \tilde{\Lambdab}\Gammab.
\end{align}
Using the identity $(\Ib_n\otimes\widetilde{\Ab}) \operatorname{vec}(\Xb) = \operatorname{vec}(\widetilde{\Ab}\Xb)$ and \eqref{eq:def1}, the matrix-normal conditional distribution in \eqref{eq:cond-mn-var} is
\begin{align}
  \label{eq:posterior_pxem_matrix_app}
  \tilde{\Zb}\mid\Xb
  \sim
  \mathbb{C}\Ncal_{k\times n}
  \left(
    \widetilde{\Ab}\Xb,
    \widetilde{\Cb},
    \Ib_n
  \right).
\end{align}
The Woodbury matrix identity \citep{Har97} gives the equivalent expressions
\begin{align}
  \label{eq:pxem_conditional_cov_woodbury}
  \widetilde{\Cb}
  =
  \left(
    \Gammab^{-1}
    +
    \tilde{\Lambdab}^{\star}
    \Psib^{-1}
    \tilde{\Lambdab}
  \right)^{-1},\qquad
  \widetilde{\Ab}
  =
  \widetilde{\Cb}
  \tilde{\Lambdab}^{\star}
  \Psib^{-1}.
\end{align}
Substituting the parameter estimates $\tilde{\thetab}^{(t)} = (\tilde{\Lambdab}^{(t)},\Gammab^{(t)},\Psib^{(t)})$ into \eqref{eq:posterior_pxem_matrix_app} yields $\tilde{\Zb}^{(t)} = \widetilde{\Ab}^{(t)}\Xb$, which is the posterior expectation of the first complete-data sufficient statistic.

We finish the proof by deriving the posterior expectation of the second complete-data sufficient statistic. Using \eqref{eq:posterior_pxem_matrix_app} and   \eqref{eq:pxem_conditional_cov_woodbury}, write
\begin{align}
\label{eq:3}
\tilde{\Zb}
=
\tilde{\Zb}^{(t)}
+
\Ub, \qquad \Ub \mid \Xb
\sim
\mathbb{C}\Ncal_{k\times n}
\left(
  \mathbf{0},
  \widetilde{\Cb}^{(t)},
  \Ib_n
\right).
\end{align}
If $\ub_1,\ldots,\ub_n$ denote the columns of $\Ub$, then $\EE(\ub_i\ub_i^\star\mid\Xb) = \widetilde{\Cb}^{(t)}$ for $i=1, \ldots, n$. Therefore,
\begin{align*}
  \EE(\Ub\Ub^\star\mid\Xb)
  =
  \sum_{i=1}^n
  \EE(\ub_i\ub_i^\star\mid\Xb)
  =
  n\widetilde{\Cb}^{(t)}.
\end{align*}
Because \(\EE(\Ub\mid\Xb)=\mathbf{0}\), we have that
\begin{align*}
  \widetilde{\Sbb}_{zz}^{(t)}
  =
  \EE\left(
    \tilde{\Zb}\tilde{\Zb}^{\star}
    \mid
    \Xb,\tilde{\thetab}^{(t)}
  \right)
  =
  n\widetilde{\Cb}^{(t)}
  +
  \tilde{\Zb}^{(t)}
  \tilde{\Zb}^{(t)\star}.
\end{align*}
The proposition is proved.
\end{proof}

\subsection{M-step updates for $\tilde{\Lambdab}$ and $\Gammab$}
\label{sec:app_fa_desc_m_step}

At iteration $t$, we hold $\Psib=\Psib^{(t)}$ fixed and minimize $\Qcal(\tilde{\Lambdab}, \Gammab, \Psib^{(t)} \mid \tilde{\thetab}^{(t)})$ in \eqref{eq:vec-comp-llk_app} with respect to $\tilde{\Lambdab}$ and $\Gammab$. The conditional sufficient statistics $\tilde{\Zb}^{(t)}$ and $\widetilde{\Sbb}_{zz}^{(t)}$ are given in \eqref{eq:pxem_posterior_means}. The terms involving $\Gammab$ are
\begin{align*}
  \Qcal_{\Gammab}(\Gammab)
=
n\log|\Gammab|
+
\tr\left(
  \widetilde{\Sbb}_{zz}^{(t)}\Gammab^{-1}
\right),
\end{align*}
and their differential is
\begin{align*}
  \mathrm{d}\Qcal_{\Gammab}(\Gammab)
  &=
  \tr\left[
    \left\{
      n\Gammab^{-1}
      -
      \Gammab^{-1}
      \widetilde{\Sbb}_{zz}^{(t)}
      \Gammab^{-1}
    \right\}
    \mathrm{d}\Gammab
  \right].
\end{align*}
Setting this differential equal to zero for every
$\mathrm{d}\Gammab$ gives
\begin{align}
  \label{eq:app_gamma_update}
  \Gammab^{(t+1)}
  =
  \frac{1}{n}\widetilde{\Sbb}_{zz}^{(t)}.
\end{align}

Similarly, the terms involving $\tilde{\Lambdab}$ in  $\Qcal(\tilde{\Lambdab}, \Gammab^{(t+1)}, \Psib^{(t)} \mid \tilde{\thetab}^{(t)})$ are
\begin{align*}
  \Qcal_{\tilde{\Lambdab}}(\tilde{\Lambdab})
  =
  \tr\left[
    \tilde{\Lambdab}
    \widetilde{\Sbb}_{zz}^{(t)}
    \tilde{\Lambdab}^{\star}
    \left(\Psib^{(t)}\right)^{-1}
  \right]
  -
  \tr\left[
    \Xb^\star
    \left(\Psib^{(t)}\right)^{-1}
    \tilde{\Lambdab}
    \tilde{\Zb}^{(t)}
  \right]
  -
  \tr\left[
    \tilde{\Zb}^{(t)\star}
    \tilde{\Lambdab}^{\star}
    \left(\Psib^{(t)}\right)^{-1}
    \Xb
  \right].
\end{align*}
Define
\begin{align*}
  \widetilde \Hb_{zz}^{(t)}
=
\operatorname{chol}\left(
  \widetilde{\Sbb}_{zz}^{(t)}
\right),
\qquad
\widetilde{\Sbb}_{zz}^{(t)}
=
\widetilde \Hb_{zz}^{(t)}
  \widetilde \Hb_{zz}^{(t)\star},
  \qquad
\widetilde \Hb_{zz}^{(t)-\star}
=
\left\{
  \widetilde \Hb_{zz}^{(t)-1}
\right\}^{\star}.
\end{align*}
The objective for updating  $\tilde{\Lambdab}$ is written as
\begin{align}
  \label{eq:lam_tilde_quadratic_form}
  \Qcal_{\tilde{\Lambdab}}(\tilde{\Lambdab})
  =
  \tr\left[
    \left\{
      \tilde{\Lambdab} \widetilde \Hb_{zz}^{(t)}
      -
      \Xb\tilde{\Zb}^{(t)\star}
      \widetilde \Hb_{zz}^{(t)-\star}
    \right\}^{\star}
    \left(\Psib^{(t)}\right)^{-1}
    \left\{
      \tilde{\Lambdab} \widetilde \Hb_{zz}^{(t)}
      -
      \Xb\tilde{\Zb}^{(t)\star}
      \widetilde \Hb_{zz}^{(t)-\star}
    \right\}
  \right]
  +
  C_t,
\end{align}
where $C_t$ is an additive term that does not depend on $\tilde{\Lambdab}$. Because $\Psib^{(t)}$ is positive definite, $\Qcal_{\tilde{\Lambdab}}(\tilde{\Lambdab})$ is minimized when
$\tilde{\Lambdab} \widetilde \Hb_{zz}^{(t)} = \Xb\tilde{\Zb}^{(t)\star} \widetilde \Hb_{zz}^{(t)-\star}$, implying that
\begin{align}
  \label{eq:app_tilde_lambda_update}
  \tilde{\Lambdab}^{(t+1)}
  =
  \Xb\tilde{\Zb}^{(t)\star}
  \left(
    \widetilde{\Sbb}_{zz}^{(t)}
  \right)^{-1}.
\end{align}
The parameter-reduction step constructs the loading matrix
$\Lambdab^{(t+1)}$ from $\tilde{\Lambdab}^{(t+1)}$ in   \eqref{eq:app_tilde_lambda_update} as
\begin{align*}
  \Lambdab^{(t+1)}
=
\tilde{\Lambdab}^{(t+1)}
\Rb^{(t+1)},
\qquad
\Rb^{(t+1)}
=
\operatorname{chol}\left(
  \Gammab^{(t+1)}
\right).
\end{align*}
Using \eqref{eq:app_gamma_update},
\begin{align*}
  \Rb^{(t+1)}
=
\frac{1}{\sqrt{n}} \widetilde \Hb_{zz}^{(t)},
\end{align*}
which further implies
\begin{align}
  \label{eq:lam_construction}
  \Lambdab^{(t+1)}
  &=
  \frac{1}{\sqrt{n}}
  \Xb\tilde{\Zb}^{(t)\star}
 \widetilde \Hb_{zz}^{(t)-\star}.
\end{align}

\subsection{M-step update for $\Psib$}
\label{sec:app_Psi_estimate}

We now hold $(\tilde{\Lambdab},  \Gammab)$ fixed at $(\tilde{\Lambdab}^{(t+1)},  \Gammab^{(t+1)})$ in  \eqref{eq:vec-comp-llk_app} and update the diagonal matrix of residual error variances $\Psib$. The terms in the expected complete-data negative log-likelihood that depend on $\Psib$ are
\begin{align*}
\Qcal\left(
  \tilde{\Lambdab}^{(t+1)},
  \Gammab^{(t+1)},
  \Psib
  \mid
  \tilde{\thetab}^{(t)}
\right)
\doteq
\Qcal_{\Psib}(\Psib)
=
n\log|\Psib|
+
\tr\left[
  \Psib^{-1}\Rb_x^{(t+1)}
\right],
\end{align*}
where $\Rb_x^{(t+1)} = \Xb\Xb^\star - \Xb\tilde{\Zb}^{(t)\star} \tilde{\Lambdab}^{(t+1)\star} - \tilde{\Lambdab}^{(t+1)} \tilde{\Zb}^{(t)} \Xb^\star + \tilde{\Lambdab}^{(t+1)} \widetilde{\Sbb}_{zz}^{(t)} \tilde{\Lambdab}^{(t+1)\star}$ and $\doteq$ denotes equality up to an additive term that does not
depend on $\Psib$. Because $\Psib$ is diagonal, minimizing $\Qcal_{\Psib}(\Psib)$ gives
\begin{align}
  \label{eq:psi_complete_data_update}
  \Psib^{(t+1)}
  =
  \frac{1}{n}
  \Rb_x^{(t+1)}
  \circ\Ib_p,
\end{align}
where $\circ$ is the Hadamard product.

We simplify the analytic form of $\Rb_x^{(t+1)}$ using \eqref{eq:app_gamma_update}, \eqref{eq:app_tilde_lambda_update}, and \eqref{eq:lam_construction}. Using \eqref{eq:app_tilde_lambda_update}, substitute $\tilde{\Lambdab}^{(t+1)} \widetilde{\Sbb}_{zz}^{(t)}$ for $\Xb\tilde{\Zb}^{(t)\star}$ in the definition of $\Rb_x^{(t+1)}$ to obtain that
\begin{align*}
  \Rb_x^{(t+1)}
  =
  \Xb\Xb^\star
  -
  \tilde{\Lambdab}^{(t+1)}
  \widetilde{\Sbb}_{zz}^{(t)}
  \tilde{\Lambdab}^{(t+1)\star}.
\end{align*}
Using
$\Gammab^{(t+1)}
=
n^{-1}\widetilde{\Sbb}_{zz}^{(t)}$ from
\eqref{eq:app_gamma_update}, we obtain
\begin{align}
  \label{eq:psi_tilde_update}
  \Psib^{(t+1)}
  =
  \left[
    \frac{1}{n}\Xb\Xb^\star
    -
    \tilde{\Lambdab}^{(t+1)}
    \Gammab^{(t+1)}
    \tilde{\Lambdab}^{(t+1)\star}
  \right]
  \circ\Ib_p.
\end{align}
Finally, using \eqref{eq:lam_construction}, the update in \eqref{eq:psi_tilde_update} can be written in terms of the reduced loading matrix, $\Lambdab^{(t+1)}$, as
\begin{align}
  \label{eq:psi_mle_estimate}
  \Psib^{(t+1)}
  =
  \left[
    \frac{1}{n}\Xb\Xb^\star
    -
    \Lambdab^{(t+1)}
    \Lambdab^{(t+1)\star}
  \right]
  \circ\Ib_p.
\end{align}

\subsection{Lasso regularization and sparse PX-EM updates}
\label{sec:lasso_derivation_app}

We first express the PX-EM objective in terms of the reduced parameter set $\thetab=(\Lambdab,\Psib)$. We then add an elementwise modulus penalty to the reduced loading objective and derive closed-form updates for $\Lambdab$ and $\Psib$. All expected complete-data objectives in this subsection are normalized by $n$ before the penalty is added.

\subsubsection{Reduced-form PX-EM objective}
\label{sec:reduced_form_pxem_obj}

We reformulate the PX-EM objective in \eqref{eq:vec-comp-llk_app} using only the reduced parameters.
Let $\Cb^{(t)}$, $\Zb^{(t)}$, and $\Sbb_{zz}^{(t)}$ denote the reduced counterparts of $\widetilde{\Cb}^{(t)}$, $\tilde{\Zb}^{(t)}$, and $\widetilde{\Sbb}_{zz}^{(t)}$ from
Proposition~\ref{prop:pxem_posterior_moments}, respectively, where
\begin{align}\label{eq:expected_Z_given_X_final}
  \Cb^{(t)}
  &=
  \left[
    \Ib_k
    +
    \Lambdab^{(t)\star}
    \left(\Psib^{(t)}\right)^{-1}
    \Lambdab^{(t)}
  \right]^{-1}, \qquad
  \Zb^{(t)}
  =
  \Cb^{(t)}
  \Lambdab^{(t)\star}
  \left(\Psib^{(t)}\right)^{-1}
  \Xb,\nonumber\\
  \Sbb_{zz}^{(t)}
  &=
  n\Cb^{(t)}
  +
  \Zb^{(t)}\Zb^{(t)\star}.
\end{align}
Similarly, the reduced counterparts of the terms used to update $\tilde{\Lambdab}$ in \eqref{eq:lam_tilde_quadratic_form} and \eqref{eq:lam_construction} are
\begin{align}
  \label{eq:A_updated_def}
  \Hb_{zz}^{(t)} = \operatorname{chol}\left(\Sbb_{zz}^{(t)} \right), \qquad
  \Hb_{zz}^{(t)-\star}
  = \left\{\Hb_{zz}^{(t)-1} \right\}^{\star}, \qquad
  \Ab^{(t)}
  =
  \frac{1}{\sqrt{n}}
  \Xb
  \Zb^{(t)\star}
  \Hb_{zz}^{(t)-\star}.
\end{align}

\begin{proposition}[Reduced-form PX-EM loading objective]
\label{prop:reduced_pxem_loading_objective}

After updating $\Gammab$ according to \eqref{eq:app_gamma_update} and applying the parameter-reduction map in  \eqref{eq:lam_construction}, the loading-dependent part of the normalized PX-EM objective in \eqref{eq:vec-comp-llk_app} satisfies
\begin{align}
  \label{eq:lam_quadratic_form_final_app}
  \frac{1}{n}
  \Qcal\left(
    \tilde{\Lambdab},
    \Gammab^{(t+1)},
    \Psib^{(t)}
    \mid
    \tilde{\thetab}^{(t)}
  \right)
  \doteq \Qcal_{\Lambdab}^{\mathrm{red}}
  \left(
    \Lambdab
    \mid
    \thetab^{(t)}
  \right) =
  \tr\left[
    \left(
      \Lambdab-\Ab^{(t)}
    \right)^\star
    \left(\Psib^{(t)}\right)^{-1}
    \left(
      \Lambdab-\Ab^{(t)}
    \right)
  \right],
\end{align}
where $\Lambdab = \tilde{\Lambdab} \operatorname{chol}\left(\Gammab^{(t+1)} \right)$, $\Ab^{(t)}$ is defined in \eqref{eq:lam_quadratic_form_final_app}, and $\doteq$ denotes equality up to an additive term that does not depend on $\Lambdab$. The unique unpenalized minimizer of the objective in  \eqref{eq:lam_quadratic_form_final_app} is $\Lambdab^{(t+1)} = \Ab^{(t)}$.
\end{proposition}

\begin{proof}
Equation~\eqref{eq:3} in the proof of Proposition~\ref{prop:pxem_posterior_moments} shows that the conditional mode-$1$ covariance of $\tilde{\Zb}$ given $\Xb$ under the expanded working
model is
\begin{align*}
  \widetilde{\Cb}^{(t)}
  =
  \left[
    \Gammab^{(t)-1}
    +
    \tilde{\Lambdab}^{(t)\star}
    \left(\Psib^{(t)}\right)^{-1}
    \tilde{\Lambdab}^{(t)}
  \right]^{-1}.
\end{align*}
Furthermore, $\Gammab^{(t)} = \Rb^{(t)}\Rb^{(t)\star}$ and $\tilde{\Lambdab}^{(t)} = \Lambdab^{(t)}\Rb^{(t)-1}$. Combining these equations with \eqref{eq:pxem_conditional_cov_woodbury}, \eqref{eq:pxem_posterior_means}, and \eqref{eq:lam_quadratic_form_final_app}, we obtain that
\begin{align}
  \label{eq:expanded_reduced_conditional_covariance}
  \widetilde{\Cb}^{(t)}
  =
  \Rb^{(t)}
  \Cb^{(t)}
  \Rb^{(t)\star}, \qquad
\tilde{\Zb}^{(t)}
  =
  \Rb^{(t)}\Zb^{(t)}, \qquad
  \widetilde{\Sbb}_{zz}^{(t)}
  =
  \Rb^{(t)}
  \Sbb_{zz}^{(t)}
  \Rb^{(t)\star}.
\end{align}
Because $\Rb^{(t)}$ and $\Hb_{zz}^{(t)}$ are lower triangular with positive diagonal entries,
\begin{align}
  \label{eq:expanded_reduced_cholesky}
  \widetilde{\Sbb}_{zz}^{(t)}
  =
  \Rb^{(t)}
  \Sbb_{zz}^{(t)}
  \Rb^{(t)\star} = \left(\Rb^{(t)}\Hb_{zz}^{(t)} \right)\left( \Rb^{(t)}\Hb_{zz}^{(t)}\right)^{\star}, \qquad
  \widetilde \Hb_{zz}^{(t)} = \operatorname{chol}\left(
    \widetilde{\Sbb}_{zz}^{(t)}
  \right)
  =
  \Rb^{(t)}\Hb_{zz}^{(t)}.
\end{align}
Therefore, the parameter-reduction map in \eqref{eq:lam_construction} implies
\begin{align}
  \label{eq:tilde_lambda_reduction_identity}
 \Lambdab = \frac{1}{\sqrt{n}} \Xb \left(\Rb^{(t)} \Zb^{(t)} \right)^\star \left(\Rb^{(t)}\Hb_{zz}^{(t)} \right)^{- \star} = \frac{1}{\sqrt{n}} \Xb \Zb^{(t) \star}  \Hb_{zz}^{(t)- \star} = \Ab^{(t)},
\end{align}
where $\Ab^{(t)}$ is defined in \eqref{eq:A_updated_def}.

Recall that the terms in the PX-EM objective involving $\tilde{\Lambdab}$ in \eqref{eq:lam_tilde_quadratic_form} are
\begin{align*}
  \Qcal_{\tilde{\Lambdab}}(\tilde{\Lambdab})
  =
  \tr\left[
    \Psib^{-1}
    \tilde{\Lambdab}
    \widetilde{\Sbb}_{zz}^{(t)}
    \tilde{\Lambdab}^{\star}
  \right]
  -
  \tr\left[
    \Psib^{-1}
    \Xb
    \tilde{\Zb}^{(t)\star}
    \tilde{\Lambdab}^{\star}
  \right]
  -
  \tr\left[
    \Psib^{-1}
    \tilde{\Lambdab}
    \tilde{\Zb}^{(t)}
    \Xb^\star
  \right].
\end{align*}
Adding $\tr\{\Psib^{-1}\Xb\Xb^\star\}$, which does not depend on $\tilde{\Lambdab}$, and completing the square gives
\begin{align}\label{eq:expanded_residual_matrix}
  \Qcal_{\tilde{\Lambdab}}(\tilde{\Lambdab})
  \doteq
  \tr\left[
    \Psib^{-1}
    \widetilde{\Rb}_x(\tilde{\Lambdab})
  \right], \qquad
  \widetilde{\Rb}_x(\tilde{\Lambdab})
  =
  \Xb\Xb^\star
  -
  \Xb\tilde{\Zb}^{(t)\star}\tilde{\Lambdab}^\star
  -
  \tilde{\Lambdab}\tilde{\Zb}^{(t)}\Xb^\star
  +
  \tilde{\Lambdab}
  \widetilde{\Sbb}_{zz}^{(t)}
  \tilde{\Lambdab}^\star.
\end{align}
where the matrix $\widetilde{\Rb}_x(\tilde{\Lambdab})$ is the conditional expected residual cross-product matrix,
\begin{align*}
  \widetilde{\Rb}_x(\tilde{\Lambdab})
  =
  \EE\left[
    \left(
      \Xb-\tilde{\Lambdab}\tilde{\Zb}
    \right)
    \left(
      \Xb-\tilde{\Lambdab}\tilde{\Zb}
    \right)^\star
    \,\middle|\,
    \Xb,\tilde{\thetab}^{(t)}
  \right].
\end{align*}
Using  \eqref{eq:expanded_reduced_conditional_covariance}, \eqref{eq:expanded_reduced_cholesky}, and  \eqref{eq:tilde_lambda_reduction_identity} we obtain
\begin{align*}
  \Xb\tilde{\Zb}^{(t)\star}\tilde{\Lambdab}^\star &= \Xb  (\Rb^{(t)} \Zb^{(t)})^{\star} \left[ \Lambdab \left(n^{-1/2}\Rb^{(t)}\Hb_{zz}^{(t)}\right)^{-1} \right]^{\star}
  = n^{1/2} \Xb \Zb^{(t) \star}  \Hb_{zz}^{(t)- \star} \Lambdab^{\star}=
  n \Ab^{(t)}\Lambdab^\star,\\
  \tilde{\Lambdab}\tilde{\Zb}^{(t)}\Xb^\star
                                                  &=
 \Lambdab \left(n^{-1/2}\Rb^{(t)}\Hb_{zz}^{(t)}\right)^{-1}\Rb^{(t)} \Zb^{(t)}\Xb^\star
  = n\Lambdab\Ab^{(t)\star},\\
  \tilde{\Lambdab}
  \widetilde{\Sbb}_{zz}^{(t)}
  \tilde{\Lambdab}^\star
  &=  \Lambdab \left(n^{-1/2}\Rb^{(t)}\Hb_{zz}^{(t)}\right)^{-1} \Rb^{(t)}
  \Sbb_{zz}^{(t)}
  \Rb^{(t)\star} \left(n^{-1/2}\Rb^{(t)}\Hb_{zz}^{(t)}\right)^{-\star} \Lambdab^\star
  = n\Lambdab\Lambdab^\star.
\end{align*}

Substituting these expressions into \eqref{eq:expanded_residual_matrix}, define
\begin{align*}
  \Bb_t(\Lambdab) = \frac{1}{n}
\widetilde{\Rb}_x(\tilde{\Lambdab})
  =
  \frac{1}{n}\Xb\Xb^\star
  -
  \Ab^{(t)}\Lambdab^\star
  -
  \Lambdab\Ab^{(t)\star}
  +
  \Lambdab\Lambdab^\star,
\end{align*}
where $\Lambdab$ is a candidate reduced loading matrix. Adding and subtracting $\Ab^{(t)}\Ab^{(t)\star}$ gives
\begin{align}
  \label{eq:reduced_residual_square}
  \Bb_t(\Lambdab)
  =
    \frac{1}{n}\Xb\Xb^\star - \Ab^{(t)}\Ab^{(t)\star}
  +
  \left(
    \Lambdab-\Ab^{(t)}
  \right)
  \left(
    \Lambdab-\Ab^{(t)}
  \right)^\star.
\end{align}

Substituting the expression for $\Bb_t(\Lambdab)$ into \eqref{eq:vec-comp-llk_app} shows that the normalized expected complete-data negative log-likelihood satisfies
\begin{align}
  \label{eq:reduced_pxem_surrogate}
  \frac{1}{n}
  \Qcal\left(
    \tilde{\Lambdab},
    \Gammab^{(t+1)},
    \Psib
    \mid
    \tilde{\thetab}^{(t)}
  \right)
  \doteq
  \log|\Psib|
  +
  \tr\left[
    \Psib^{-1}\Bb_t(\Lambdab)
  \right],
\end{align}
where we have ignored additive terms that does not depend on $\Lambdab$ or $\Psib$.

Setting $\Psib=\Psib^{(t)}$ in \eqref{eq:reduced_pxem_surrogate} gives the objective for updating $\Lambdab$ as
\begin{align}
  \label{eq:lam_quadratic_form_final_tmp}
  \Qcal_{\Lambdab}^{\mathrm{red}}
  \left(
    \Lambdab
    \mid
    \thetab^{(t)}
  \right)
  =
  \tr\left[
    \left(
      \Lambdab-\Ab^{(t)}
    \right)^\star
    \left(\Psib^{(t)}\right)^{-1}
    \left(
      \Lambdab-\Ab^{(t)}
    \right)
  \right].
\end{align}
Because $\Psib^{(t)}$ is positive definite, the unique unpenalized minimizer of \eqref{eq:lam_quadratic_form_final_tmp} is $\Lambdab^{(t+1)}=\Ab^{(t)}$. The proposition is proved.
\end{proof}

\subsubsection{Weighted complex soft-thresholding}

The elementwise complex lasso norm is $  \|\Lambdab\|_{1,\CC} = \sum_{r=1}^p\sum_{c=1}^k |\lambda_{rc}|$. We impose sparsity by augmenting $\Qcal_{\Lambdab}^{\mathrm{red}} \left(\Lambdab \mid \thetab^{(t)} \right)$ in \eqref{eq:lam_quadratic_form_final_app} with $\rho\|\Lambdab\|_{1,\CC}$, where $\rho\geq0$. Before deriving the sparse update for $\Lambdab$, we establish the elementwise complex soft-thresholding operator in the following lemma.

\begin{lemma}[Weighted complex soft-thresholding]
\label{lem:complex_soft_thresholding}

Let $\Ab=(a_{rc})\in\CC^{p\times k}$ and $\Psib = \diag\left(\psi_{11},\ldots,\psi_{pp} \right)$, where the diagonal entries of $\Psib$ are positive. The optimization problem
\begin{align}
  \label{eq:general_complex_lasso_objective}
  \min_{\Lambdab\in\CC^{p\times k}}
  \left\{
    \tr\left[
      (\Lambdab-\Ab)^\star
      \Psib^{-1}
      (\Lambdab-\Ab)
    \right]
    +
    \rho\|\Lambdab\|_{1,\CC}
  \right\}
\end{align}
has a unique minimizer $\widehat{\Lambdab}=(\widehat{\lambda}_{rc})$ with entries $\widehat{\lambda}_{rc}
  = \mathcal{T}_{\rho\psi_{rr}/2}(a_{rc})$, where
\begin{align}
  \label{eq:complex_soft_threshold_operator}
  \mathcal{T}_{\tau}(z)
  =
  \begin{cases}
    \displaystyle
    \left(
      1-\frac{\tau}{|z|}
    \right)z,
    & |z|>\tau,\\[2mm]
    0,
    & |z|\leq\tau.
  \end{cases}
\end{align}
The operator $\mathcal{T}_{\tau}$ is the complex soft-thresholding operator. It preserves the phase of each entry that remains nonzero while reducing its modulus by $\tau$.
\end{lemma}

\begin{proof}
Because $\Psib$ is diagonal, the objective in
\eqref{eq:general_complex_lasso_objective} separates across the entries
of $\Lambdab$. For each $(r,c)$, it is sufficient to minimize
\begin{align*}
  g_{rc}(\lambda)
  =
  \frac{1}{\psi_{rr}}
  |\lambda-a_{rc}|^2
  +
  \rho|\lambda|.
\end{align*}
Using the polar representations $\lambda
= s\exp(\mathrm{i}\omega)$ and $a_{rc} = q\exp(\mathrm{i}\theta)$, where $s,q\geq0$ and $\theta, \omega \in [0, 2\pi)$, we obtain
\begin{align*}
  g_{rc}(\lambda)
  =
  \frac{1}{\psi_{rr}}
  \left[
    s^2+q^2-2sq\cos(\omega-\theta)
  \right]
  +
  \rho s.
\end{align*}
For fixed $s>0$, this expression is minimized over $\omega$ when $\omega=\theta$ modulo $2\pi$. When $s=0$, the phase is irrelevant. Therefore, minimizing $  g_{rc}(\lambda)$ reduces to
\begin{align*}
  \widehat s = \argmin_{s\geq0}
  \left\{
    \frac{1}{\psi_{rr}}(s-q)^2+\rho s
  \right\}, \qquad \widehat s = \left(q-\frac{\rho\psi_{rr}}{2} \right)_+,
\end{align*}
where $(x)_+=\max(0,x)$. The objective is strictly convex, so $\widehat s$ is its unique minimizer. Combining \eqref{eq:complex_soft_threshold_operator} with the optimal modulus and phase gives $\widehat \lambda_{rc} = \widehat s \exp(\mathrm{i}\mathrm{arg}(a_{rc})) = \mathcal{T}_{\rho\psi_{rr}/2}(a_{rc})$ for every $r$ and $c$. The lemma is proved.
\end{proof}

\subsubsection{Closed-form sparse PX-EM updates}
\label{sec:sparse_vec_model_derivation_app}

Combining Proposition~\ref{prop:reduced_pxem_loading_objective} and Lemma~\ref{lem:complex_soft_thresholding} gives the CM-step in the sparse PX-EM update for $\Lambdab$. We first restate Corollary \ref{cor:sparse_pxem_loading_update} from Section \ref{sec:sparse_vec_model_derivation} of the main manuscript.

\begin{corollary}[Closed-form sparse PX-EM updates]
\label{cor:sparse_pxem_loading_update_app}

Let $\Ab^{(t)}=(a_{rc}^{(t)})$ be defined in \eqref{eq:A_updated_def}, and let $\psi_{rr}^{(t)}$ denote the $r$th diagonal entry of $\Psib^{(t)}$. The unique minimizer over
$\Lambdab\in\CC^{p\times k}$ of
\begin{align}
  \label{eq:complex_lasso_loading_objective_app}
  \Qcal_{\Lambdab}^{\mathrm{red}}
  \left(
    \Lambdab
    \mid
    \thetab^{(t)}
  \right)
  +
  \rho\|\Lambdab\|_{1,\CC} =  \tr\left[
    \left(
      \Lambdab-\Ab^{(t)}
    \right)^\star
    \left(\Psib^{(t)}\right)^{-1}
    \left(
      \Lambdab-\Ab^{(t)}
    \right)
  \right] +
  \rho\|\Lambdab\|_{1,\CC},
\end{align}
is $\Lambdab^{(t+1)}$, whose entries are
\begin{align}
  \label{eq:sparse_lam_mod_update_app}
  \lambda_{rc}^{(t+1)}
  =
  \mathcal{T}_{\rho\psi_{rr}^{(t)}/2}
  \left(
    a_{rc}^{(t)}
  \right),
  \qquad
  r=1,\ldots,p,\quad c=1,\ldots,k,
\end{align}
where $\mathcal{T}_{\tau}$ is the complex soft-thresholding operator
defined in \eqref{eq:complex_soft_threshold_operator} as
\begin{align}
  \label{eq:sparse_lam_piecewise_update_app}
  \lambda_{rc}^{(t+1)}
  =
  \begin{cases}
    \displaystyle
    \left(
      1-
      \frac{\rho\psi_{rr}^{(t)}}
      {2|a_{rc}^{(t)}|}
    \right)
    a_{rc}^{(t)},
    &
    |a_{rc}^{(t)}|
    >
    \dfrac{\rho\psi_{rr}^{(t)}}{2},\\[3mm]
    0,
    &
    |a_{rc}^{(t)}|
    \leq
    \dfrac{\rho\psi_{rr}^{(t)}}{2}.
  \end{cases}
\end{align}
All entries of $\Lambdab^{(t+1)}$ can be updated simultaneously.
\end{corollary}

\begin{proof}
The form of $\Lambdab^{(t+1)}$ in \eqref{eq:sparse_lam_mod_update_app} follows directly from
Lemma~\ref{lem:complex_soft_thresholding} after setting $\Ab=\Ab^{(t)}$ and $\Psib=\Psib^{(t)}$. The corollary is proved.
\end{proof}

The update for the diagonal residual covariance matrix holds $\Lambdab=\Lambdab^{(t+1)}$ fixed. The unpenalized PX-EM procedure for $\Psib$ in \eqref{eq:psi_mle_estimate} depends on the loading matrix only through $\Lambdab\Lambdab^\star$ and is invariant to unitary rotations of $\Lambdab$. After parameter reduction and complex soft-thresholding, we replace the unpenalized reduced loading estimate in \eqref{eq:psi_mle_estimate} with the sparse estimate $\Lambdab^{(t+1)}$ and update $\Psib$ as
\begin{align}
  \Psib^{(t+1)}
  &=
  \left[
    \frac{1}{n}\Xb\Xb^\star
    -
    \Lambdab^{(t+1)}
    \Lambdab^{(t+1)\star}
  \right]
  \circ\Ib_p.
  \label{eq:sparse-psi-update-app}
\end{align}
Therefore, the sparsity structure of $\Lambdab^{(t+1)}$ is retained, and each diagonal entry of $\Psib^{(t+1)}$ represents the empirical marginal variance not explained by the retained sparse factor structure. We refer to the resulting procedure as the sparse PX-EM procedure.

\section{Factor analysis for complex-valued arrays}
\label{sec:app_ssfa_desc}

The separable array-valued extension of the complex sparse factor model posits
\begin{align}
  \label{eq:SSFA_model_app}
  \Xcal_i
  \sim
  \CC\Ncal_{p_1\times\cdots\times p_d}
  \left(
    \mathbf{0},
    \Sigmab_1,\ldots,\Sigmab_d
  \right),
  \qquad
  \Sigmab_j
  =
  \Lambdab_j\Lambdab_j^\star+\Psib_j,
\end{align}
where
$\Lambdab_j\in\CC^{p_j\times k_j}$ is the mode-$j$ loading matrix and
$\Psib_j$ is a positive definite diagonal matrix for $j=1, \ldots, d$. The mode-$j$ parameters and the model parameters are
\begin{align}
\label{eq:sfa-mdl-par}
  \thetab_j=(\Lambdab_j,\Psib_j),
  \qquad
  \thetab=(\thetab_1,\ldots,\thetab_d).
\end{align}
Let $\Xcal\in\CC^{p_1\times\cdots\times p_d\times n}$ denote the array obtained by stacking the observations along mode $d+1$. Then,
\begin{align*}
  \Xcal
  \sim
  \CC\Ncal_{p_1\times\cdots\times p_d\times n}
  \left(
    \mathbf{0},
    \Sigmab_1,\ldots,\Sigmab_d,\Ib_n
  \right).
\end{align*}

\subsection{Mode-wise whitening}
\label{sec:app_ssfa_whitening}

One SSFA PX-EM iteration consists of $d$ mode-specific update steps. During the update from $\thetab^{(t)}$ to $\thetab^{(t+1)}$, define $  \thetab^{[t,j-1]} = \thetab^{(t+(j-1)/d)}$ and $\thetab^{[t,j]} = \thetab^{(t+j/d)}$ as the parameter values immediately before and after the mode-$j$ update, respectively. In particular, $\thetab^{[t,0]}=\thetab^{(t)}$ at the beginning of the $(t+1)$th iteration, and $\thetab^{[t,d]}=\thetab^{(t+1)}$ at its end. At the beginning of the mode-$j$ update, define
\begin{align*}
  \Sigmab_\ell^{[t,j-1]}
  =
  \Lambdab_\ell^{[t,j-1]}
  \Lambdab_\ell^{[t,j-1]\star}
  +
  \Psib_\ell^{[t,j-1]}, \qquad
  \Lb_\ell^{[t,j-1]}
  =
  \operatorname{chol}
  \left(
    \Sigmab_\ell^{[t,j-1]}
  \right),
  \qquad
  \ell=1,\ldots,d.
\end{align*}

The mode-$j$ update is performed after whitening every mode except modes $j$ and $d+1$. For $\ell=1,\ldots,d$, define the whitening factors
\begin{align*}
  \Mb_{\ell\mid j}^{[t,j-1]}
  =
  \begin{cases}
    \Ib_{p_j},
    & \ell=j,\\[1mm]
    \left(
      \Lb_\ell^{[t,j-1]}
    \right)^{-1},
    & \ell\neq j.
  \end{cases}
\end{align*}
The resulting whitened array is
\begin{align}
  \label{eq:app_array_whitening}
  \breve{\Xcal}_{(j)}^{[t]}
  =
  \Xcal
  \times_1\Mb_{1\mid j}^{[t,j-1]}
  \times_2\cdots
  \times_d\Mb_{d\mid j}^{[t,j-1]}
  \times_{d+1}\Ib_n.
\end{align}
With the parameters of the remaining modes held fixed at their values in $\thetab^{[t,j-1]}$, the mode-$j$ conditional likelihood is equivalent to the likelihood implied by
\begin{align*}
  \breve{\Xcal}_{(j)}^{[t]}
  \sim
  \CC\Ncal_{p_1\times\cdots\times p_d\times n}
  \left(
    \mathbf{0},
    \Ib_{p_1},\ldots,\Ib_{p_{j-1}},
    \Sigmab_j,
    \Ib_{p_{j+1}},\ldots,\Ib_{p_d},
    \Ib_n
  \right).
\end{align*}

We represent $\breve{\Xcal}_{(j)}^{[t]}$ using a vector-valued complex factor model in \eqref{eq:fact-mdl} from Section \ref{sec:vec_model_derivation} of the main manuscript. Let $\breve{\Xb}_{(j)}^{[t]}$ denote the mode-$j$ matricization of $\breve{\Xcal}_{(j)}^{[t]}$, and define $n_j = n\prod_{\ell\neq j}p_\ell$. Then,
\begin{align}
  \label{eq:app_array_whitened_matrix_model}
  \breve{\Xb}_{(j)}^{[t]}
  \sim
  \CC\Ncal_{p_j\times n_j}
  \left(
    \mathbf{0},
    \Sigmab_j,
    \Ib_{n_j}
  \right).
\end{align}
Using $\Sigmab_j=\Lambdab_j\Lambdab_j^\star+\Psib_j$, the model in \eqref{eq:app_array_whitened_matrix_model} has the vector-valued factor model representation
\begin{align}
  \label{eq:app_array_whitened_factor_model}
  \breve{\Xb}_{(j)}^{[t]}
  &=
  \Lambdab_j\Zb_{(j)}+\Eb_{(j)},\qquad
  \Zb_{(j)}
  &\sim
  \CC\Ncal_{k_j\times n_j}
  \left(
    \mathbf{0},
    \Ib_{k_j},
    \Ib_{n_j}
  \right), \qquad
  \Eb_{(j)}
  &\sim
  \CC\Ncal_{p_j\times n_j}
  \left(
    \mathbf{0},
    \Psib_j,
    \Ib_{n_j}
  \right),
\end{align}
where $\Zb_{(j)}$ and $\Eb_{(j)}$ are independent. Therefore, after whitening and matricization, the mode-$j$ conditional model in \eqref{eq:app_array_whitened_factor_model} has the same form as the vector-valued factor model in Section~\ref{sec:vec_model_derivation} of the main manuscript, after making the substituting $(\breve{\Xb}_{(j)}^{[t]}, n_j, p_j, k_j, \Lambdab_j, \Psib_j)$ for $(\Xb, n, p, k, \Lambdab, \Psib)$.

\subsection{Mode-wise sparse PX-EM update}
\label{sec:app_ssfa_mode_update}

We first introduce the notation required to state the lasso-regularized parameter updates for the whitened and matricized mode-$j$ factor model in \eqref{eq:app_array_whitened_factor_model}. Fix $j\in\{1,\ldots,d\}$, and let $\rho \geq0$ denote the mode-$j$ penalty parameter. The following quantities are the mode-$j$ counterparts of those defined in \eqref{eq:expected_Z_given_X_final} and \eqref{eq:A_updated_def}. They are computed using $\breve{\Xb}_{(j)}^{[t]}$ and the current mode-$j$ parameters $\thetab_j^{[t,j-1]}$:
\begin{align}
  \label{eq:ssfa_mode_posterior_quantities}
  \Cb_j^{[t]}
  &=
  \left[
    \Ib_{k_j}
    +
    \Lambdab_j^{[t,j-1]\star}
    \left(\Psib_j^{[t,j-1]}\right)^{-1}
    \Lambdab_j^{[t,j-1]}
  \right]^{-1}, \qquad
  \Zb_{(j)}^{[t]}
  =
  \Cb_j^{[t]}
  \Lambdab_j^{[t,j-1]\star}
  \left(\Psib_j^{[t,j-1]}\right)^{-1}
  \breve{\Xb}_{(j)}^{[t]},
  \nonumber\\
  \Sbb_{j,zz}^{[t]}
  &=
  n_j\Cb_j^{[t]}
  +
  \Zb_{(j)}^{[t]}
  \Zb_{(j)}^{[t]\star},
  \qquad
  \Hb_{j,zz}^{[t]}
  =
  \operatorname{chol}
  \left(
    \Sbb_{j,zz}^{[t]}
  \right),
  \qquad
  \Ab_j^{[t]}
  =
  \frac{1}{\sqrt{n_j}}
  \breve{\Xb}_{(j)}^{[t]}
  \Zb_{(j)}^{[t]\star}
  \Hb_{j,zz}^{[t]-\star}.
\end{align}

Applying the results of Section~\ref{sec:lasso_derivation_app} to the quantities in \eqref{eq:ssfa_mode_posterior_quantities} gives the mode-$j$ parameter updates. Let
$\Ab_j^{[t]}=(a_{j,rc}^{[t]})$ and $\psi_{j,rr}^{[t,j-1]}$ denote the $r$th diagonal entry of $\Psib_j^{[t,j-1]}$. Corollary~\ref{cor:sparse_pxem_loading_update_app} in Section \ref{sec:sparse_vec_model_derivation_app} gives the mode-$j$ sparse loading update
\begin{align}
  \label{eq:ssfa_mode_loading_update}
  \lambda_{j,rc}^{[t,j]}
  =
  \mathcal{T}_{\rho \psi_{j,rr}^{[t,j-1]}/2}
  \left(
    a_{j,rc}^{[t]}
  \right),
  \qquad
  r=1,\ldots,p_j,\quad
  c=1,\ldots,k_j,
\end{align}
where $\mathcal{T}_{\tau}$ is the complex soft-thresholding operator defined in \eqref{eq:sparse_lam_piecewise_update} of Section~\ref{sec:sparse_vec_model_derivation} in the main manuscript. The update for the mode-$j$ diagonal residual error covariance matrix follows from \eqref{eq:sparse-psi-update-app}:
\begin{align}
  \label{eq:ssfa_mode_psi_update}
  \Psib_j^{[t,j]}
  =
  \Bigg[
    \frac{1}{n_j}
    \breve{\Xb}_{(j)}^{[t]}
    \breve{\Xb}_{(j)}^{[t]\star}
    -
    \Lambdab_j^{[t,j]} \Lambdab_j^{[t,j] \star}
  \Bigg]
  \circ\Ib_{p_j}.
\end{align}
This update retains the sparsity structure of $\Lambdab_j^{[t,j]}$ and estimates $\Psib_j^{[t,j]}$ as the diagonal component of the empirical covariance of the whitened mode-$j$ matricization $\breve{\Xb}_{(j)}^{[t]}$ after removing the covariance explained by the retained sparse mode-$j$ factor structure.
The parameters of the remaining modes are held fixed:
\begin{align}
  \label{eq:ssfa_other_modes_fixed}
  \thetab_\ell^{[t,j]}
  =
  \thetab_\ell^{[t,j-1]},
  \qquad
  \ell\neq j.
\end{align}

Starting from $\thetab^{[t,0]}=\thetab^{(t)}$, apply the whitening and
mode-specific update steps in
\eqref{eq:app_array_whitening}--\eqref{eq:ssfa_other_modes_fixed}
successively for $j=1,\ldots,d$. The resulting parameter value $\thetab^{(t+1)} = \thetab^{[t,d]}$ completes one iteration of the SSFA PX-EM procedure.

The expanded PX-EM parameters can also be recovered from the mode-specific reduced updates. Applying the results of Section~\ref{sec:app_fa_desc_m_step}, the expanded factor covariance matrix for mode $j$ is
\begin{align}
  \label{eq:ssfa_mode_gamma_update}
  \Gammab_j^{[t,j]}
  =
  \frac{1}{n_j}
  \Sbb_{j,zz}^{[t]}.
\end{align}
The corresponding expanded loading matrix and reduced loading matrix are
\begin{align}
  \label{eq:ssfa_mode_tilde_lambda_update}
  \tilde{\Lambdab}_j^{[t,j]}
  =
  \Lambdab_j^{[t,j]}
  \operatorname{chol}
  \left(
    \Gammab_j^{[t,j]}
  \right)^{-1}, \qquad   \Lambdab_j^{[t,j]}
  =
  \tilde{\Lambdab}_j^{[t,j]}
  \operatorname{chol}
  \left(
    \Gammab_j^{[t,j]}
  \right),
\end{align}
respectively. The expanded parameters reduce to the mode-$j$ loading matrix obtained in \eqref{eq:ssfa_mode_loading_update}.

Starting from $\tilde{\thetab}^{[t,0]}=\tilde{\thetab}^{(t)}$, apply \eqref{eq:ssfa_mode_gamma_update} and \eqref{eq:ssfa_mode_tilde_lambda_update} successively for $j=1,\ldots,d$. At the mode-$j$ step, the expanded parameters of the remaining modes are held fixed:
\begin{align*}
  \tilde{\thetab}_\ell^{[t,j]}
  =
  \tilde{\thetab}_\ell^{[t,j-1]},
  \qquad
  \ell\neq j.
\end{align*}
The resulting expanded parameter value $  \tilde{\thetab}^{(t+1)} = \tilde{\thetab}^{[t,d]}$ corresponds to the reduced parameter $\thetab^{(t+1)}=\thetab^{[t,d]}$.

\section{Identifiability of the mode-specific covariance matrices}
\label{app:identifiability_balancing}

Under the SSFA model in \eqref{eq:SSFA_model_intro}, if $\Xcal \sim \CC \Ncal_{p_1 \times \cdots \times p_d}\left(\mathbf{0}, \Sigmab_1, \ldots, \Sigmab_d\right)$, then $\operatorname{Cov}\left\{\operatorname{vec}(\Xcal) \right\} = \Sigmab_d\otimes\cdots\otimes\Sigmab_1$. This Kronecker product is identifiable, but its mode-specific factors are identifiable only up to multiplicative rescaling. In particular, for any $c_1,\ldots,c_d>0$ satisfying $\prod_{j=1}^d c_j=1$,
\begin{align*}
  (c_d\Sigmab_d)\otimes\cdots\otimes(c_1\Sigmab_1)
  =
  \left(
    \prod_{j=1}^d c_j
  \right)
  \left(
    \Sigmab_d\otimes\cdots\otimes\Sigmab_1
  \right)
  =
  \Sigmab_d\otimes\cdots\otimes\Sigmab_1.
\end{align*}

Among all admissible multiplicative rescalings, we impose a balancing constraint to select a unique representative. We first establish the result for diagonal matrices with positive diagonal entries. The following proposition shows that equating their minimum diagonal entries uniquely determines the rescaling vector.

\begin{proposition}
\label{prop:balancing_identifiability}

Let $\Psib_1,\ldots,\Psib_d$ be diagonal matrices with positive diagonal entries. There exists a unique vector $(c_1,\ldots,c_d)$ satisfying $c_j>0$, $j=1,\ldots,d$, and $\prod_{j=1}^d c_j=1$ such that
\begin{align}
  \min\left\{\diag(c_1\Psib_1)\right\}
  =
  \cdots
  =
  \min\left\{\diag(c_d\Psib_d)\right\}.
  \label{eq:min-diag}
\end{align}
The unique rescaling constants are
\begin{align}
  c_j
  &=
  \frac{\overline{\psi}_{\min}}{\psi_{j,\min}},
  \qquad
  \psi_{j,\min}
  =
  \min\left\{\diag(\Psib_j)\right\},
  \qquad
  \overline{\psi}_{\min}
  =
  \left(
    \prod_{\ell=1}^d\psi_{\ell,\min}
  \right)^{1/d}.
  \label{eq:iden-c}
\end{align}
The tuple $(c_1\Psib_1,\ldots,c_d\Psib_d)$ is the unique admissible rescaling satisfying \eqref{eq:min-diag}.
\end{proposition}

\begin{proof}
Define
\begin{align*}
  c_j
  &=
  \frac{\overline{\psi}_{\min}}{\psi_{j,\min}},
  \qquad
  j=1,\ldots,d.
\end{align*}
Because $\psi_{j,\min}>0$, each $c_j$ is positive. Furthermore, $(c_1, \ldots, c_d)$ satisfies
\begin{align*}
  \prod_{j=1}^d c_j
  &=
  \frac{\overline{\psi}_{\min}^{\,d}}
       {\prod_{j=1}^d\psi_{j,\min}}
  =
  1,
\end{align*}
so $(c_1,\ldots,c_d)$ is an admissible rescaling vector. For $j=1,\ldots,d$, let $\widetilde{\Psib}_j=c_j\Psib_j$. Then,
\begin{align*}
  \min\left\{
    \diag(\widetilde{\Psib}_j)
  \right\}
  &=
  c_j\psi_{j,\min}
  =
  \overline{\psi}_{\min},
  \qquad
  j=1,\ldots,d.
\end{align*}
Therefore, the rescaled matrices obtained using $(c_1,\ldots,c_d)$  satisfy \eqref{eq:min-diag}.

To establish uniqueness of the rescaling vector, let $(h_1,\ldots,h_d)$ be any admissible rescaling vector satisfying \eqref{eq:min-diag}. There exists a constant $a>0$ such that
\begin{align*}
  h_j\psi_{j,\min}
  &=
  a,
  \qquad
  j=1,\ldots,d.
\end{align*}
Therefore,
\begin{align*}
  h_j
  &=
  \frac{a}{\psi_{j,\min}},
  \qquad
  j=1,\ldots,d.
\end{align*}
Using $\prod_{j=1}^d h_j=1$ gives
\begin{align*}
  1
  &=
  \prod_{j=1}^d
  \frac{a}{\psi_{j,\min}}
  =
  \frac{a^d}
       {\prod_{j=1}^d\psi_{j,\min}},
\end{align*}
and
\begin{align*}
  a
  &=
  \left(
    \prod_{j=1}^d\psi_{j,\min}
  \right)^{1/d}
  =
  \overline{\psi}_{\min}.
\end{align*}
It follows that $h_j=\overline{\psi}_{\min}/\psi_{j,\min}=c_j$ for every $j$. Therefore, the admissible rescaling vector satisfying \eqref{eq:min-diag} is unique. The proposition is proved.
\end{proof}

The following corollary extends Proposition~\ref{prop:balancing_identifiability} to the mode-specific covariance matrices in \eqref{eq:SSFA_model_intro}.

\begin{corollary}
\label{cor:balancing_factor_identifiability}

For $j=1,\ldots,d$, suppose that $\Sigmab_j\in\CC^{p_j\times p_j}$ is Hermitian positive definite with factor-analytic representation
\begin{align}
  \Sigmab_j
  &=
  \Lambdab_j\Lambdab_j^\star+\Psib_j,
  \qquad
  \Lambdab_j\in\CC^{p_j\times k_j},
  \label{eq:app-cor}
\end{align}
where $\Psib_j$ is a real-valued diagonal matrix with positive diagonal entries. Let $(c_1,\ldots,c_d)$ be the rescaling constants in \eqref{eq:iden-c}, and define
\begin{align*}
  \widetilde{\Sigmab}_j
  &=
  c_j\Sigmab_j,
  \qquad
  j=1,\ldots,d.
\end{align*}
Each $\widetilde{\Sigmab}_j$ admits the factor-analytic representation
\begin{align}
  \widetilde{\Sigmab}_j
  &=
  \Lb_j\Lb_j^\star+\Db_j,
  \qquad
  \Lb_j
  =
  c_j^{1/2}\Lambdab_j,
  \qquad
  \Db_j
  =
  c_j\Psib_j,
  \qquad
  j=1,\ldots,d,
  \label{eq:app-cor-2}
\end{align}
and the rescaled diagonal matrices satisfy
\begin{align}
  \min\left\{\diag(\Db_1)\right\}
  &=
  \cdots
  =
  \min\left\{\diag(\Db_d)\right\}.
  \label{eq:app-cor-3}
\end{align}
Furthermore, $\widetilde{\Sigmab}_d \otimes\cdots\otimes \widetilde{\Sigmab}_1 = \Sigmab_d \otimes\cdots\otimes \Sigmab_1$. The tuple $(\widetilde{\Sigmab}_1,\ldots,\widetilde{\Sigmab}_d)$ is the unique admissible multiplicative rescaling of $(\Sigmab_1,\ldots,\Sigmab_d)$ for which the corresponding diagonal matrices $(\Db_1,\ldots,\Db_d)$ satisfy \eqref{eq:app-cor-3}.
\end{corollary}

\begin{proof}
  By Proposition~\ref{prop:balancing_identifiability}, there exists a unique admissible rescaling vector $(c_1,\ldots,c_d)$ such that the minimum diagonal entries of $c_1\Psib_1,\ldots,c_d\Psib_d$ are equal. Define
  \begin{align*}
      \widetilde{\Sigmab}_j=c_j\Sigmab_j,
  \qquad
  \Lb_j=c_j^{1/2}\Lambdab_j,
  \qquad
  \Db_j=c_j\Psib_j,
  \qquad
  j=1,\ldots,d.
  \end{align*}
Using the factor-analytic representation of $\Sigmab_j$,
\begin{align}
  \widetilde{\Sigmab}_j
  &=
  c_j\Sigmab_j
  =
  c_j
  \left(
    \Lambdab_j\Lambdab_j^\star+\Psib_j
  \right)
  =
  \left(
    c_j^{1/2}\Lambdab_j
  \right)
  \left(
    c_j^{1/2}\Lambdab_j
  \right)^\star
  +
  c_j\Psib_j
  =
  \Lb_j\Lb_j^\star+\Db_j.
  \label{eq:id-map}
\end{align}
This proves \eqref{eq:app-cor-2}. Proposition~\ref{prop:balancing_identifiability}
also gives
\begin{align*}
  \min\left\{\diag(\Db_1)\right\}
  =
  \cdots
  =
  \min\left\{\diag(\Db_d)\right\},
\end{align*}
which proves \eqref{eq:app-cor-3}. Because $\prod_{j=1}^d c_j=1$,
\begin{align*}
  \widetilde{\Sigmab}_d
  \otimes\cdots\otimes
  \widetilde{\Sigmab}_1
  =
  (c_d\Sigmab_d)
  \otimes\cdots\otimes
  (c_1\Sigmab_1)
  =
  \left(
    \prod_{j=1}^d c_j
  \right)
  \left(
    \Sigmab_d\otimes\cdots\otimes\Sigmab_1
  \right)
  =
  \Sigmab_d\otimes\cdots\otimes\Sigmab_1.
\end{align*}

To establish uniqueness, let $(b_1,\ldots,b_d)$ be another admissible rescaling vector such that $b_1\Psib_1,\ldots,b_d\Psib_d$ satisfy \eqref{eq:app-cor-3}. Proposition~\ref{prop:balancing_identifiability} implies that $b_j=c_j$ for every $j$, so $  b_j\Sigmab_j = c_j\Sigmab_j = \widetilde{\Sigmab}_j$ for $j=1,\ldots,d$. Therefore, $(\widetilde{\Sigmab}_1,\ldots,\widetilde{\Sigmab}_d)$ is the unique admissible multiplicative rescaling whose corresponding diagonal matrices satisfy \eqref{eq:app-cor-3}.
\end{proof}

\section{Model selection using EBIC}
\label{sec:ebic_description}

After fitting SSFA  \eqref{eq:SSFA_model_intro} over a grid of tuning parameters constructed as described in Section~\ref{sec:rho_grid}, we select $\rho$ using the extended Bayesian information criterion (EBIC) \citep{chen2008extended}. For a fixed $\rho$, the EBIC is
\begin{align*}
  \mathrm{EBIC}(\rho)
  =
  -2 \log \Lcal(\widehat{\thetab}_{\rho})
  +
  h_{\rho}\log(np)
  +
  2\gamma \log \binom{m}{h_{\rho}}.
\end{align*}
where $\widehat {\thetab}_{\rho}$ is the parameter estimate when the tuning parameter equals $\rho$,
$\Lcal(\widehat{\thetab}_{\rho})$ denotes the likelihood evaluated at $\widehat{\thetab}_{\rho}$, $h_{\rho}$ denotes the number of nonzero entries in $\widehat {\thetab}_{\rho}$, $n$ is the number of observations, and $p=\prod_{j=1}^d p_j$ is the total number of variables in a vectorized observation. The quantity $m$ denotes the total number of candidate parameters subject to selection, and $\gamma \in [0,1]$ controls the additional complexity penalty of the EBIC. In our implementation, we set $\gamma = 1$.

\section{Loading matrix estimation accuracy}
\label{sec:load-err}

\subsection{Null-space projection error and principal angles}
\label{sec:supp-null-proj-true-proj}

The following proposition justifies the null-space projection error as a measure of accuracy for the loading spaces estimated by SSFA.

\begin{proposition}
\label{prop:null-proj-true-proj}
For $j=1,\ldots,d$, let $\Lambdab_j\in\CC^{p_j\times k_j}$ denote the true loading matrix and $\widehat{\Lambdab}_j$ denote its estimate. Define the orthogonal projectors onto their column spaces by
\begin{align*}
  \Pb_j
  &=
  \Lambdab_j
  \left(
    \Lambdab_j^\star\Lambdab_j
  \right)^+
  \Lambdab_j^\star,
  \qquad
  \widehat{\Pb}_j
  =
  \widehat{\Lambdab}_j
  \left(
    \widehat{\Lambdab}_j^\star
    \widehat{\Lambdab}_j
  \right)^+
  \widehat{\Lambdab}_j^\star,
\end{align*}
where $+$ denotes the Moore--Penrose pseudoinverse. The null-space projection error is
\begin{align}
  \operatorname{err}_{\Pb_j}^2
  &=
  \left\|
    \left(
      \Ib_{p_j}-\widehat{\Pb}_j
    \right)
    \Pb_j
  \right\|_F^2.
  \label{eq:appendix-null-err-proj}
\end{align}
Let
\begin{align*}
  r_j
  &=
  \operatorname{rank}(\Lambdab_j),
  \qquad
  \widehat{r}_j
  =
  \operatorname{rank}(\widehat{\Lambdab}_j),
  \qquad
  r_{\min,j}
  =
  \min(r_j,\widehat{r}_j).
\end{align*}
Denote the principal angles between the column spaces of $\Lambdab_j$ and $\widehat{\Lambdab}_j$ by
\begin{align*}
    0
  \leq
  \theta_{(j)1}
  \leq
  \cdots
  \leq
  \theta_{(j)r_{\min,j}}
  \leq
  \frac{\pi}{2}.
\end{align*}
Then,
\begin{align}
  \operatorname{err}_{\Pb_j}^2
  &=
  r_j-r_{\min,j}
  +
  \sum_{i=1}^{r_{\min,j}}
  \sin^2\left(\theta_{(j)i}\right),
  \label{eq:appendix-null-err-proj-angle}
\end{align}
where the sum \eqref{eq:appendix-null-err-proj-angle} is interpreted as zero when $r_{\min,j}=0$.
\end{proposition}
The term $r_j-r_{\min,j}$ measures the loss from underestimating the dimension of the true loading space, whereas the sum measures angular disagreement over the dimensions shared by the two spaces.

\begin{proof}
Let $\Qb_j\in\CC^{p_j\times r_j}$ and $\widehat{\Qb}_j\in\CC^{p_j\times\widehat r_j}$ have orthonormal columns spanning the column spaces of $\Lambdab_j$ and $\widehat{\Lambdab}_j$, respectively. Then,
\begin{align*}
  \Pb_j
  &=
  \Qb_j\Qb_j^\star,
  \qquad
  \widehat{\Pb}_j
  =
  \widehat{\Qb}_j\widehat{\Qb}_j^\star.
\end{align*}
Because $\Pb_j$ and $\widehat{\Pb}_j$ are Hermitian and idempotent,
\begin{align*}
  \operatorname{err}_{\Pb_j}^2
  &=
  \operatorname{tr}
  \left[
    \Pb_j
    \left(
      \Ib_{p_j}-\widehat{\Pb}_j
    \right)^2
    \Pb_j
  \right]
  =
  \operatorname{tr}
  \left[
    \Pb_j
    \left(
      \Ib_{p_j}-\widehat{\Pb}_j
    \right)
    \Pb_j
  \right]
  =
  \operatorname{tr}(\Pb_j)
  -
  \operatorname{tr}
  \left(
    \Pb_j\widehat{\Pb}_j\Pb_j
  \right)
  \\
  &=
  r_j
  -
  \operatorname{tr}
  \left(
    \Qb_j^\star
    \widehat{\Qb}_j
    \widehat{\Qb}_j^\star
    \Qb_j
  \right)
  =
  r_j
  -
  \left\|
    \widehat{\Qb}_j^\star\Qb_j
  \right\|_F^2.
\end{align*}
By the definition of principal angles, the singular values of $\widehat{\Qb}_j^\star\Qb_j$ are $\cos(\theta_{(j)1}),\ldots, \cos(\theta_{(j)r_{\min,j}})$. Therefore,
\begin{align*}
  \left\|
    \widehat{\Qb}_j^\star\Qb_j
  \right\|_F^2
  &=
  \sum_{i=1}^{r_{\min,j}}
  \cos^2\left(\theta_{(j)i}\right).
\end{align*}
Substituting this identity into the preceding expression gives
\begin{align*}
  \operatorname{err}_{\Pb_j}^2
  &=
  r_j
  -
  \sum_{i=1}^{r_{\min,j}}
  \cos^2\left(\theta_{(j)i}\right)
  =
  r_j-r_{\min,j}
  +
  \sum_{i=1}^{r_{\min,j}}
  \left[
    1-\cos^2\left(\theta_{(j)i}\right)
  \right]
  \\
  &=
  r_j-r_{\min,j}
  +
  \sum_{i=1}^{r_{\min,j}}
  \sin^2\left(\theta_{(j)i}\right),
\end{align*}
which proves the result.
\end{proof}

\subsection{Weighted null-space projection error}
\label{sec:supp-null-proj-lambda}

The following proposition shows that the weighted null-space projection error is a weighted average of the squared sines of the angles between the left singular vectors of the true loading matrix and the estimated loading space, with weights proportional to the squared singular values.

\begin{proposition}
\label{prop:null-proj-lambda}

For $j=1,\ldots,d$, let $\Lambdab_j\in\CC^{p_j\times k_j}$ be a nonzero true loading matrix, and let $\widehat{\Lambdab}_j$ denote its estimate. Define the orthogonal projector onto the column space of $\widehat{\Lambdab}_j$ by
\begin{align*}
  \widehat{\Pb}_j
  &=
  \widehat{\Lambdab}_j
  \left(
    \widehat{\Lambdab}_j^\star
    \widehat{\Lambdab}_j
  \right)^+
  \widehat{\Lambdab}_j^\star,
\end{align*}
where $+$ denotes the Moore--Penrose pseudoinverse. The weighted null-space projection error is
\begin{align}
  \operatorname{err}_{\Lambdab_j}^2
  &=
  \frac{
    \left\|
      \left(
        \Ib_{p_j}-\widehat{\Pb}_j
      \right)
      \Lambdab_j
    \right\|_F^2
  }{
    \|\Lambdab_j\|_F^2
  }.
  \label{eq:appendix-null-err}
\end{align}

Let $  \Lambdab_j = \Ub_j\Deltab_j\Vb_j^\star$ be a ``compact'' singular value decomposition of $\Lambdab_j$, where
$r_j=\operatorname{rank}(\Lambdab_j)$. For $i=1,\ldots,r_j$, let $\ub_{(j)i}$ denote the $i$th column of $\Ub_j$, and let $d_{(j)i}>0$ be the corresponding singular value. Define $\omega_{(j)i}\in[0,\pi/2]$ by
\begin{align}
  \cos\left(\omega_{(j)i}\right)
  &=
  \left\|
    \widehat{\Qb}_j^\star\ub_{(j)i}
  \right\|_2,
  \label{eq:loading_vector_subspace_angle}
\end{align}
where the columns of $\widehat{\Qb}_j$ form an orthonormal basis for the column space of $\widehat{\Lambdab}_j$. The
angle between $\ub_{(j)i}$ and the estimated loading space is $\omega_{(j)i}$. Then,
\begin{align}
  \operatorname{err}_{\Lambdab_j}^2
  &=
  \sum_{i=1}^{r_j}
  v_{(j)i}
  \sin^2\left(\omega_{(j)i}\right),
  \qquad
  v_{(j)i}
  =
  \frac{
    d_{(j)i}^2
  }{
    \displaystyle
    \sum_{\ell=1}^{r_j}d_{(j)\ell}^2
  }.
  \label{eq:appendix-null-err-angle}
\end{align}
In particular, $\operatorname{err}_{\Lambdab_j}^2\in[0,1]$.
\end{proposition}

\begin{proof}
Using the compact singular value decomposition of $\Lambdab_j$,
\begin{align*}
  \left\|
    \left(
      \Ib_{p_j}-\widehat{\Pb}_j
    \right)
    \Lambdab_j
  \right\|_F^2
  &=
  \left\|
    \left(
      \Ib_{p_j}-\widehat{\Pb}_j
    \right)
    \Ub_j\Deltab_j\Vb_j^\star
  \right\|_F^2
  =
  \left\|
    \left(
      \Ib_{p_j}-\widehat{\Pb}_j
    \right)
    \Ub_j\Deltab_j
  \right\|_F^2\\
  &=
  \sum_{i=1}^{r_j}
  d_{(j)i}^2
  \left\|
    \left(
      \Ib_{p_j}-\widehat{\Pb}_j
    \right)
    \ub_{(j)i}
  \right\|_2^2.
\end{align*}
The second equality follows because $\Vb_j$ has orthonormal columns. Because $\|\ub_{(j)i}\|_2=1$ and $\widehat{\Pb}_j=\widehat{\Qb}_j\widehat{\Qb}_j^\star$ is an orthogonal projector,
\begin{align*}
  \left\|
    \left(
      \Ib_{p_j}-\widehat{\Pb}_j
    \right)
    \ub_{(j)i}
  \right\|_2^2
  &=
  1-
  \ub_{(j)i}^\star
  \widehat{\Pb}_j
  \ub_{(j)i}
  =
  1-
  \left\|
    \widehat{\Qb}_j^\star
    \ub_{(j)i}
  \right\|_2^2
  =
  1-\cos^2\left(\omega_{(j)i}\right)
  =
  \sin^2\left(\omega_{(j)i}\right).
\end{align*}
Therefore,
\begin{align}
  \left\|
    \left(
      \Ib_{p_j}-\widehat{\Pb}_j
    \right)
    \Lambdab_j
  \right\|_F^2
  &=
  \sum_{i=1}^{r_j}
  d_{(j)i}^2
  \sin^2\left(\omega_{(j)i}\right).
  \label{eq:appendix-null-err-reform1}
\end{align}
Furthermore,
\begin{align*}
  \|\Lambdab_j\|_F^2
  &=
  \|\Deltab_j\|_F^2
  =
  \sum_{i=1}^{r_j}d_{(j)i}^2.
\end{align*}
Dividing \eqref{eq:appendix-null-err-reform1} by $\|\Lambdab_j\|_F^2$ gives
\begin{align*}
  \operatorname{err}_{\Lambdab_j}^2
  &=
  \frac{
    \displaystyle
    \sum_{i=1}^{r_j}
    d_{(j)i}^2
    \sin^2\left(\omega_{(j)i}\right)
  }{
    \displaystyle
    \sum_{i=1}^{r_j}d_{(j)i}^2
  }
  =
  \sum_{i=1}^{r_j}
  v_{(j)i}
  \sin^2\left(\omega_{(j)i}\right).
\end{align*}
Because $v_{(j)i}\geq0$, $\sum_{i=1}^{r_j}v_{(j)i}=1$, and $\sin^2(\omega_{(j)i})\in[0,1]$, it follows that $\operatorname{err}_{\Lambdab_j}^2\in[0,1]$.
\end{proof}

\end{document}